\documentclass{article}

\usepackage{comment,bm,dsfont}
\usepackage[colorlinks,linkcolor=black,citecolor=blue,urlcolor=black]{hyperref}
\usepackage{amsmath, amsthm, amsfonts, bbm, amssymb,mathtools}
\usepackage{mathtools}
\usepackage{ifpdf}
\usepackage{graphics,graphicx}
\usepackage[usenames,dvipsnames,svgnames,table]{xcolor}
\usepackage{enumitem} 
\usepackage[ruled]{algorithm2e}
\usepackage{hhline}
\usepackage{textcomp}
\usepackage[export]{adjustbox}
\usepackage{cutwin}
\usepackage[numbers]{natbib}
\usepackage{scalerel}
\usepackage{soul}
\usepackage{chngcntr}
\newcommand\mathbox[1]{\mathord{\ThisStyle{%
  \fboxsep3\LMpt\relax\kern1\LMpt\fbox{$\SavedStyle#1$}\kern1\LMpt}}}
\usepackage{float}
\usepackage{subcaption}
\usepackage{makecell}

\makeatletter 
\newcommand{\labitem}[2]{%
\def\@itemlabel{\textbf{#1}}
\item
\def\@currentlabel{#1}\label{#2}}
\makeatother

\makeatletter 
\newcommand{\labitemc}[2]{%
\def\@itemlabel{\textbf{#1}}
\item
\def\@currentlabel{#1}\label{#2}}
\makeatother

\setlist{nolistsep} 

\def\spacingset#1{\renewcommand{\baselinestretch}%
{#1}\small\normalsize} \spacingset{1}

\usepackage[final]{neurips_2020}
\usepackage[T1]{fontenc}
\usepackage{booktabs}       
\usepackage{amsfonts}       
\usepackage{nicefrac}       
\usepackage{microtype}      

\def \R{\mathbb{R}}
\def \E{\mathbb{E}}
\def \N{\mathbb{N}}

\newcommand{\pto}{\stackrel{P}{\rightarrow}}

\newcommand{\dap}{\stackrel{d}{\approx}}

\newcommand{\Zc}{\mathcal{Z}}
\newcommand{\Lc}{\mathcal{L}}

\newcommand{\Sc}{\mathcal{S}}
\newcommand{\Pc}{\mathcal{P}}

\newcommand{\bx}{\mathbf{x}}

\newcommand{\bv}{\mathbf{v}}
\newcommand{\bu}{\mathbf{u}}
\newcommand{\bw}{\mathbf{w}}

\newcommand{\BB}{\mathbf{B}}

\newcommand{\VV}{\mathbf{V}}

\newcommand{\zd}{\boldsymbol\delta}

\newcommand{\bTheta}{\bm{\Theta}}

\newcommand{\bU}{\bm{U}}

\newcommand{\bg}{\bm{g}}
\newcommand{\ub}{\bm{u}}

\newcommand{\itg}{\lfloor t/\gamma \rfloor} 
\DeclareMathOperator{\sgn}{sign}
\DeclareMathOperator{\diag}{diag}
\renewcommand{\tilde}{\widetilde}
\theoremstyle{definition}
\newtheorem{theo}{Theorem}[section]

\newtheorem{prop}[theo]{Proposition}

\newtheorem{defin}[theo]{Definition}
\newtheorem{rem}[theo]{Remark}

\author{%
  Shih--Kang Chao\thanks{Corresponding author.} \\
  Department of Statistics\\
  University of Missouri\\
  Columbia, MO 65211 \\
  \texttt{chaosh@missouri.edu} \\
  \And
  Zhanyu Wang \\
  Department of Statistics\\
  Purdue University \\
  West Lafayette, IN 47907 \\
  \texttt{wang4094@purdue.edu}\\
  \AND
  Yue Xing\\
  Department of Statistics\\
  Purdue University \\
  West Lafayette, IN 47907 \\
  \texttt{xing49@purdue.edu}\\
  \And
  Guang Cheng\\
  Department of Statistics\\
  Purdue University \\
  West Lafayette, IN 47907 \\
  \texttt{chengg@purdue.edu}\\
}

\renewenvironment{proof}[1][\proofname]{{\noindent\bfseries Proof of #1.}}{\qed}

\begin{document}

\title{Directional Pruning of Deep Neural Networks}
\maketitle
\begin{abstract}
In the light of the fact that the stochastic gradient descent (SGD) often finds a flat minimum valley in the training loss, we propose a novel directional pruning method which searches for a sparse minimizer in or close to that flat region. The proposed pruning method does not require retraining or the expert knowledge on the sparsity level. To overcome the computational formidability of estimating the flat directions, we propose to use a carefully tuned $\ell_1$ proximal gradient algorithm which can provably achieve the directional pruning with a small learning rate after sufficient training. The empirical results demonstrate the promising results of our solution in highly sparse regime (92\% sparsity) among many existing pruning methods on the ResNet50 with the ImageNet, while using only a slightly higher wall time and memory footprint than the SGD. Using the VGG16 and the wide ResNet 28x10 on the CIFAR-10 and CIFAR-100, we demonstrate that our solution reaches the same minima valley as the SGD, and the minima found by our solution and the SGD do not deviate in directions that impact the training loss. The code that reproduces the results of this paper is available at  \url{https://github.com/donlan2710/gRDA-Optimizer/tree/master/directional_pruning}. 
\end{abstract}

\section{Introduction}\label{sec:intro}
Deep neural networks (DNNs), after properly trained, provide the state-of-the-art performance in various domains. Overparameterization is a common practice in modern deep learning, which facilitates better expressive power and faster convergence. On the other hand, overparameterization makes DNN exceedingly large, especially for large-scale tasks. For example, the ImageNet \citep{imagenet09,imagenet15} may need billions of parameters \cite{BHMM19} to become sufficiently overparameterized. {As the number of parameters in DNN is growing fast, the cost to deploy and process large DNNs can be prohibitive on devices with low memory/processing resources or with strict latency requirements}, such as mobile phones, augmented reality devices and autonomous cars. Many achievements have been made in shrinking the DNN while maintaining accuracy, and the MIT Technological Review lists the ``tiny AI'' as one of the breakthroughs in 2020 \cite{mitreview20}. 

Among many methods for shrinking DNN, sparse DNN has attracted much attention. {Here, sparsity refers to the situation that most model parameters are zero in a DNN}. Sparse DNN not only requires less memory and storage capacity, but also reduces inference time \citep{CWZZ18}. One of the popular ways to get sparse DNNs is magnitude pruning \citep{HPTD15,han2015compression,Molchanov17,ZG17,LSTHD19,FC19,FDRC19,GEH19}. Magnitude pruning first learns the model parameters with an optimizer, e.g. stochastic gradient descent (SGD), and then prunes based on the learned magnitude of parameters with an a priori threshold. However, determining a threshold requires some expert knowledge and trial-and-error, as a principle for setting the threshold is not available. In addition, na\"ively masking parameters usually worsens the training loss and testing accuracy. Hence, retraining is needed for the pruned network to regain a similar performance as the dense network \citep{HPTD15}. Unfortunately, retraining as an additional step requires some care \cite{FC19} and additional computation.

\subsection{Directional pruning}

In this paper, we try to answer when a coefficient can be pruned without paying the price of increasing the training loss, and how we can prune based on this. These answers rely on the local geometry of the DNN loss function $\ell(\bw)$, where $\bw$ denotes the parameters.
 
Suppose that $\bw^{SGD}\in\R^d$, the parameter trained by the SGD, has reached a valley of minima. Hence, $\nabla \ell(\bw^{SGD})\approx 0$. The Hessian $\nabla^2 \ell(\bw^{SGD})$ has multiple nearly zero eigenvalues \citep{Sagun16,Sagun18,GKX19,P19}, and the directions associated with these eigenvalues are the flat directions on the loss landscape. Perturbation in these directions causes little change in the training loss by the second order Taylor expansion of $\ell(\bw)$ around $\bw^{SGD}$. We denote the subspace generated by these directions as $\Pc_0$. 

Following \cite{LDS90,BS93}, pruning $\bw^{SGD}$ can be viewed as a perturbation of $\bw^{SGD}$:
\begin{align}
    \bw^{SGD} - A \cdot \sgn(\bw^{SGD}).\label{eq:prune}
\end{align}
Here, $\sgn(\bw^{SGD})\in\{-1,1\}^{d}$ is the sign vector of $\bw^{SGD}$ and $A$ is a diagonal matrix with $0 \leq A_{jj} \leq |w_j^{SGD}|$ for $j=1,\ldots,d$. The $j$th coefficient is pruned if $A_{jj} =|w_j^{SGD}|$. For example, in a 2D illustration in the left panel of Figure \ref{fig:dp}, \eqref{eq:prune} is a vector starting from the origin to a point in the orange rectangle.

Retraining is needed if $A \cdot \sgn(\bw^{SGD})\not\in\Pc_0$. Some empirical studies even suggest $\Pc_0$ is nearly orthogonal to the $\bw^{SGD}$ \cite{GRD18,GKX19}, so generally $A \cdot \sgn(\bw^{SGD})\not\in\Pc_0$. Therefore, we instead consider $\bw^{SGD}-\lambda\cdot\bTheta$ where the perturbation direction $\bTheta\in\Pc_0$ and $\lambda>0$. We maximize the number of $j$ such that $\sgn(\Theta_j)=\sgn(w_j^{SGD})$ for $j=1,\ldots,d$, in order to decay as many coefficients in $\bw^{SGD}$ as possible. Specifically, we select $\bTheta$ as
\begin{align*}
	\bTheta = \arg\min_{\bu\in\Pc_0}\big\|\bu-\sgn(\bw^{SGD})\big\|_2^2,
\end{align*}
i.e. $\bTheta = \Pi_0\{\sgn(\bw^{SGD})\}$, where $\Pi_0$ denotes the projection on the subspace $\Pc_0$. The vector $\bTheta$ does not always decrease the magnitude of $\bw^{SGD}$, and it does whenever $\sgn(w_j^{SGD})\cdot \Theta_j>0$, or 
\begin{align}
	s_j:=\sgn(w_j^{SGD})\cdot \big(\Pi_0\{\sgn(\bw^{SGD})\}\big)_j>0.\label{eq:s}
\end{align}

Decreasing the magnitude of the coefficients with $s_j>0$ in $\bw^{SGD}$ would cause little changes in the training loss, as long as we simultaneously increase the magnitude of coefficients $j'\neq j$ with $s_{j'}<0$ proportional to $|s_{j'}|$.  As illustrated in the left panel of Figure \ref{fig:dp}, the adverse effect due to decreasing the magnitude of $w_2$ ($s_2>0$) can be compensated by increasing the magnitude of $w_1$, so that the net change is the red vector in $\Pc_0$. Note that this argument has a similar spirit as the ``optimal brain surgeon''\cite{BS93}, and it is the key to remove the need of retraining. The $s_j$ can thus be understood as a score to indicate whether pruning the $j$th coefficient causes a (ir)redeemable training loss change. We propose the novel ``directional pruning'' using the score $s_j$ in \eqref{eq:s}.
	
	\begin{figure}[!h]
			\centering
		\includegraphics[width=0.47\textwidth,valign=t]{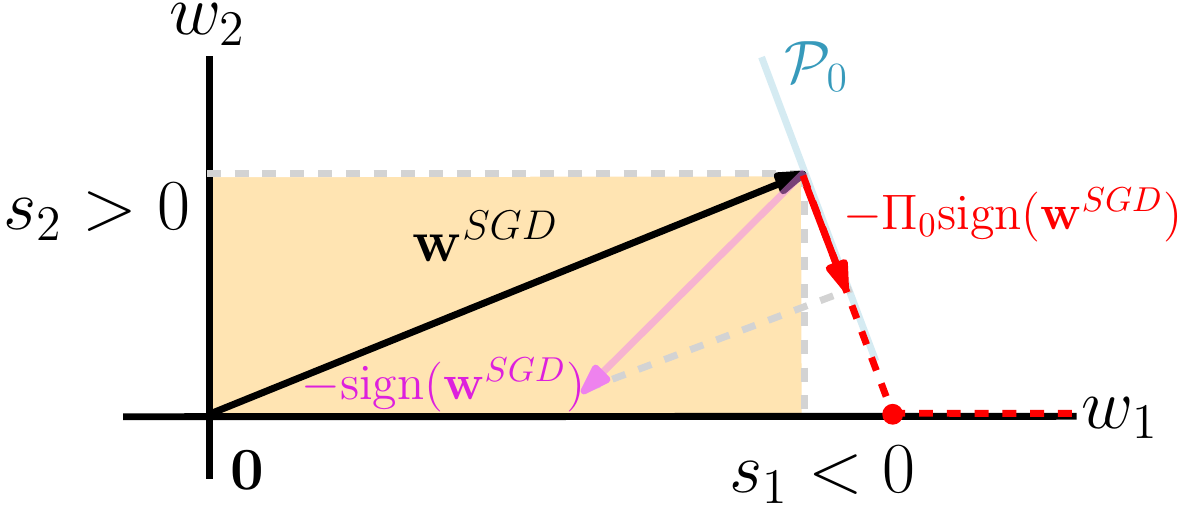}
		\includegraphics[width=0.38\textwidth,valign=t]{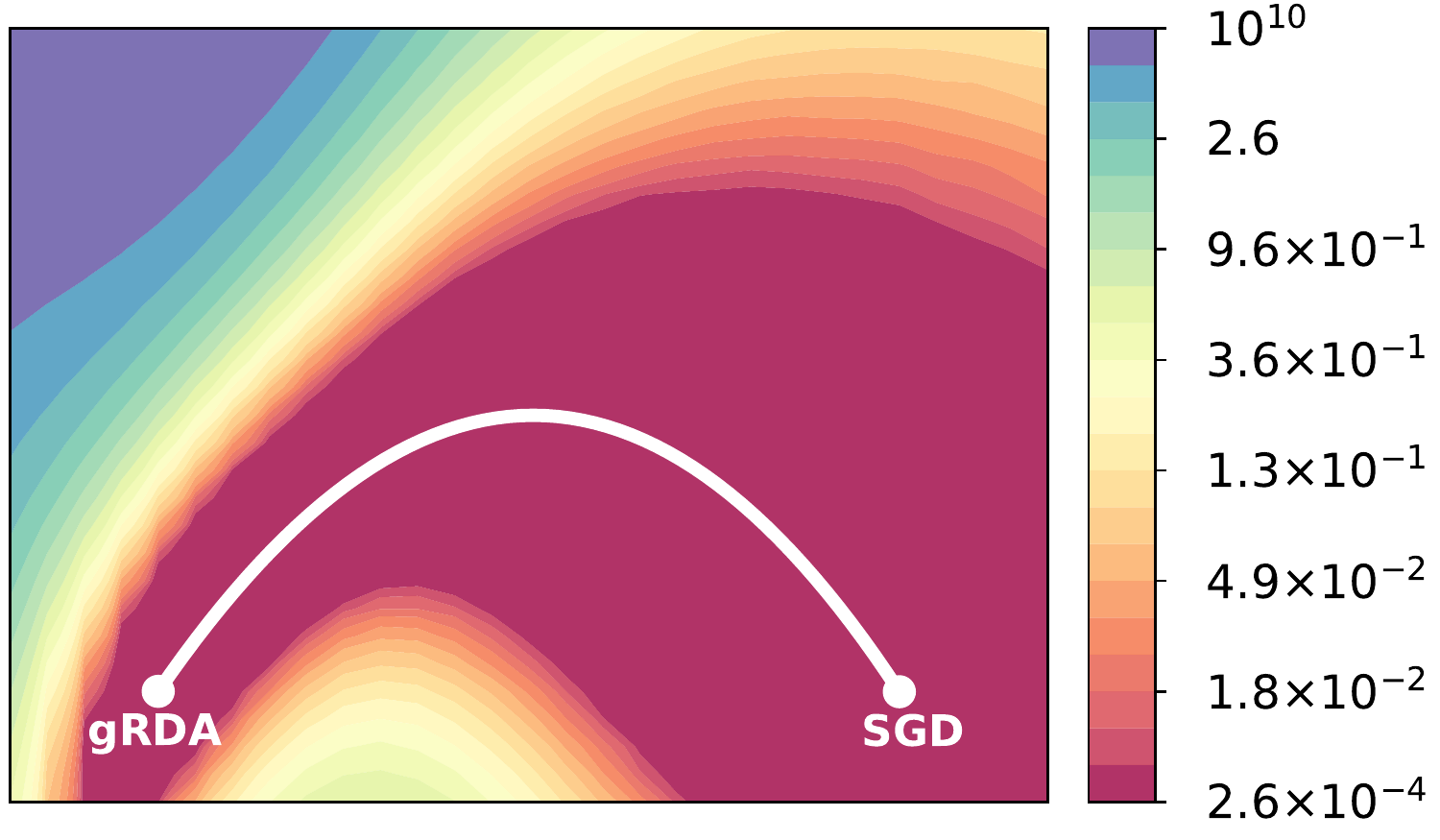}
	 \caption{ {\bf Left:} a 2D graphical illustration of the directional pruning. The orange region contains all possible locations of the vector $\bw^{SGD} - A \cdot \sgn(\bw^{SGD})$. The directional pruning with different $\lambda$ takes solutions on the red dashed line. {\bf Right:} training loss contour of the wide ResNet28$\times 10$ (WRN28x10 \cite{zagoruyko2016wide}) on the CIFAR-100 around the minimal loss path (the white curve) between minimizers found by the SGD and \eqref{eq:grda} \cite{CC19} (the algorithm we propose to use) using \cite{Garipov18}. While no coefficient of the SGD minimizer is zero, our solution has only 9.7\% active parameters. Testing accuracy is 76.6\% for the SGD and 76.81\% for our solution.}\label{fig:dp}
	\end{figure}

\begin{defin}[Directional pruning based on SGD]\label{def:dp} Suppose $\ell(\bw)$ is the training loss, and $\nabla\ell(\bw^{SGD})=0$ where $\bw^{SGD}$ is the minimizer found by SGD. Suppose none of the coefficients in $\bw^{SGD}$ is zero. With $\lambda>0$ and $s_j$ defined in \eqref{eq:s}, the directional pruning solves
	\begin{align}
		\arg\min_{\bw\in\R^d} \frac{1}{2}\|\bw^{SGD}-\bw\|_2^2 + \lambda\sum_{j=1}^d s_j |w_j|. \label{eq:grdamin}
	\end{align}
\end{defin}

In \eqref{eq:grdamin}, the coefficients with $s_j>0$ are pruned with sufficiently large $\lambda$ by the absolute value penalization, but the magnitude of $w_{j'}$ with $s_{j'}\leq 0$ is un-penalized, and are even encouraged to increase. For a 2D illustration, the solution path for different $\lambda>0$ is the dashed red curve in the left panel of Figure \ref{fig:dp}. If $\lambda$ is too large, the coefficients $j$ with $s_j<0$ may overshoot, illustrated as the flat part on the dashed red line extended to the right of the red point. 

\begin{rem}[Solution of \eqref{eq:grdamin}] \label{rem:def_dp}
The objective function in \eqref{eq:grdamin} is separable for each coefficient. The part with $s_j>0$ is solved by the $\ell_1$ proximal operator. The part with $s_j<0$ is non-convex, but it still has the unique global minimizer if $w_j^{SGD}\neq 0$. The solution of \eqref{eq:grdamin} is $$\widehat{w}_j = \sgn(w^{SGD}_j) \big[|w^{SGD}_j| - \lambda s_j\big]_+,$$ where $[a]_+ = \max\{0,a\}$. See Proposition \ref{prop:uniq_sol_def} in the appendix for a proof. 
\end{rem}

Implementing the directional pruning is very challenging due to high dimensionality. Specifically, the matrix $\nabla^2\ell$ of modern deep neural network is often very large so that estimating $\Pc_0$ is computationally formidable. Perhaps surprisingly, we will show that there is a very simple algorithm \eqref{eq:grda} presented in Section \ref{sec:alg}, that can asymptotically solve \eqref{eq:grdamin} without explicitly estimating the Hessian. The right panel of Figure \ref{fig:dp} shows that if $\lambda$ is selected appropriately, our method achieves a similar training loss as the dense network with $\bw^{SGD}$, while being highly sparse with a test accuracy comparable to the SGD. More detailed empirical analysis is in Section \ref{sec:conn}. 

\begin{rem}[Major differences to the ``optimal brain surgeon''] It is worth noting that \eqref{eq:grdamin} is different from the optimization problem in \cite{BS93,HSW93}. While an analytic map between directional pruning and optimal brain surgeon is interesting for future study, the two are generally nonequivalent. Particularly, directional pruning perturbs from $\bw^{SGD}$ continuously in $\lambda$ like a restricted $\ell_1$ weight decay on $\Pc_0$ (Remark \ref{rem:def_dp}), while optimal brain surgeon yields a discontinuous perturbation like a hard thresholding (see p.165 of \cite{BS93}). The main advantage of directional pruning is that it can be computed with the gRDA algorithm presented in Section \ref{sec:alg}, which does not require to estimate the Hessian or its inverse.
\end{rem}

\subsection{Contributions}

Our major contribution is to propose the novel directional pruning method (Definition \ref{def:dp}), and further prove that the algorithm \eqref{eq:grda} \cite{CC19} achieves the effect of the directional pruning asymptotically. The \eqref{eq:grda} has been applied for sparse statistical inference problems with a convex loss and principal component analysis \cite{CC19}. The connection between the directional pruning and \eqref{eq:grda} is theoretically proved by leveraging the continuous time approximation developed in \cite{CC19} under proper assumptions on the gradient flow and the Hessian matrix. It is worth noting that this algorithm does not require to explicitly estimate $\Pc_0$, and it can be implemented like an optimizer in a typical deep learning framework, e.g. Tensorflow or PyTorch.

Empirically, we demonstrate that \eqref{eq:grda} successfully prunes ResNet50 on ImageNet, and achieves 73\% testing accuracy with only 8\% active parameters. 
Upon benchmarking with other popular algorithms, \eqref{eq:grda} yields a high accuracy and sparsity tradeoff among many contemporary methods. We also successfully prune deep networks on CIFAR-10/100, and the results are in the appendix. Using VGG16 on CIFAR-10 and WRN28x10 on CIFAR-100, we show that \eqref{eq:grda} reaches the same valley of minima as the SGD, empirically verifying the directional pruning. Using VGG16 and WRN28x10 on CIFAR-10, we show the proportion of the difference between \eqref{eq:grda} and the SGD in the leading eigenspace of the Hessian is low, as another evidence for \eqref{eq:grda} performing the directional pruning. 

\section{The gRDA algorithm}\label{sec:alg}

Consider training data $Z_i = \{(X_i,Y_i)\}_{i=1}^N$, where $X_i$ is the input variable, e.g. images, and $Y_i$ is the response variable, e.g. a vector of real numbers, or labels $Y_n\in\{0,1\}^{n_l}$, where $n_l\in\N$. Suppose $h(x;\bw)\in\R^{n_l}$ is the output of an $L$-layer feedforward overparameterized DNN, with parameters $\bw\in\R^d$. Let $\Lc(h;y):\R^{n_l\times n_l}\to\R_+$ be a loss function, e.g. the $\ell_2$ loss $\Lc(h;y)=\|h-y\|_2^2$ 
or the cross-entropy loss.  
Let $f(\bw;Z) := \Lc(h(X;\bw),Y)$, and $\nabla f(\bw;Z)$ be the gradient of $f(\bw;Z)$, the loss function $\ell(\bw)$ and its gradient are defined by
\begin{align}
	\ell(\bw):=\E_{\Zc}[f(\bw;Z)], \quad G(\bw)=\nabla \ell(\bw) = \E_{\Zc}[\nabla f(\bw;Z)], \label{eq:G}
\end{align} 
where $\E_{\Zc}[f(\bw;Z)]=N^{-1}\sum_{i=1}^N f(\bw;Z_i)$. 

We adopt the generalized regularized dual averaging (gRDA) algorithms originally proposed in \cite{CC19}. This algorithm has been successfully applied to the ad click-through rate prediction \cite{autofis}. Specifically, let $\{\hat i_k\}_{k=1}^\infty$ be i.i.d. uniform random variables on $\{1,\ldots,N\}$ independent from the training data, 
\begin{align}\tag{\texttt{gRDA}}
	w_{n+1,j} = \Sc_{g(n,\gamma)}\bigg(w_{0,j}-\gamma\sum_{k=0}^{n}\nabla f_j(\bw_{k};Z_{\hat i_{k+1}})\bigg), \mbox{ for $j=1,\ldots,d$},
	\label{eq:grda}
\end{align}
where $\Sc_{g}:v\mapsto \sgn(v)(|v|-g)_+$ is the soft-thresholding operator, $\bw_0$ is an initializer chosen at random from a distribution; $\gamma$ is the learning rate; $g(n,\gamma)>0$ is the {tuning function}, detailed in \eqref{eq:tune}. We can extend \eqref{eq:grda} to minibatch gradients, by replacing $\nabla f_j(\bw_k;Z_{\hat i_{k+1}})$ with an average $|S_{k+1}|^{-1}\sum_{i\in S_{k+1}} \nabla f(\bw_k;Z_i)$, where $S_{k+1}\subset\{1,\ldots,N\}$ is sampled uniformly. We will focus on \eqref{eq:grda}, i.e. $|S_k|=1$ for all $k$, but our theory can be generalized to any fixed minibatch size.

The tuning function $g(n,\gamma)$ controls the growth rate of penalization. Motivated by \cite{CC19},
\begin{align}
	g(n,\gamma)= c \gamma^{1/2}(n\gamma)^\mu, \label{eq:tune}
\end{align}
where $c,\mu>0$ are the two hyperparameters positively related to the strength of penalization. The $(n\gamma)^\mu$ is used to match the growing magnitude of SGD. The $\gamma^{1/2}$ is an important scaling factor; without it, \eqref{eq:grda} with $\mu=1$ reduces to the regularized dual averaging (RDA) algorithm \cite{X10} that minimizes $\ell(\bw)+\lambda\|\bw\|_1$ rather than the directional pruning problem in \eqref{eq:grdamin}. Note that if $c=0$, then \eqref{eq:grda} recovers the stochastic gradient descent:
\begin{align}\tag{\texttt{SGD}}
	\bw_{n+1}^{SGD} = \bw_n^{SGD} -\gamma \nabla f(\bw_n^{SGD};Z_{\hat i_{n+1}}). \label{eq:sgd}
\end{align}

In this paper, we only consider the constant learning rate. In practice, a ``constant-and-drop'' learning rate is often adopted. See Section \ref{app:algorithm} and \ref{sec:cd} in the appendix for the algorithms in pseudocode. 

\begin{rem}[Selection of $\mu$ and $c$ in practice]
	Our empirical results and theory in later sections suggest $\mu\in\{0.501,0.51,0.55\}$ generally performs well regardless of the task and network used. For a given $\mu$, we recommend to search for the greatest $c$ (starting with e.g. $10^{-4}$) such that gRDA yields a comparable test acc. as SGD using $1-5$ epochs.
\end{rem}

\section{Theoretical analysis}\label{sec:th}
To show \eqref{eq:grda} asymptotically achieves the directional pruning in Definition \ref{def:dp}, we leverage some tools from the continuous time analysis. Define the gradient flow $\bw(t)$ to be the solution of the ordinary differential equation
\begin{align}\tag{GF}
	 \dot\bw = -G(\bw),\;\bw(0)=\bw_0, \label{eq:gf}
\end{align}
where $\bw_0$ is a random initializer, and $G$ is defined in \eqref{eq:G}. The $\bw(t)$ can provably find a good global minimizer under various conditions \citep{ACH18,ACGH19,DZPS19,Lee19,OS19,DLT18}. Throughout this paper, we assume the solution of \eqref{eq:gf} is unique.

Let $H(\cdot):=\E_{\Zc}[\nabla^2 f(\cdot;Z)]$ be the Hessian matrix. Let $\Phi(t,s)\in\R^{d\times d}$ be the solution (termed the principal matrix solution, see Chapter 3.4 of \cite{T12}) of the matrix ODE system ($s$ is the initial time):
		\begin{align}
			\frac{d\Phi(t,s)}{dt} = -H(\bw(t)) \Phi(t,s),\quad \Phi(s,s)=I_d.\label{eq:odex}
		\end{align}
Let $\bw_\gamma(t):=\bw_{\lfloor t/\gamma\rfloor}$ and $\bw^{SGD}(t)$ be the piecewise constant interpolated process of \eqref{eq:grda} and \eqref{eq:sgd}, respectively, with the same learning rate, where $\lfloor a\rfloor$ takes the greatest integer that is less than or equal to $a$. We will make the following assumptions:
\begin{itemize}
	\labitemc{(A1)}{as:M} $G(\bw):\R^d\to\R^d$ is continuous on $\R^d$.	
\end{itemize}

 Define
	\begin{align}
		\Sigma(\bw):=\E_{\Zc}\big[\big(\nabla f(\bw;Z)-G(\bw)\big)\big(\nabla f(\bw;Z)-G(\bw)\big)^\top\big].\label{eq:sig}
	\end{align}	
\begin{itemize}
		\labitemc{(A2)}{as:L} $\Sigma:\R^d\to\R^{d\times d}$ is continuous. $\E_{\Zc}\big[\sup_{\|\bw\|\leq K}\big\|\nabla f\big(\bw,Z\big)\big\|_2^2\big]<\infty$ for any $K>0$ a.s.
			\labitemc{(A3)}{as:H} $H: \R^d\to\R^{d\times d}$ is continuous, and there exists a non-negative definite matrix $\bar H$ such that 
		$\int_0^\infty\|H(\bw(s))- \bar H\| ds <\infty$ where $\|\cdot\|$ is the spectral norm, and the eigenspace of $\bar H$ associated with the zero eigenvalues matches $\Pc_0$. 
		\labitemc{(A4)}{as:pms} $\int_0^t s^{\mu-1}\Phi(t,s)\sgn(\bw(s))ds = o(t^\mu)$.
		\labitemc{(A5)}{as:sign} There exists $\bar T>0$ such that for all $t>\bar T$: (i) $\sgn\{\bw(t)\}=\sgn\{\bw(\bar T)\}$; (ii) $\sgn\{w_j(t)\}=\sgn\{w_j^{SGD}(t)\}$ for all $j$.
\end{itemize}
 
The key theoretical result of this paper shows that \eqref{eq:grda} performs the directional pruning (Definition \ref{def:dp}) for a sufficiently large $t$.

\begin{theo}\label{th:dp}
	Under assumptions \ref{as:M}-\ref{as:sign}, and assume $\mu\in(0.5,1)$ and $c>0$ in \eqref{eq:tune}. Then, as $\gamma\to 0$, \eqref{eq:grda} asymptotically performs directional pruning based on $\bw^{SGD}(t)$; particularly, 
	\begin{align}
		\bw_\gamma(t)\dap\arg\min_{\bw\in\R^d}\bigg\{ \frac{1}{2}\|\bw^{SGD}(t)-\bw\|_2^2 + \lambda_{\gamma,t}\sum_{j=1}^d \bar s_j |w_j|\bigg\}, \quad \mbox{for any $t>\bar T$}, \label{eq:dp1}
	\end{align}
	where $\dap$ means ``asymptotic in distribution'' under the empirical probability measure of the gradients, $\lambda_{\gamma,t} = c \sqrt{\gamma} t^\mu$ and the $\bar s_j$ satisfies $\lim_{t\to\infty}|\bar s_j - s_j| = 0$ for all $j$.
\end{theo}

This theorem holds in the asymptotic regime ($\gamma\to 0$) with a finite time horizon, i.e. any fixed $t\geq \bar T$. It is important that $\lambda$ grows with $t$, because the magnitude of SGD asymptotically grows like a Gaussian process, i.e., in $t^{0.5}$. Hence, $\mu$ should be slightly greater than 0.5. The proof of Theorem \ref{th:dp} is in Section \ref{sec:pfdp} of the appendix.

\begin{rem}[Condition \ref{as:H}]
The eigenspace of $\bar H$ associated with the zero eigenvalues and $\Pc_0$ matches when $\bw(t)$ and SGD converge to the same flat valley of minima. For the $\ell_2$ loss and in the teacher-student framework, \cite{DLT18,Xiao19,YQ19} showed $\bw(t)\to\bw^*$ exponentially fast for one hidden layer networks, so the limit $\bar H=H(\bw^*)$ and the condition holds. For the cross-entropy loss, we suspect that $\bar H$ satisfying \ref{as:H} is not a zero matrix, but its exact form needs further investigation. 
\end{rem}

\begin{figure}[H]
    \begin{minipage}{0.6\textwidth}
    \begin{rem}[Condition \ref{as:pms}]
    	This condition can be verified (by Problem 3.31 of \cite{T12}) 
    	if $\sgn(\bw(t))$ is mainly restricted in the eigenspace of $H(\bw(t))$ associated with positive eigenvalues as $t\to\infty$. Empirically, this appears to hold as \cite{GRD18,GKX19} show that $\bw(t)$ lies mainly in the subspace of $H(\bw(t))$ associated with the positive eigenvalues, and Figure \ref{fig:sgdsign} suggests the angle between $\bw(t)$ and $\sgn(\bw(t))$ is very small.
    \end{rem}
    \begin{rem}[Condition \ref{as:sign}]
    	For (i), under the cross-entropy loss, several papers \cite{SENS18,GLSS18,JT19,LL20} show that $\bw(t)/\|\bw(t)\|_2$ converges to a unique direction while $\|\bw(t)\|_2\to\infty$. This implies that $\sgn(\bw(t))$ stabilizes after a finite time. For the $\ell_2$ loss, \cite{DLT18,Xiao19} show $\bw(t)\to\bw^*$ for one hidden layer networks under regularity conditions, and the condition follows. The (ii) holds if the learning rate is sufficiently small, so that the deviation between the gradient flow and the SGD is small. 
    \end{rem}
    \end{minipage}
    \hfill
    \begin{minipage}{0.37\textwidth}
    	 \includegraphics[width=\textwidth]{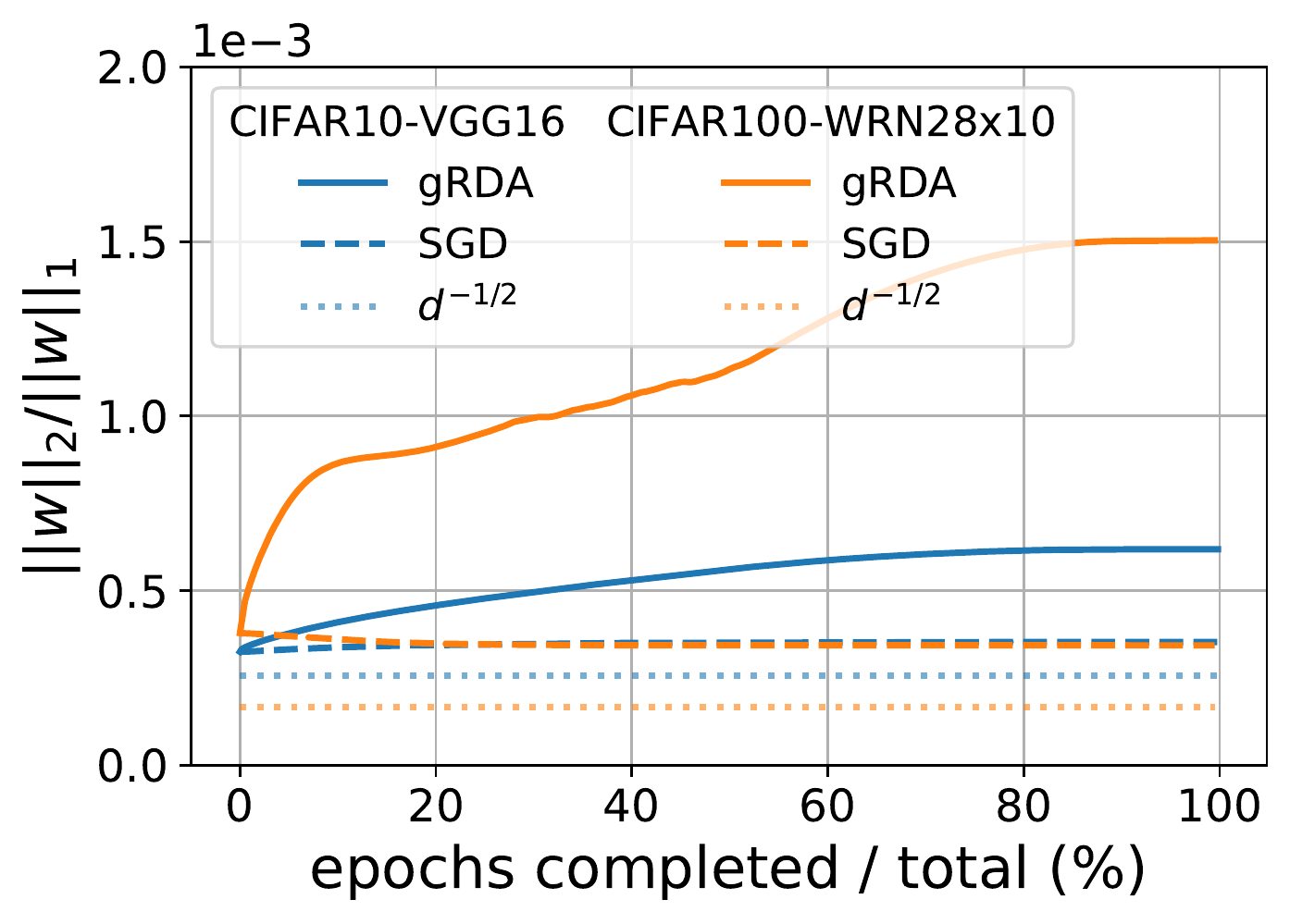}
    	 \caption{$\|\bw\|_2/\|\bw\|_1$ is close to its lower bound $d^{-1/2}$ when the coefficients in $\bw$ are of similar magnitude, i.e. the direction of $\bw$ is the same as $\sgn(\bw)$.}\label{fig:sgdsign}
    \end{minipage}
\end{figure}

\section{Empirical experiments}\label{sec:empi}

This section presents the empirical performance of \eqref{eq:grda}, and the evidence that \eqref{eq:grda} performs the directional pruning (Definition \ref{def:dp}). Section \ref{sec:imagenet} considers ResNet50 with ImageNet, and compares with several existing pruning algorithms. To check if \eqref{eq:grda} performs the directional pruning, Section \ref{sec:conn} presents the local geometry of the loss around the minimal loss curve that connects the minima found by \eqref{eq:sgd} and \eqref{eq:grda}, and Section \ref{sec:proj} investigates the direction of deviation between the minima found by \eqref{eq:sgd} and \eqref{eq:grda}. 

\subsection{ResNet50 on the ImageNet}\label{sec:imagenet}

We use \eqref{eq:grda} to simultaneously prune and train the ResNet50 \cite{HZRS15} on the ImageNet dataset without any post-processing like retraining. The learning rate schedule usually applied jointly with the SGD with momentum does not work well for \eqref{eq:grda}, so we use either a constant learning rate or dropping the learning rate only once in the later training stage. Please find more implementation details in Section \ref{ap:imagenet_details} in the appendix. The results are shown in Figure \ref{fig:resnet50}, where $\mu$ is the increasing rate of the soft thresholding in the tuning function \eqref{eq:tune} of \eqref{eq:grda}.
\begin{figure}[!h]
		\centering
	\includegraphics[width=\textwidth]{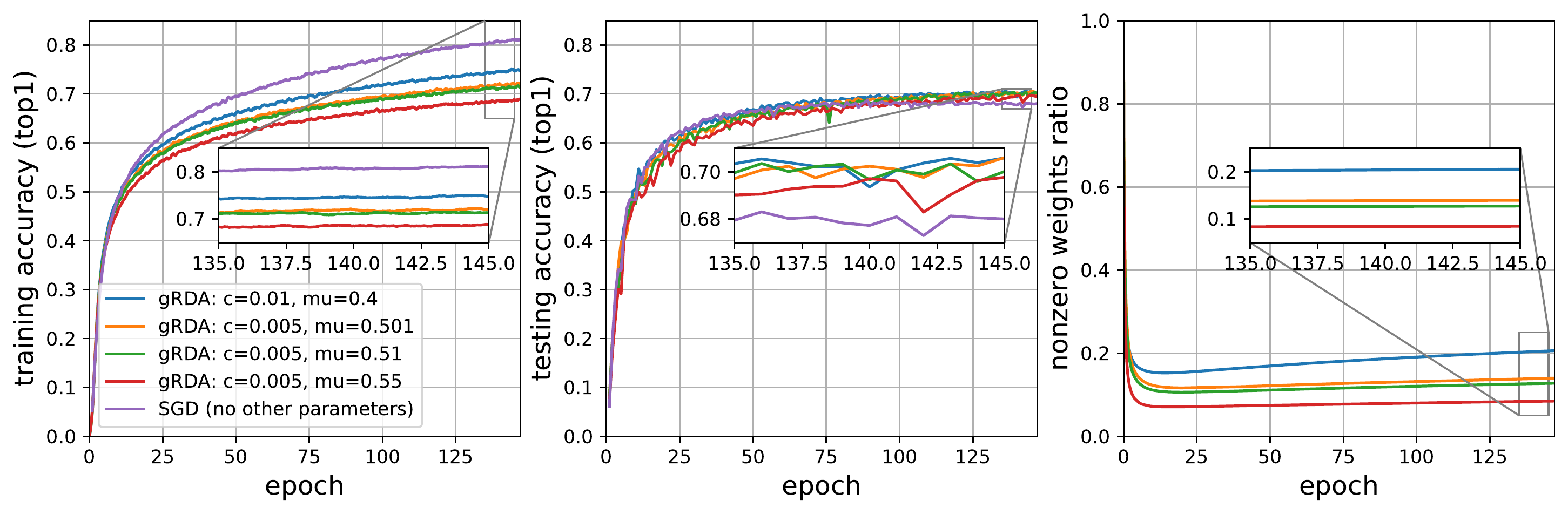}
 \caption{Learning trajectories of \eqref{eq:sgd} and \eqref{eq:grda} for ResNet50 \cite{HZRS15} on ImageNet image recognition task. {\bf Left:} top 1 training accuracy. {\bf Center:} top 1 testing accuracy. {\bf Right:} the ratio between the number of nonzero parameters and the total number of parameters. The number of nonzero weights slightly increases, contradicting with Theorem \ref{th:dp}. This could be because that Assumption \ref{as:sign} fails due to the large learning rate. $\gamma=0.1$ for both SGD and gRDA. Minibatch size is 256.
 }\label{fig:resnet50}
\end{figure}

	 {\bf Accuracy}: gRDAs can perform as accurate as \eqref{eq:sgd} after sufficient training. Larger $\mu$ (in the tuning function \eqref{eq:tune}) can perform worse than \eqref{eq:sgd} in the early stage of training, but eventually beat \eqref{eq:sgd} in the late stage of training. The {training} accuracy of \eqref{eq:sgd} is higher than that of the gRDAs. This may result from a too large learning rate, so the coefficients $w_j$'s with $s_j<0$ (in \eqref{eq:grdamin}) overshoot and their magnitudes become too large. \\
	 {\bf Sparsity}: Sparsity increases rapidly at the early stage of training. With $\mu=0.55$ in Figure \ref{fig:resnet50}, \eqref{eq:grda} reaches 92\% sparsity, while the testing accuracy is higher than \eqref{eq:sgd}.\\
	 {\bf Wall time and memory footprint}: \eqref{eq:grda} has a slightly higher wall time than \eqref{eq:sgd}, but the memory footprint is similar. See Section \ref{ap:resource_comparison} for a detailed comparison.

The left panel of Figure \ref{fig:compareresnet50} compares \eqref{eq:grda} with the magnitude pruning \cite{ZG17} and the variational dropout \cite{pmlr-v70-molchanov17a}, and \eqref{eq:grda} is particularly competitive in the high sparsity (90-92\%) regime. The right panel of Figure \ref{fig:compareresnet50} compares different pruning algorithms that do not require expert knowledge for selecting the layerwise pruning level with \eqref{eq:grda} in terms of the layerwise sparsity. We compare \eqref{eq:grda} with the Erd\H{o}s-R\'enyi-Kernel of \cite{evci19}, variational dropout \cite{pmlr-v70-molchanov17a} and a reinforcement-learning based AutoML method \cite{He18}. Our \eqref{eq:grda} achieves a high sparsity 92\% with a competitive testing accuracy. In addition, the layerwise sparsity pattern generated by gRDA is similar to the variational dropout and the AutoML, as these methods generate higher sparsity in the 3$\times$3 convolutional layers, and lower sparsity in the 1$\times$1 layers and the initial layers, which are less wide than the latter layers. Among these methods, \eqref{eq:grda} is unique in that its spirit is interweaving with the local loss landscape.

\begin{figure}[!htbp]
		\centering
	\includegraphics[width=0.43\textwidth, trim={0 0 0 45}, clip]{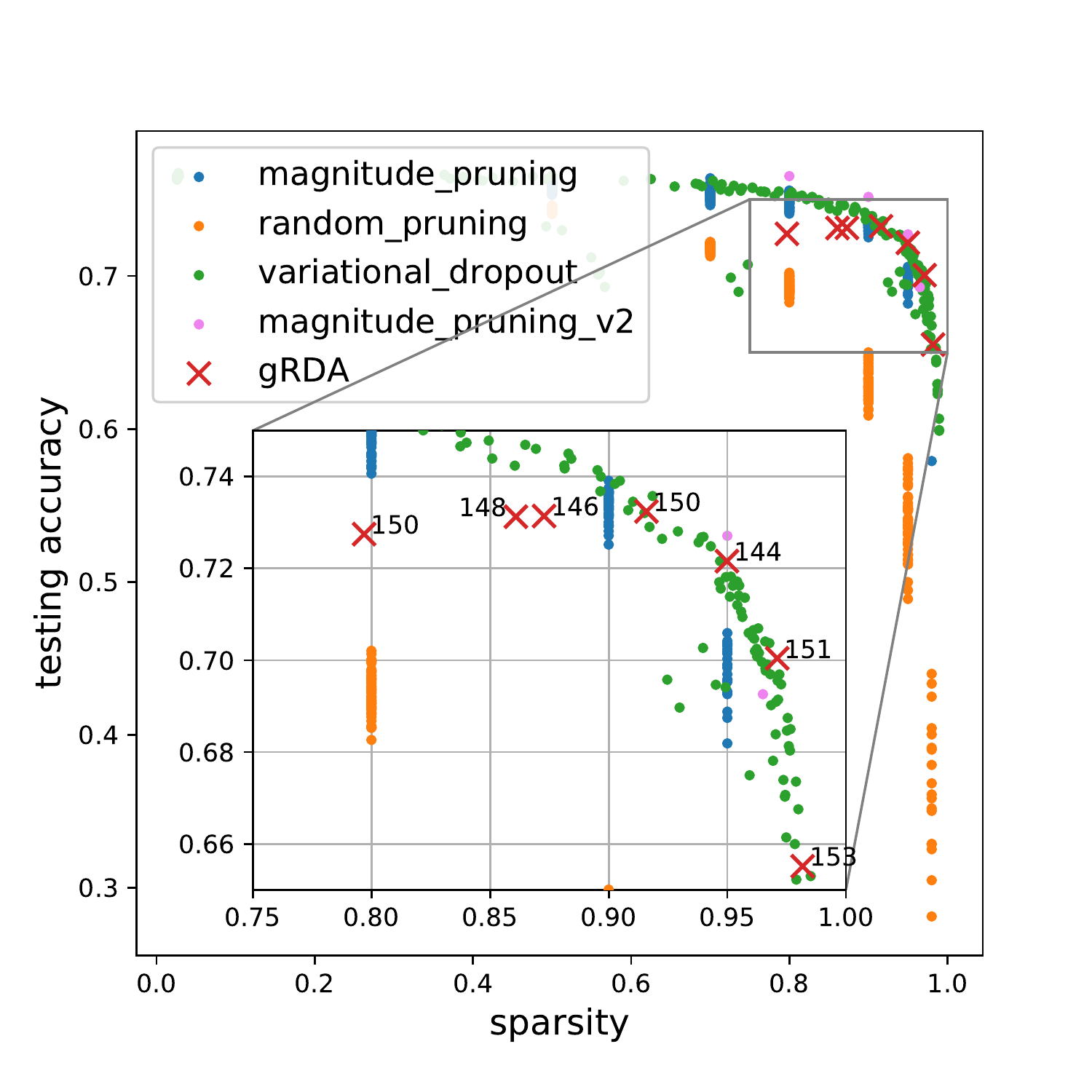}
	\includegraphics[width=0.54\textwidth, trim={0 0 0 45}, clip]{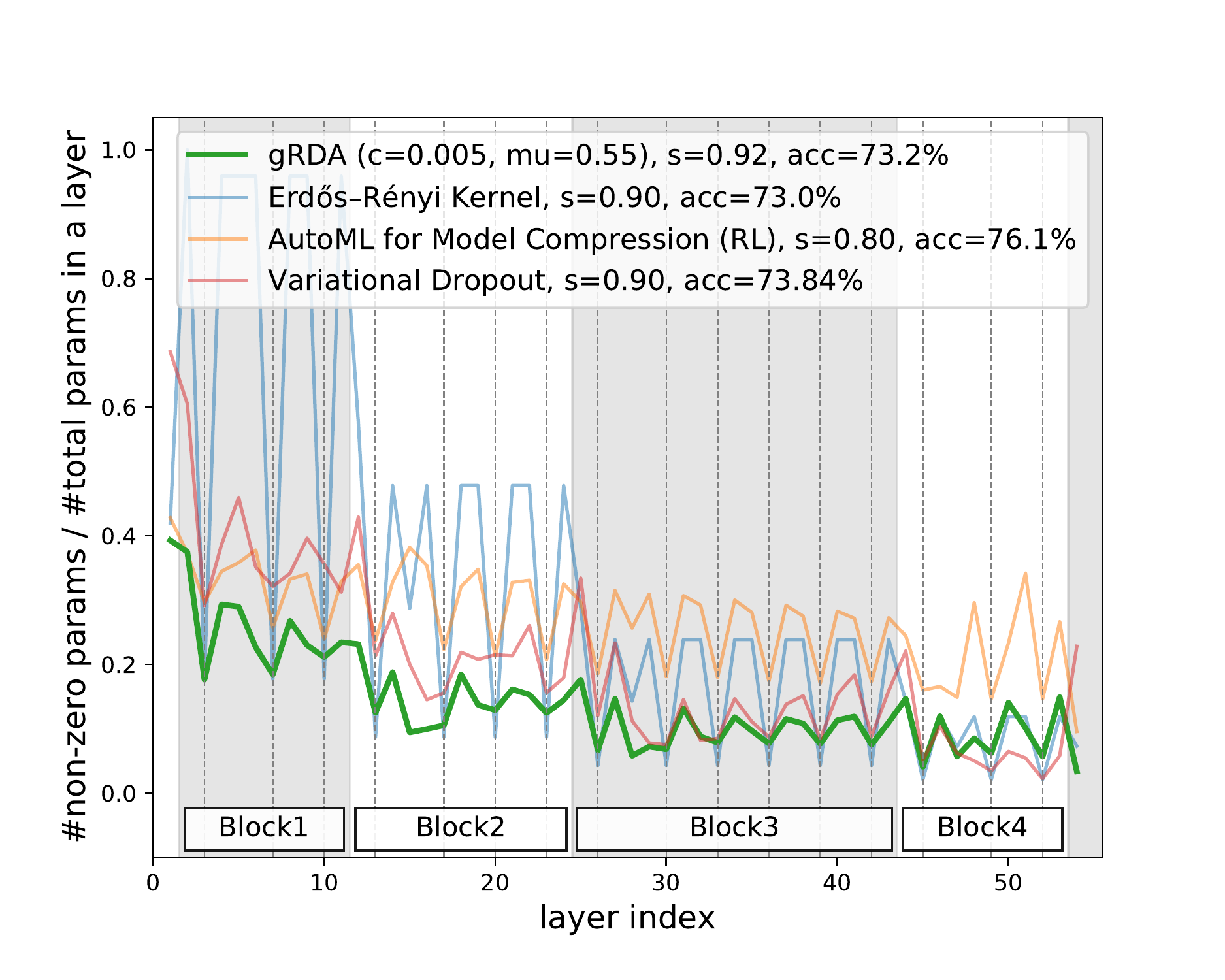}
 \caption{{\bf Left:} A comparison of gRDA with the magnitude pruning \cite{ZG17} and variational dropout \cite{pmlr-v70-molchanov17a} with ResNet50 on ImageNet, done by \cite{GEH19} with around 100 epochs using SGD with momentum. Our solution is among the high performers in the very sparse regime (90-92\%). 
 The numbers next to the red crosses are the epochs. 
 {\bf Right:} Layerwise sparsity produced by different ``automatic'' pruning algorithms. All methods show the pattern that the 3x3 conv layers (on dashed lines) are greatly pruned (valleys), and the 1x1 conv layers are less pruned (peaks).}\label{fig:compareresnet50}
\end{figure}

\subsection{Connectivity between the minimizers of gRDA and SGD}\label{sec:conn}

In this section, we check whether \eqref{eq:sgd} and \eqref{eq:grda} reach the same valley, which implies \eqref{eq:grda} is performing the directional pruning. Similar analysis has been done for the minima found by \eqref{eq:sgd} with different initializers  \cite{WS19,Garipov18,N19,Draxler18,HHY19,EPGE19,Gotmare18}. 

We train VGG16 \cite{vgg16} on CIFAR-10 and WRN28x10 on CIFAR-100 until nearly zero training loss using both \eqref{eq:sgd} and \eqref{eq:grda}. The minima here found by \eqref{eq:grda} generally have sparsity around 90\% or higher for larger $\mu$. We use the method of \cite{Garipov18} to search for a quadratic B\'ezier curve of minimal training loss connecting the minima found by the gRDA and SGD, and then visualize the contour of the training losses and testing errors on the hyperplane containing the minimal loss curve. 
See Sections \ref{append:cifar} and \ref{append:conn} for details on implementation.

The results are shown for different choices of $\mu$, which is the increasing rate of the soft thresholding in the tuning function \eqref{eq:tune} of \eqref{eq:grda}. As observed from the contours in Figure \ref{fig:conn}, the learned parameters of both SGD and gRDA lie in the same valley on the training loss landscape if $\mu$ is properly tuned, namely, $0.6$ for VGG16 and $0.501$ for WRN28x10. This verifies that \eqref{eq:grda} performs the directional pruning. For large $\mu$, a hill exists on the minimal loss/error path, which may be due to the too large learning rate that leads to large magnitude for the coefficients $j$ with $s_j<0$. The details (training accuracy, testing accuracy, sparsity) of the endpoints trained on VGG16 and WRN28x10 are shown in Tables \ref{tab:cifar10vgg} and \ref{tab:cifar100wideres} of the Appendix. For the testing error in Figure \ref{fig:conn}, the gRDA somewhat outperforms SGD when $\mu$ is slightly greater than $0.5$. Interestingly, the neighborhood of the midpoint on the B\'ezier curve often has a higher testing accuracy than the both endpoints, except for WRN28x10 on CIFAR-100 with $\mu=0.501$ and 0.55. This finding resonates with the results of \cite{Izmailov18}. 

\begin{figure}	
	\centering
	\begin{subfigure}[t]{0.48\textwidth}
		\centering
		\includegraphics[width=\textwidth]{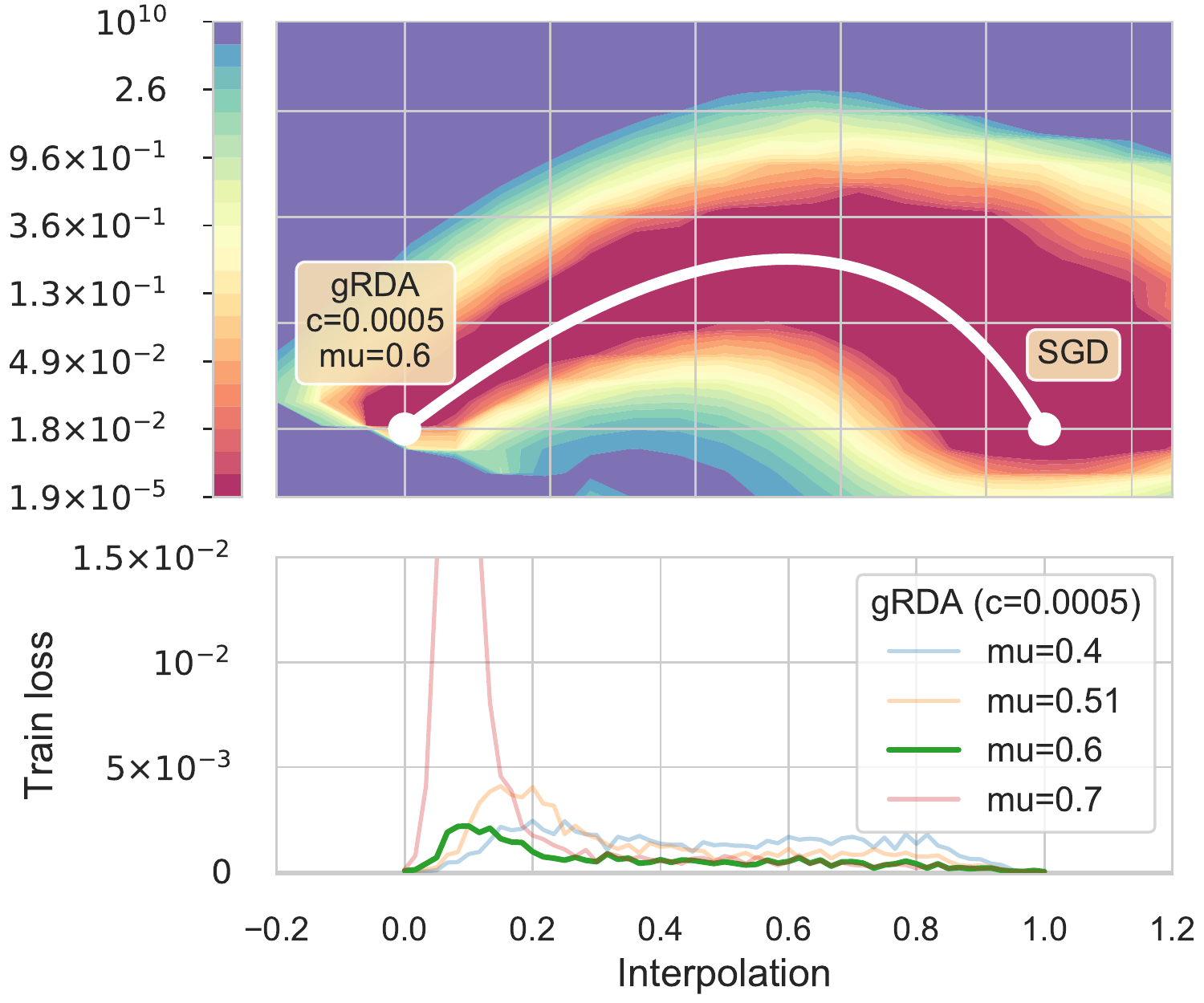}
		\caption{VGG16/CIFAR-10/Train loss}\label{fig:5a}		
	\end{subfigure}
	\quad
	\begin{subfigure}[t]{0.48\textwidth}
		\centering
		\includegraphics[width=\textwidth]{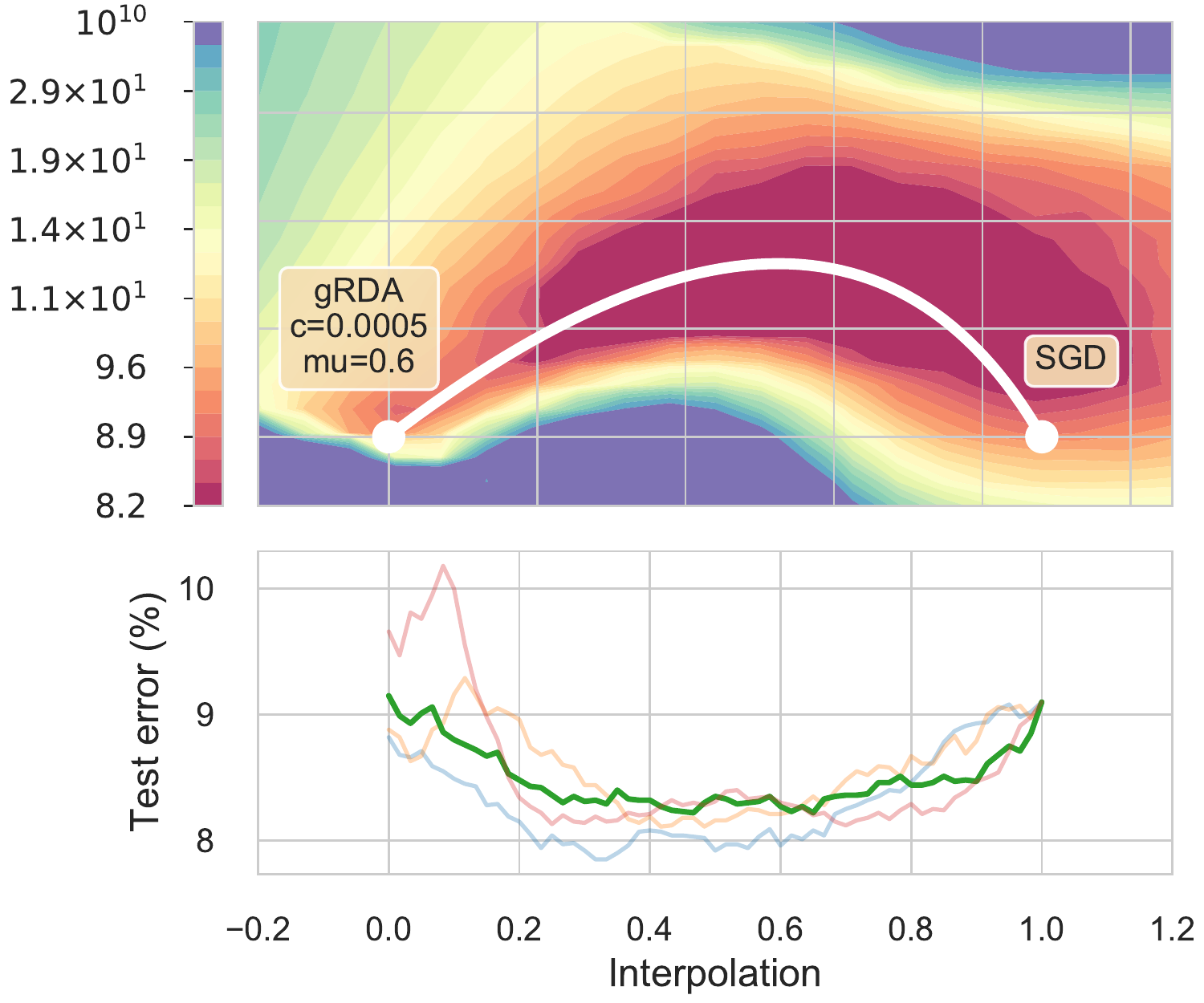}
		\caption{VGG16/CIFAR-10/Test error}\label{fig:5b}
	\end{subfigure}\\
	\begin{subfigure}[t]{0.48\textwidth}
		\centering
	\includegraphics[width=\textwidth]{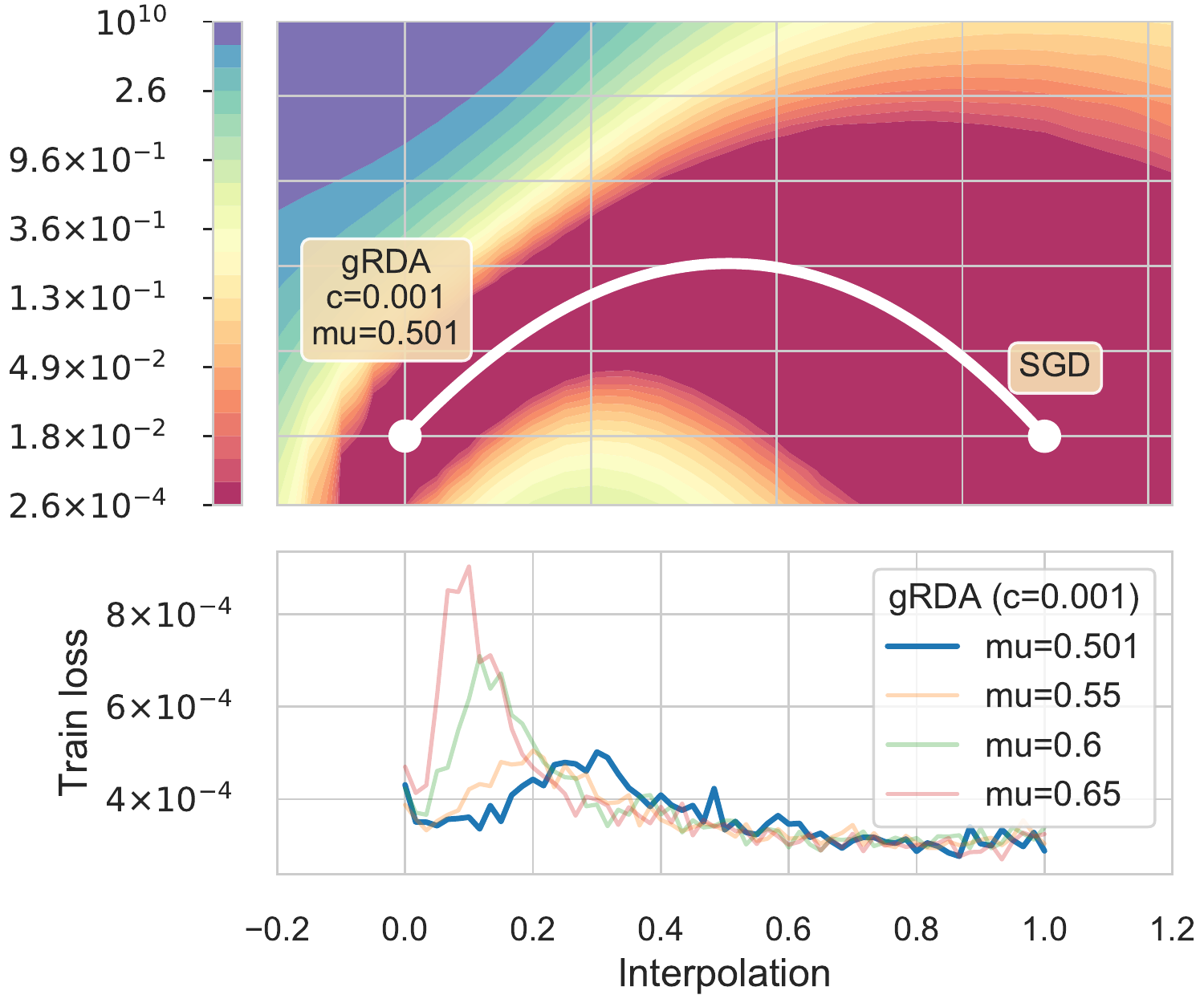}
		\caption{WRN28x10/CIFAR-100/Train loss}\label{fig:6a}		
	\end{subfigure}
	\quad
	\begin{subfigure}[t]{0.48\textwidth}
		\centering
	\includegraphics[width=\textwidth]{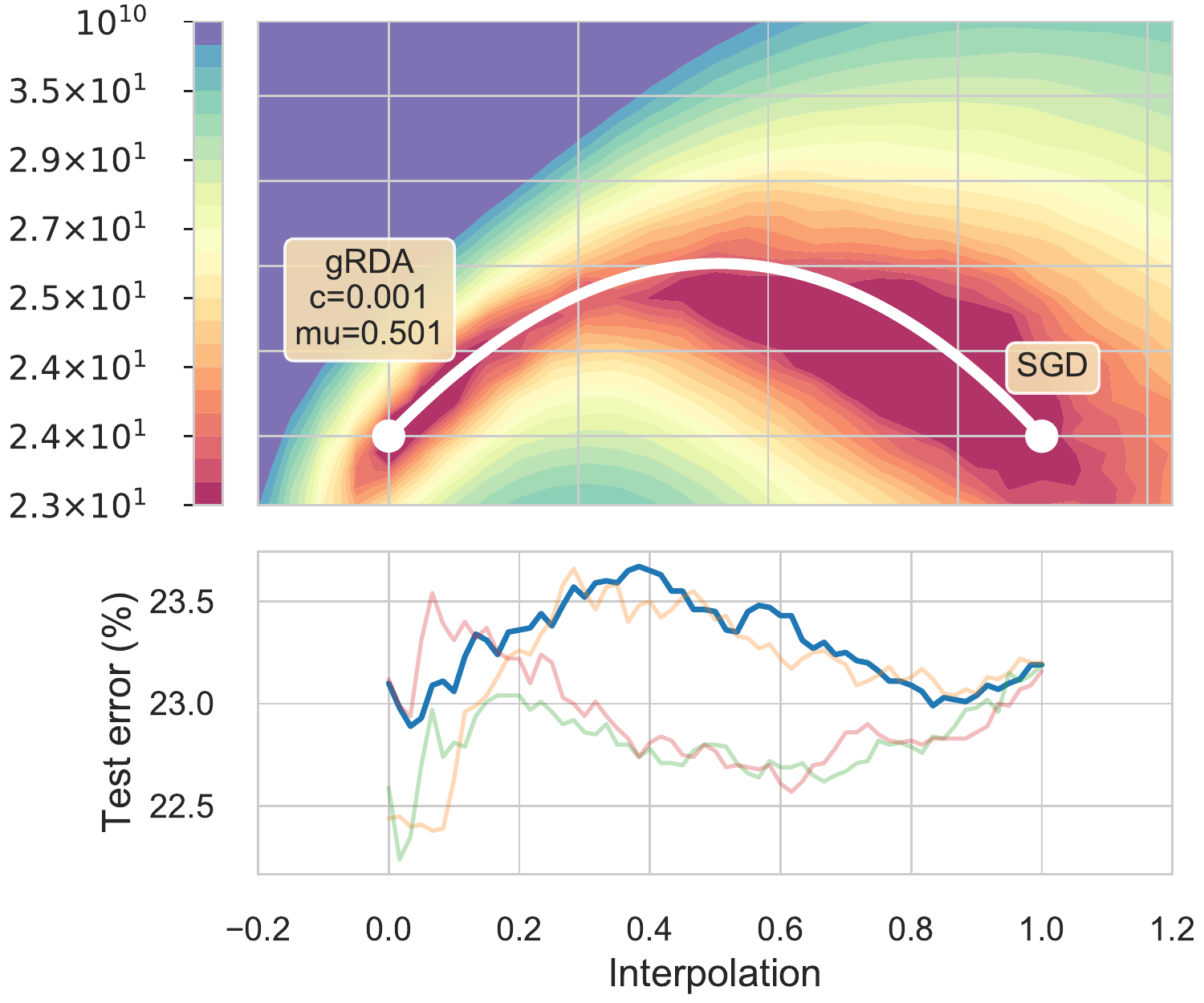}
		\caption{WRN28x10/CIFAR-100/Test error}\label{fig:6b}
	\end{subfigure}
 \caption{The upper figure in each panel shows the contour of training loss and testing error on the hyperplane containing the minimal loss B\'ezier curve (white) interpolating the minimizers found by the SGD and the gRDA. The lower plot of each panel shows the training loss/testing error on the minimal loss B\'ezier curve interpolating minimizers of SGD and gRDA under different $\mu$.
 }\label{fig:conn}
\end{figure}

\subsection{Direction of $\bw^{gRDA}-\bw^{SGD}$}\label{sec:proj} 

The directional pruning (Definition \ref{def:dp}) implies that the vector $\Delta_n:=\bw_n^{gRDA}-\bw_n^{SGD}$ should lie in $\Pc_0$ as $n\to\infty$ if tuned appropriately. Unfortunately, checking this empirically requires estimating $\Pc_0$ which is computationally formidable. Nonetheless, there exists a dominating low dimensional subspace in $\Pc_0^\perp$ (the subspace orthogonal to $\Pc_0$); particularly, a few studies \cite{Sagun16,Sagun18,GKX19,P18} have empirically shown that for various networks on the CIFAR-10, the magnitude of the ten leading eigenvalues of $H(\bw^{SGD})$ are dominating the others. 

Let $\Pc_n^{top}:=\mbox{span}\{\ub_{1,n},\ub_{2,n},\ldots,\ub_{10,n}\}$ be the top subspace spanned by the eigenvectors $\ub_{j,n}$ associated with the top 10 eigenvalues of $H(\bw_n^{SGD})$. Define
\begin{align}
	P_n := \left[
	\begin{array}{ccc}
		\longleftarrow &\ub_{1,n} &\longrightarrow \\
		\longleftarrow &\ub_{2,n} &\longrightarrow \\
		&\vdots&\\
		\longleftarrow &\ub_{10,n} &\longrightarrow \\		
	\end{array}
	\right].\label{eq:P}
\end{align}

We train the VGG16 and WRN28x10 on the CIFAR-10, until the training data are nearly interpolated and the training loss is almost zero. During the training process, we fix the initializer and minibatches when we use different optimizers to ensure the comparability. We compute $P_n$ on the training trajectory of VGG16 and WRN28x10. See Section \ref{app:proj} for details on the computation of these eigenvectors. We test the hypothesis that the proportion of $\Delta_n$ in $\Pc_n^{top}$ is low, i.e. $\|P_n \Delta_n\|/\|\Delta_n\|$ is low. The results from the VGG16 and WRN28x10 in Figure \ref{fig:proj} basically confirm this hypothesis, as the magnitude of the proportion of $\Delta_n$ in $\Pc_n^{top}$ is very small under the two networks. Particularly, the proportion is always very small for WRN28x10. The results for different $\mu$ are similar, showing that $\Delta_n$ is pointing to the same direction regardless of $\mu$. 

\begin{figure}[!h]
		\centering 
		\includegraphics[width=0.48\textwidth]{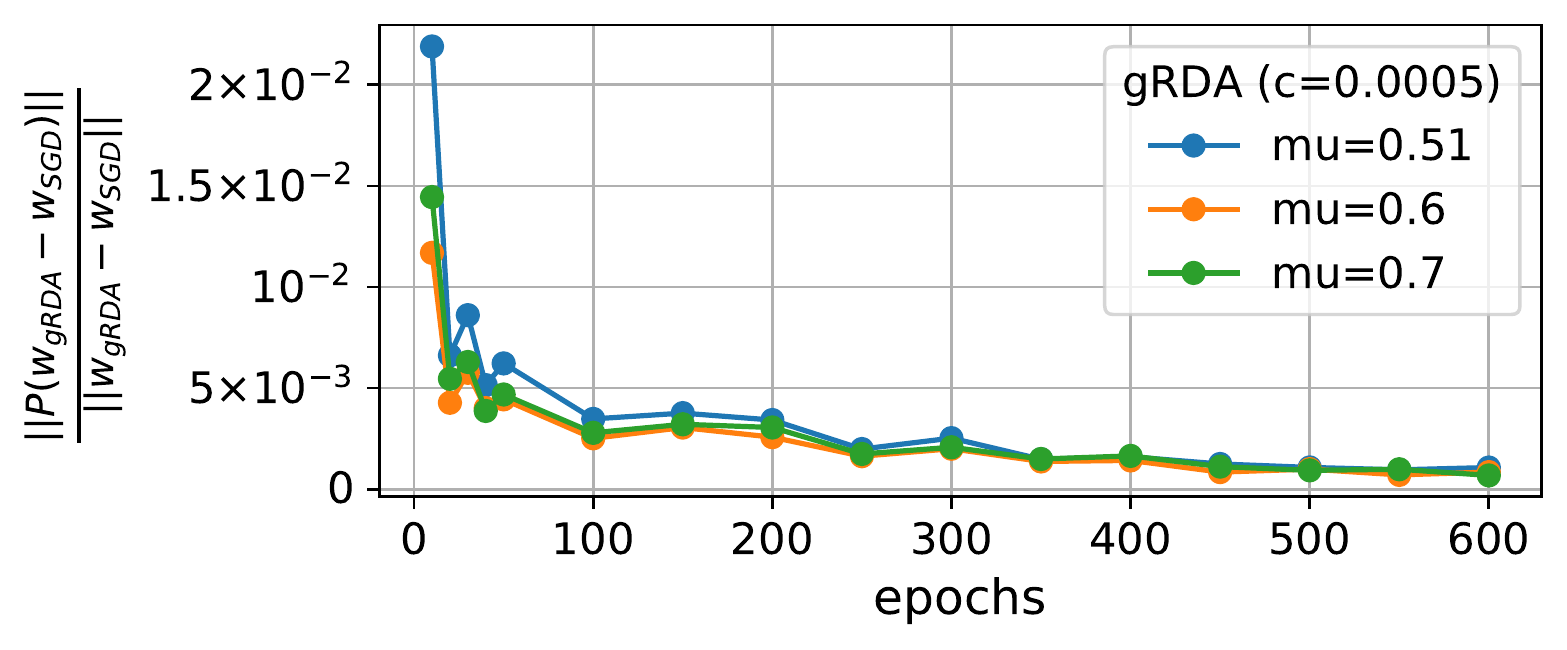}
		\includegraphics[width=0.48\textwidth]{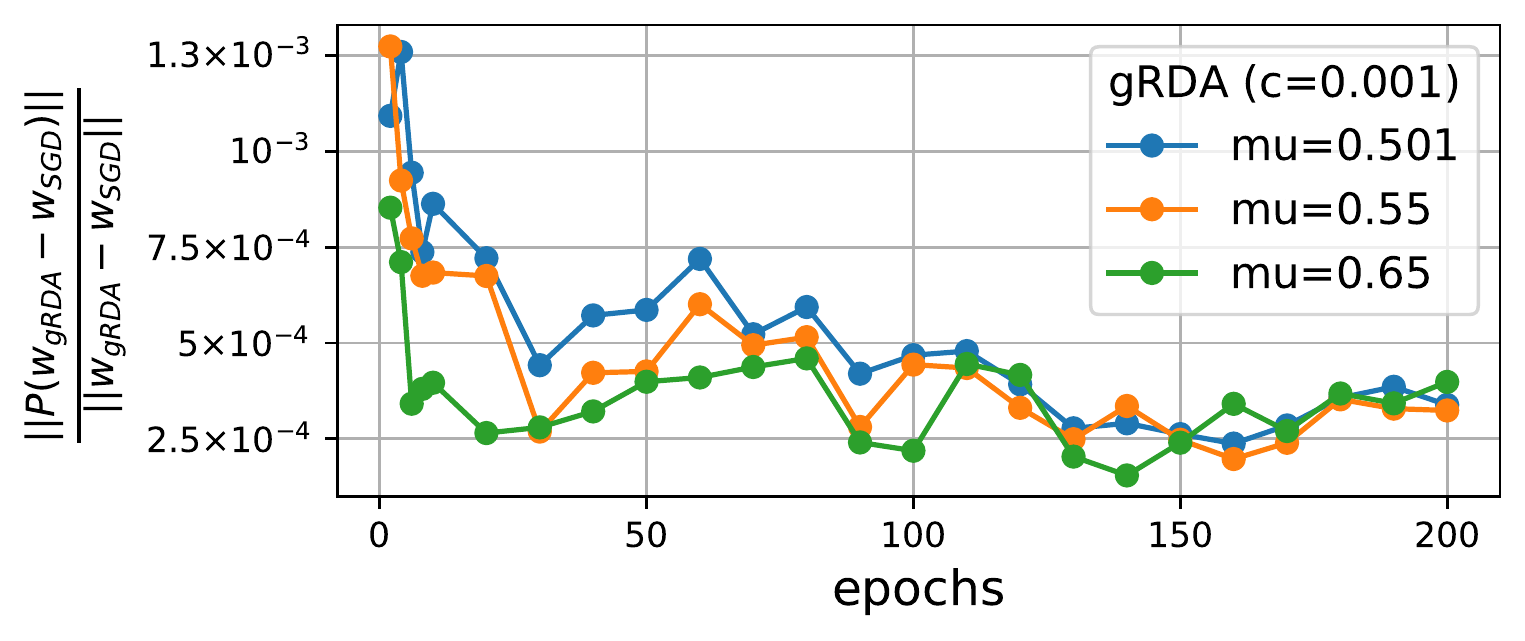}
 \caption{The fraction of the different between SGD and gRDA on the eigenspace associated with the leading 10 eigenvalues. \textbf{Left:} VGG16. \textbf{Right:} WRN28x10. The $\|\cdot\|$ is the $\ell_2$ norm.}\label{fig:proj} 
\vspace{-0.3cm}
\end{figure}

\section{Discussion and future work}
We propose the novel directional pruning for deep neural networks, that aims to prune DNNs while preserving the training accuracy. For implementation, we show that \eqref{eq:grda} asymptotically achieves the directional pruning after sufficient epochs of training. Empirical evidence shows that our solution yields a accuracy and sparsity tradeoff within the range of many contemporary pruning techniques.

The testing accuracy of \eqref{eq:grda} is almost always higher than \eqref{eq:sgd} if $\mu$ is slightly greater than 0.5 when using the ResNets, and some interpolation between the minima found by \eqref{eq:grda} and \eqref{eq:sgd} often has a better testing accuracy than the two minima; see Figure \ref{fig:conn}. As suggested by Figure \ref{fig:proj}, \eqref{eq:grda} appears to deviate from \eqref{eq:sgd} in the flatter directions. These evidences support \cite{HHY19}, who argue that the valley of minima is actually asymmetric, and points on the flatter side tend to generalize better. We think a further study of the testing accuracy of \eqref{eq:grda} along the lines initiated in this work may be an interesting future research topic, as this would shed some light on the mystery of generalization.

\section*{Broader Impact}
Our paper belongs to the cluster of works focusing on efficient and resource-aware deep learning. There are numerous positive impacts of these works, including the reduction of memory footprint and computational time, so that deep neural networks can be deployed on devices equipped with less capable computing units, e.g. the microcontroller units. In addition, we help facilitate on-device deep learning, which could replace traditional cloud computation and foster the protection of privacy.

Popularization of deep learning, which our research helps facilitate, may result in some negative societal consequences. For example, the unemployment may increase due to the increased automation enabled by the deep learning.

\section*{Acknowledgments}
We thank the anonymous reviewers for the helpful comments. Shih-Kang Chao would like to acknowledge the financial support from the Research Council of the University of Missouri. This work was completed while Guang Cheng was a member of the Institute for Advanced Study, Princeton in the fall of 2019. Guang Cheng would like to acknowledge the hospitality of the IAS and the computational resource it has provided.



{
\bibliographystyle{plain}
\bibliography{sDNN}}

\begin{thebibliography}{10}

\bibitem{mitreview20}
10 breakthrough technologies, February 2020.

\bibitem{ACGH19}
Sanjeev Arora, Nadav Cohen, Noah Golowich, and Wei Hu.
\newblock A convergence analysis of gradient descent for deep linear neural
  networks.
\newblock In {\em International Conference on Learning Representations}, 2019.

\bibitem{ACH18}
Sanjeev Arora, Nadav Cohen, and Elad Hazan.
\newblock On the optimization of deep networks: Implicit acceleration by
  overparameterization.
\newblock In Jennifer Dy and Andreas Krause, editors, {\em Proceedings of the
  35th International Conference on Machine Learning}, volume~80 of {\em
  Proceedings of Machine Learning Research}, pages 244--253,
  Stockholmsm{\"a}ssan, Stockholm Sweden, 10--15 Jul 2018. PMLR.

\bibitem{BHMM19}
Mikhail Belkin, Daniel Hsu, Siyuan Ma, and Soumik Mandal.
\newblock Reconciling modern machine-learning practice and the classical
  bias{\textendash}variance trade-off.
\newblock {\em Proceedings of the National Academy of Sciences},
  116(32):15849--15854, 2019.

\bibitem{BMP90}
Albert Benveniste, Michel M{\'e}tivier, and Pierre Priouret.
\newblock {\em Adaptive Algorithms and Stochastic Approximations}.
\newblock Springer, 1990.

\bibitem{BKS93}
J.~A. Bucklew, T.~G. Kurtz, and W.~A. Sethares.
\newblock Weak convergence and local stability properties of fixed step size
  recursive algorithms.
\newblock {\em IEEE Transactions on Information Theory}, 39(3):966--978, May
  1993.

\bibitem{YQ19}
Yuan Cao and Quanquan Gu.
\newblock Tight sample complexity of learning one-hidden-layer convolutional
  neural networks.
\newblock In H.~Wallach, H.~Larochelle, A.~Beygelzimer, F.~d\textquotesingle
  Alch\'{e}-Buc, E.~Fox, and R.~Garnett, editors, {\em Advances in Neural
  Information Processing Systems 32}, pages 10612--10622. Curran Associates,
  Inc., 2019.

\bibitem{CC19}
Shih-Kang Chao and Guang Cheng.
\newblock A generalization of regularized dual averaging and its dynamics.
\newblock {\em ArXiv Preprint arXiv:1909.10072}, 2019.

\bibitem{CWZZ18}
Y.~{Cheng}, D.~{Wang}, P.~{Zhou}, and T.~{Zhang}.
\newblock Model compression and acceleration for deep neural networks: The
  principles, progress, and challenges.
\newblock {\em IEEE Signal Processing Magazine}, 35(1):126--136, Jan 2018.

\bibitem{imagenet09}
J.~Deng, W.~Dong, R.~Socher, L.-J. Li, K.~Li, and F.-F. Li.
\newblock {ImageNet: A Large-Scale Hierarchical Image Database}.
\newblock In {\em CVPR09}, 2009.

\bibitem{Draxler18}
Felix Draxler, Kambis Veschgini, Manfred Salmhofer, and Fred Hamprecht.
\newblock Essentially no barriers in neural network energy landscape.
\newblock In Jennifer Dy and Andreas Krause, editors, {\em Proceedings of the
  35th International Conference on Machine Learning}, volume~80 of {\em
  Proceedings of Machine Learning Research}, pages 1309--1318,
  Stockholmsm{\"a}ssan, Stockholm Sweden, 10--15 Jul 2018. PMLR.

\bibitem{DLT18}
Simon~S. Du, Jason~D. Lee, and Yuandong Tian.
\newblock When is a convolutional filter easy to learn?
\newblock In {\em International Conference on Learning Representations}, 2018.

\bibitem{DZPS19}
Simon~S. Du, Xiyu Zhai, Barnabas Poczos, and Aarti Singh.
\newblock Gradient descent provably optimizes over-parameterized neural
  networks.
\newblock In {\em International Conference on Learning Representations}, 2019.

\bibitem{E89}
M.~S.~P. Eastham.
\newblock {\em The asymptotic solution of linear differential systems,
  applications of the Levinson theorem}.
\newblock Clarendon Press, Oxford, 1989.

\bibitem{evci19}
Utku Evci, Trevor Gale, Jacob Menick, Pablo~Samuel Castro, and Erich Elsen.
\newblock Rigging the lottery: Making all tickets winners.
\newblock {\em To appear in ICML 2020}, 2020.

\bibitem{EPGE19}
Utku Evci, Fabian Pedregosa, Aidan Gomez, and Erich Elsen.
\newblock The difficulty of training sparse neural networks.
\newblock {\em ArXiv Preprint ArXiv:1906.10732}, 2019.

\bibitem{FC19}
Jonathan Frankle and Michael Carbin.
\newblock The lottery ticket hypothesis: Finding sparse, trainable neural
  networks.
\newblock In {\em International Conference on Learning Representations}, 2019.

\bibitem{FDRC19}
Jonathan Frankle, Gintare~Karolina Dziugaite, Daniel~M. Roy, and Michael
  Carbin.
\newblock Stabilizing the lottery ticket hypothesis.
\newblock {\em ArXiv Preprint arXiv:1903.01611}, 2019.

\bibitem{GEH19}
Trevor Gale, Erich Elsen, and Sara Hooker.
\newblock The state of sparsity in deep neural networks.
\newblock {\em ArXiv Preprint Arxiv 1902.09574}, 2019.

\bibitem{Garipov18}
Timur Garipov, Pavel Izmailov, Dmitrii Podoprikhin, Dmitry~P Vetrov, and
  Andrew~G Wilson.
\newblock Loss surfaces, mode connectivity, and fast ensembling of dnns.
\newblock In S.~Bengio, H.~Wallach, H.~Larochelle, K.~Grauman, N.~Cesa-Bianchi,
  and R.~Garnett, editors, {\em Advances in Neural Information Processing
  Systems 31}, pages 8789--8798. Curran Associates, Inc., 2018.

\bibitem{GKX19}
Behrooz Ghorbani, Shankar Krishnan, and Ying Xiao.
\newblock An investigation into neural net optimization via hessian eigenvalue
  density.
\newblock In Kamalika Chaudhuri and Ruslan Salakhutdinov, editors, {\em
  Proceedings of the 36th International Conference on Machine Learning},
  volume~97 of {\em Proceedings of Machine Learning Research}, pages
  2232--2241, Long Beach, California, USA, 09--15 Jun 2019. PMLR.

\bibitem{hessian-eigenthings}
Noah Golmant, Zhewei Yao, Amir Gholami, Michael Mahoney, and Joseph Gonzalez.
\newblock pytorch-hessian-eigentings: efficient pytorch hessian
  eigendecomposition, October 2018.

\bibitem{Gotmare18}
Akhilesh Gotmare, Nitish~Shirish Keskar, Caiming Xiong, and Richard Socher.
\newblock Using mode connectivity for loss landscape analysis.
\newblock {\em ArXiv}, abs/1806.06977, 2018.

\bibitem{GLSS18}
Suriya Gunasekar, Jason Lee, Daniel Soudry, and Nathan Srebro.
\newblock Characterizing implicit bias in terms of optimization geometry.
\newblock In Jennifer Dy and Andreas Krause, editors, {\em Proceedings of the
  35th International Conference on Machine Learning}, volume~80 of {\em
  Proceedings of Machine Learning Research}, pages 1832--1841,
  Stockholmsm{\"a}ssan, Stockholm Sweden, 10--15 Jul 2018. PMLR.

\bibitem{GRD18}
Guy Gur-Ari, Daniel~A. Roberts, and Ethan Dyer.
\newblock Gradient descent happens in a tiny subspace.
\newblock {\em ArXiv Preprint Arxiv 1812.04754}, 2018.

\bibitem{han2015compression}
Song Han, Huizi Mao, and William~J. Dally.
\newblock Deep compression: Compressing deep neural networks with pruning,
  trained quantization and huffman coding.
\newblock {\em ArXiv Preprint Arxiv 1510.00149}, 2015.
\newblock cite arxiv:1510.00149Comment: Published as a conference paper at ICLR
  2016 (oral).

\bibitem{HPTD15}
Song Han, Jeff Pool, John Tran, and William~J. Dally.
\newblock Learning both weights and connections for efficient neural networks.
\newblock In {\em Proceedings of the 28th International Conference on Neural
  Information Processing Systems - Volume 1}, NIPS'15, pages 1135--1143,
  Cambridge, MA, USA, 2015. MIT Press.

\bibitem{HSW93}
B.~{Hassibi}, D.~G. {Stork}, and G.~J. {Wolff}.
\newblock Optimal brain surgeon and general network pruning.
\newblock In {\em IEEE International Conference on Neural Networks}, pages
  293--299 vol.1, 1993.

\bibitem{BS93}
Babak Hassibi and David~G. Stork.
\newblock Second order derivatives for network pruning: Optimal brain surgeon.
\newblock In S.~J. Hanson, J.~D. Cowan, and C.~L. Giles, editors, {\em Advances
  in Neural Information Processing Systems 5}, pages 164--171. Morgan-Kaufmann,
  1993.

\bibitem{HHY19}
Haowei He, Gao Huang, and Yang Yuan.
\newblock Asymmetric valleys: Beyond sharp and flat local minima.
\newblock In H.~Wallach, H.~Larochelle, A.~Beygelzimer, F.~d\textquotesingle
  Alch\'{e}-Buc, E.~Fox, and R.~Garnett, editors, {\em Advances in Neural
  Information Processing Systems 32}, pages 2549--2560. Curran Associates,
  Inc., 2019.

\bibitem{HZRS15}
Kaiming He, Xiangyu Zhang, Shaoqing Ren, and Jian Sun.
\newblock Deep residual learning for image recognition.
\newblock {\em 2016 IEEE Conference on Computer Vision and Pattern Recognition
  (CVPR)}, pages 770--778, 2015.

\bibitem{He18}
Yihui He, Ji~Lin, Zhijian Liu, Hanrui Wang, Li{-}Jia Li, and Song Han.
\newblock {AMC:} automl for model compression and acceleration on mobile
  devices.
\newblock In Vittorio Ferrari, Martial Hebert, Cristian Sminchisescu, and Yair
  Weiss, editors, {\em Computer Vision - {ECCV} 2018 - 15th European
  Conference, Munich, Germany, September 8-14, 2018, Proceedings, Part {VII}},
  volume 11211 of {\em Lecture Notes in Computer Science}, pages 815--832.
  Springer, 2018.

\bibitem{Izmailov18}
Pavel Izmailov, Dmitrii Podoprikhin, Timur Garipov, Dmitry~P. Vetrov, and
  Andrew~Gordon Wilson.
\newblock Averaging weights leads to wider optima and better generalization.
\newblock In {\em UAI}, 2018.

\bibitem{JT19}
Ziwei Ji and Matus Telgarsky.
\newblock The implicit bias of gradient descent on nonseparable data.
\newblock In Alina Beygelzimer and Daniel Hsu, editors, {\em Proceedings of the
  Thirty-Second Conference on Learning Theory}, volume~99 of {\em Proceedings
  of Machine Learning Research}, pages 1772--1798, Phoenix, USA, 25--28 Jun
  2019. PMLR.

\bibitem{KS98}
Ioannis Karatzas and Steven Shreve.
\newblock {\em Brownian Motion and Stochastic Calculus}, volume 113 of {\em
  Graduate Texts in Mathematics}.
\newblock Springer, New York, 1998.

\bibitem{krizhevsky2009learning}
Alex Krizhevsky.
\newblock Learning multiple layers of features from tiny images.
\newblock Technical Report TR-2009, University of Toronto, 2009.

\bibitem{LDS90}
Yann LeCun, John~S. Denker, and Sara~A. Solla.
\newblock Optimal brain damage.
\newblock In D.~S. Touretzky, editor, {\em Advances in Neural Information
  Processing Systems 2}, pages 598--605. Morgan-Kaufmann, 1990.

\bibitem{Lee19}
Jaehoon Lee, Lechao Xiao, Samuel~S. Schoenholz, Yasaman Bahri~Roman Novak,
  Jascha Sohl-Dickstein, and Jeffrey Pennington.
\newblock Wide neural networks of any depth evolve as linear models under
  gradient descent.
\newblock {\em ArXiv Preprint Arxiv 1902.06720}, 2019.

\bibitem{autofis}
Bin Liu, Chenxu Zhu, Guilin Li, Weinan Zhang, Jincai Lai, Ruiming Tang,
  Xiuqiang He, Zhengguo Li, and Yong Yu.
\newblock {AutoFIS}: Automatic feature interaction selection in factorization
  models for click-through rate prediction.
\newblock {\em KDD 2020}, 2020.

\bibitem{LSTHD19}
Zhuang Liu, Mingjie Sun, Tinghui Zhou, Gao Huang, and Trevor Darrell.
\newblock Rethinking the value of network pruning.
\newblock In {\em International Conference on Learning Representations}, 2019.

\bibitem{LL20}
Kaifeng Lyu and Jian Li.
\newblock Gradient descent maximizes the margin of homogeneous neural networks.
\newblock In {\em International Conference on Learning Representations}, 2020.

\bibitem{pmlr-v70-molchanov17a}
Dmitry Molchanov, Arsenii Ashukha, and Dmitry Vetrov.
\newblock Variational dropout sparsifies deep neural networks.
\newblock In Doina Precup and Yee~Whye Teh, editors, {\em Proceedings of the
  34th International Conference on Machine Learning}, volume~70 of {\em
  Proceedings of Machine Learning Research}, pages 2498--2507, International
  Convention Centre, Sydney, Australia, 06--11 Aug 2017. PMLR.

\bibitem{Molchanov17}
Pavlo Molchanov, Stephen Tyree, Tero Karras, Timo Aila, and Jan Kautz.
\newblock Pruning convolutional neural networks for resource efficient
  inference.
\newblock In {\em ICLR}, 2017.

\bibitem{N19}
Quynh Nguyen.
\newblock On connected sublevel sets in deep learning.
\newblock In Kamalika Chaudhuri and Ruslan Salakhutdinov, editors, {\em
  Proceedings of the 36th International Conference on Machine Learning},
  volume~97 of {\em Proceedings of Machine Learning Research}, pages
  4790--4799, Long Beach, California, USA, 09--15 Jun 2019. PMLR.

\bibitem{OCC15}
Francesco Orabona, Koby Crammer, and Nicol{\`o} Cesa-Bianchi.
\newblock A generalized online mirror descent with applications to
  classification and regression.
\newblock {\em Machine Learning}, 99(3):411--435, Jun 2015.

\bibitem{OS19}
Samet Oymak and Mahdi Soltanolkotabi.
\newblock Overparameterized nonlinear learning: Gradient descent takes the
  shortest path?
\newblock In Kamalika Chaudhuri and Ruslan Salakhutdinov, editors, {\em
  Proceedings of the 36th International Conference on Machine Learning},
  volume~97 of {\em Proceedings of Machine Learning Research}, pages
  4951--4960, Long Beach, California, USA, 09--15 Jun 2019. PMLR.

\bibitem{P18}
Vardan Papyan.
\newblock The full spectrum of deepnet hessians at scale: Dynamics with sgd
  training and sample size.
\newblock {\em ArXiv Preprint Arxiv 1811.07062}, 2018.

\bibitem{P19}
Vardan Papyan.
\newblock Measurements of three-level hierarchical structure in the outliers in
  the spectrum of deepnet hessians.
\newblock In Kamalika Chaudhuri and Ruslan Salakhutdinov, editors, {\em
  Proceedings of the 36th International Conference on Machine Learning},
  volume~97 of {\em Proceedings of Machine Learning Research}, pages
  5012--5021, Long Beach, California, USA, 09--15 Jun 2019. PMLR.

\bibitem{NEURIPS2019_9015}
Adam Paszke, Sam Gross, Francisco Massa, Adam Lerer, James Bradbury, Gregory
  Chanan, Trevor Killeen, Zeming Lin, Natalia Gimelshein, Luca Antiga, Alban
  Desmaison, Andreas Kopf, Edward Yang, Zachary DeVito, Martin Raison, Alykhan
  Tejani, Sasank Chilamkurthy, Benoit Steiner, Lu~Fang, Junjie Bai, and Soumith
  Chintala.
\newblock Pytorch: An imperative style, high-performance deep learning library.
\newblock In H.~Wallach, H.~Larochelle, A.~Beygelzimer, F.~d\textquotesingle
  Alch\'{e}-Buc, E.~Fox, and R.~Garnett, editors, {\em Advances in Neural
  Information Processing Systems 32}, pages 8024--8035. Curran Associates,
  Inc., 2019.

\bibitem{P83}
A.~Pazy.
\newblock {\em Semigroups of Linear Operators and Applications to Partial
  Differential Equations}.
\newblock Springer-Verlag New York, 1983.

\bibitem{P96}
V.~I. Piterbarg.
\newblock {\em Asymptotic methods in the theory of Gaussian processes and
  fields}, volume 148 of {\em Translations of Mathematical Monographs}.
\newblock American Mathematical Society, Providence, RI., 1996.

\bibitem{imagenet15}
Olga Russakovsky, Jia Deng, Hao Su, Jonathan Krause, Sanjeev Satheesh, Sean Ma,
  Zhiheng Huang, Andrej Karpathy, Aditya Khosla, Michael Bernstein,
  Alexander~C. Berg, and Li~Fei-Fei.
\newblock Imagenet large scale visual recognition challenge.
\newblock {\em International Journal of Computer Vision}, 115(3):211--252, Dec
  2015.

\bibitem{Sagun16}
Levent Sagun, Leon Bottou, and Yann LeCun.
\newblock Eigenvalues of the hessian in deep learning: Singularity and beyond.
\newblock {\em ArXiv Preprint Arxiv 1611.07476}, 2016.

\bibitem{Sagun18}
Levent Sagun, Utku Evci, V.~Ugur Guney, Yann Dauphin, and Leon Bottou.
\newblock Empirical analysis of the hessian of over-parametrized neural
  networks.
\newblock {\em ICLR 2018 Workshop}, 2018.

\bibitem{vgg16}
Karen Simonyan and Andrew Zisserman.
\newblock Very deep convolutional networks for large-scale image recognition.
\newblock In {\em International Conference on Learning Representations}, 2015.

\bibitem{SENS18}
Daniel Soudry, Elad Hoffer, Mor~Shpigel Nacson, and Nathan Srebro.
\newblock The implicit bias of gradient descent on separable data.
\newblock In {\em International Conference on Learning Representations}, 2018.

\bibitem{T12}
Gerald Teschl.
\newblock {\em Ordinary Differential Equations and Dynamical Systems}, volume
  140 of {\em Graduate Studies in Mathematics}.
\newblock American Mathematical Society, Providence, Rhode Island, 2012.

\bibitem{WS19}
Mingwei Wei and David~J Schwab.
\newblock How noise affects the hessian spectrum in overparameterized neural
  networks.
\newblock {\em ArXiv Preprint Arxiv 1910.00195}, 2019.

\bibitem{X10}
Lin Xiao.
\newblock Dual averaging methods for regularized stochastic learning and online
  optimization.
\newblock {\em Journal of Machine Learning Research}, 11:2543--2596, 2010.

\bibitem{zagoruyko2016wide}
Sergey Zagoruyko and Nikos Komodakis.
\newblock Wide residual networks.
\newblock In Edwin R.~Hancock Richard C.~Wilson and William A.~P. Smith,
  editors, {\em Proceedings of the British Machine Vision Conference (BMVC)},
  pages 87.1--87.12. BMVA Press, September 2016.

\bibitem{Xiao19}
Xiao Zhang, Yaodong Yu, Lingxiao Wang, and Quanquan Gu.
\newblock Learning one-hidden-layer relu networks via gradient descent.
\newblock In Kamalika Chaudhuri and Masashi Sugiyama, editors, {\em Proceedings
  of Machine Learning Research}, volume~89 of {\em Proceedings of Machine
  Learning Research}, pages 1524--1534. PMLR, 16--18 Apr 2019.

\bibitem{ZG17}
Michael~H. Zhu and Suyog Gupta.
\newblock To prune, or not to prune: Exploring the efficacy of pruning for
  model compression.
\newblock {\em CoRR}, 2017.

\end{thebibliography}

\newpage
\vskip 2em \centerline{\Large \bf APPENDIX} \vskip -1em
\setcounter{subsection}{0}
\renewcommand{\thesubsection}{A.\arabic{subsection}}
\setcounter{equation}{0}
\renewcommand{\theequation}{A.\arabic{equation}}
\setcounter{theo}{0}
\renewcommand{\thetheo}{A.\arabic{theo}}\vskip 2em

\begin{prop}\label{prop:uniq_sol_def}
    Consider the optimization problem 
	\begin{align}
		\arg\min_{w_j\in\R^d} \bigg\{f(w_j):= \frac{1}{2}\|w^{SGD}_j - w_j\|_2^2 + \lambda s_j |w_j|\bigg\}. \label{eq:grdamin1}
	\end{align}
	For $w^{SGD}_j \in \R \textbackslash \{0\}, s_j \in \R, \lambda > 0$, it has an explicit solution:
	\begin{align}
		\widehat{w}_j = \sgn(w^{SGD}_j) \big[|w^{SGD}_j| - \lambda s_j\big]_+.
	\end{align}
\end{prop}
\begin{proof}[Proposition \ref{prop:uniq_sol_def}] \\
When $s_j = 0$, the solution is $\widehat{w}_j = w^{SGD}_j$. \\
When $s_j > 0$, the objective function is convex, therefore we only need to verify if $0$ is a subgradient of $f(w_j)$ at $\widehat{w}_j$. 
\begin{itemize}
    \item If $|w^{SGD}_j| > \lambda s_j$, $\widehat{w}_j = w^{SGD}_j - \lambda s_j \sgn(w^{SGD}_j)$. We can see that $\sgn(\widehat{w}_j) = \sgn(w^{SGD}_j)$, and since $\sgn(w_j)$ is a subgradient of $|w_j|$, we have $\widehat{w}_j - w^{SGD}_j + \lambda s_j \sgn(w^{SGD}_j) = 0$ as a subgradient of $f(w_j)$ at $\widehat{w}_j$.
    \item If $|w^{SGD}_j| \leq \lambda s_j$, $\widehat{w}_j = 0$. Since the subgradient set of $|w_j|$ is $[-1,1]$ at $w_j=0$, we have $0 \in [w^{SGD}_j - \lambda s_j, w^{SGD}_j + \lambda s_j] \Longleftrightarrow 0 \in [\widehat{w}_j - w^{SGD}_j - \lambda s_j, \widehat{w}_j - w^{SGD}_j + \lambda s_j]$ (the subgradient set of $f(w_j)$ at $w_j=0$).
\end{itemize}
When $s_j < 0$, the objective function is not convex, therefore we need to check the values of $f(w_j)$ at stationary points, boundary points, and non-differentiable points ($w_j=0$). Since the absolute value function $g(x)=|x|$ is $x$ when $x>0$ and $-x$ when $x<0$, we will find the possible stationary points at $w_j>0$ and $w_j<0$ separately, and $f(w_j)$ is smooth and strongly convex on each of the two parts. 

Without loss of generality, we first assume $w^{SGD}_j > 0$:
\begin{itemize}
    \item On $w_j > 0$, $f(w_j)=\frac{1}{2}\|w^{SGD}_j - w_j\|_2^2 + \lambda s_j w_j$. The stationary point is $w^{SGD}_j - \lambda s_j$ with objective function value $\frac{(\lambda s_j)^2}{2} + \lambda s_j |w^{SGD}_j - \lambda s_j|$;
    \item On $w_j < 0$, $f(w_j)=\frac{1}{2}\|w^{SGD}_j - w_j\|_2^2 - \lambda s_j w_j$. The stationary point is $ w^{SGD}_j + \lambda s_j$ (if it exists) with objective function value  $\widehat{w}_j$ is $\frac{(\lambda s_j)^2}{2} + \lambda s_j |w^{SGD}_j + \lambda s_j|$; note that if $w^{SGD}_j + \lambda s_j \geq 0$, then there is no stationary point in $(-\infty, 0)$;
    \item At $w_j = 0$, the objective function value is $\frac{(w^{SGD}_j)^2}{2}$.
\end{itemize}
Since $w^{SGD}_j > 0$ and $\lambda s_j < 0$, we have $$
\frac{(\lambda s_j)^2}{2} + \lambda s_j |w^{SGD}_j - \lambda s_j| > \frac{(\lambda s_j)^2}{2} + \lambda s_j |w^{SGD}_j + \lambda s_j|,$$
We also have $\frac{(\lambda s_j)^2}{2} + \lambda s_j |w^{SGD}_j - \lambda s_j| = \lambda s_j w^{SGD}_j - \frac{(\lambda s_j)^2}{2} < 0 < \frac{(w^{SGD}_j)^2}{2}$. 
Therefore the global minimizer of $f(w_j)$ is the right stationary point $\widehat{w}_j = w^{SGD}_j - \lambda s_j = \sgn(w^{SGD}_j) \max(0, |w^{SGD}_j| - \lambda s_j)$. Similar analysis holds for $w^{SGD}_j < 0$.
\end{proof}

\renewcommand{\thesection}{B}
\setcounter{subsection}{0}
\renewcommand{\thesubsection}{B.\arabic{subsection}}
\setcounter{equation}{0}
\renewcommand{\theequation}{B.\arabic{equation}}
\setcounter{theo}{0}
\renewcommand{\thetheo}{B.\arabic{theo}}

\section{Proof of theorem}

\subsection{Preliminary results}
The first result shows that $\bw_\gamma(t)$ converges in the functional space $D([0,T])^{d}$ for any $T>0$ in probability, where $D([0,T])$ is the space of all functions on $[0,T]$ that are right continuous with left limit. Denote $\pto$ the convergence in probability. The following result is immediate following by \cite{CC19}.
\begin{theo}[Asymptotic trajectory]\label{th:at} Suppose \ref{as:M} and \ref{as:L} hold, and the solution of gradient flow in \eqref{eq:gf} is unique, then as $\gamma\to 0$, $\bw_\gamma\pto\bw$ in $D([0,T])^d$ for any $T>0$, where $\bw(t)$ is the gradient flow.
\end{theo}
\vspace{-0.5em}
The asymptotic trajectory of the dual process $\bv_{n}$ and primal process $\bw_n$ are {the same}, i.e. they are both $\bw$. The key reason is that the threshold $g(n,\gamma)$ in $\Sc_{g(n,\gamma)}(\cdot)$ in \eqref{eq:grda} tends to zero: $\sup_{t\in[0,T]}\lim_{\gamma\to 0}\big|g(\itg,\gamma)\big|=0$, so $\bv = \bw$ in the limit. 

\begin{proof}[Theorem \ref{th:at}]
	The proof is an application of Theorem 3.13(a) of \cite{CC19}. 
\end{proof}


The asymptotic trajectory is deterministic, which cannot explain the stochasticity of the learning dynamics. However, in practice, the stochasticity of sampling minibatches has great influence on the quality of training. 

We investigate how the stochasticity enters \eqref{eq:grda}. \eqref{eq:grda} can be written in the stochastic mirror descent (SMD) representation \cite{OCC15}: 
\begin{align}\tag{\texttt{SMD}}
	\begin{split}\label{eq:smd}
	\bw_{n+1} &= \Sc_{g(n,\gamma)}\big(\bv_{n+1}\big),\\
	\mbox{ where }\bv_{n+1} &= \bv_n - \gamma \nabla f(\bw_n;Z_{n+1})\\
	\end{split}
\end{align}
The process $\bv_n$ is an auxiliary process in the dual space (generated by the gradients), and the primal process $\bw_n$, corresponding to the parameters of DNN, can be represented as a transformation of the dual process by $\Sc_{g(n,\gamma)}$. 

Random gradients enter $\bv_n$, while $\bw_n$ is obtained by taking a deterministic transformation of $\bv_n$. To characterize the randomness of $\bv_n$, consider $\bv_\gamma(t):=\bv_{\lfloor t/\gamma\rfloor}$ the piecewise constant interpolated process, where $\lfloor a\rfloor$ takes the greatest integer that is less than or equal to $a$. This next theorem provides the distribution of $\bv_\gamma(t)$.

\begin{theo}[Distributional dynamics]\label{th:dd}
	Suppose \ref{as:M}, \ref{as:L} and \ref{as:H} hold. 
	In addition, suppose the root of the coordinates in $\bw(t)$ occur at time $\{T_k\}_{k=1}^\infty \subset [0,\infty)$. Let $\bw_0$ with $w_{0,j}\neq 0$ (e.g. from a normal distribution) and $T_0=0$. 
	Then, as $\gamma$ is small, for $t\in (T_K,T_{K+1})$,
		\begin{align}\label{eq:vg}
			\bv_\gamma(t) \dap \bg^\dagger(t) + \bw(t) - \sqrt{\gamma} \zd(t) + \sqrt{\gamma}\int_{0}^{t} \Phi(t,s)^\top \Sigma^{1/2}(\bw(s))d\BB(s),
		\end{align}
		where $\dap$ denotes approximately in distribution, $\Sigma(\bw)$ is the covariance kernel defined in \eqref{eq:sig} and $\BB(s)$ is a $d$-dimensional standard Brownian motion, 
		and
		\begin{align}
			\bg^\dagger(t)&:=\sqrt{\gamma} c t^{\mu} \sgn(\bw(T_K^+))\label{eq:gdag}\\
			\zd(t)&:= c \mu \int_0^t s^{\mu-1} \Phi(t,s)\sgn(\bw(s)) ds + c \sum_{k=1}^K \Big\{\Phi(t,T_k) \big\{\sgn(\bw(T_k^+))-\sgn(\bw(T_k^-))\big\} T_k^\mu\Big\}  \label{eq:zd}
		\end{align}
		 $\Phi(t,s)\in\R^{d\times d}$ is the principal matrix solution (Chapter 3.4 of \cite{T12}) of the ODE system:
		\begin{align}
			d\bx(t) = -H(\bw(t)) \bx(t) dt,\quad \bx(0)=\bx_0. \label{eq:odex1}
		\end{align}
\end{theo}
The proof follows by a functional central limit theorem in \cite{CC19} for Markov processes generated by regularized stochastic algorithms.

\begin{proof}[Theorem \ref{th:dd}]
	Consider the centered and scaled processes
		\begin{align}
			\VV_\gamma(t) &:= \frac{\bv_\gamma(t)-\bw(t)}{\sqrt{\gamma}}.\label{eq:scagrda} 
		\end{align}
By Theorem 3.13(b) of \cite{CC19}, $\VV_\gamma \dap \VV$ on $(T_k,T_{k+1})$ for each $k=0,\ldots,K$ as $\gamma$ is small, where $\VV$ obeys the stochastic differential equation (SDE): 
		\begin{align}
			d\VV(t)&=-H(\bw(t))\big[\VV(t)-\sgn(\bw(t))ct^\mu\big]dt+\Sigma^{1/2}(\bw(t)) d\BB(t),\label{eq:Vgrda}
		\end{align}
		with initial $\VV(T_k) = \VV(T_k^-)$, and $\BB(t)$ is the $d$-dimensional standard Brownian motion. Note that $\VV(T_0)=\VV(0)=\VV_\gamma(0)=0$ almost surely.

	Under condition \ref{as:H}, the global Lipschitz and linear growth conditions hold, so there exists a unique strong solution of \eqref{eq:Vgrda} by Theorem 5.2.9 of \cite{KS98}. 

	In addition, by condition \ref{as:H}, the solution operator $\Phi(t,s)$ of the inhomogeneous ODE system,
	\begin{align}
		d\bx(t) = -H(\bw(t)) \bx(t) dt,\quad \bx(s)=\bx_s,\label{eq:asymode}
	\end{align}
	uniquely exists, and the solution is $\bx(t)=\Phi(t,s)\bx_s$ by Theorem 5.1 of \citep{P83}. $\Phi(t,s)$ satisfies the properties in Theorem 5.2 of \citep{P83}; in particular, for all $0<s<r<t$, 
	\begin{align}
		(s,t) &\mapsto \Phi(t,s) \mbox{ is continuous} \label{eq:opconti}\\
		\Phi(t,t)&=I_d  \label{eq:op0} \\
		\frac{\partial}{\partial t}\Phi(t,s) &= -H(\bw(t))\Phi(t,s), \label{eq:op1}\\
		\frac{\partial}{\partial s}\Phi(t,s) &= \Phi(t,s) H(\bw(s)),\label{eq:op2}\\
		 \Phi(t,s)&=\Phi(t,r)\Phi(r,s).\label{eq:combine}
	\end{align}
	Recall from \eqref{eq:Vgrda} that for $t\in(T_k,T_{k+1})$,
	\begin{align}
		d\VV(t)&=-H(\bw(t)) \big[\VV(t)-\sgn(\bw(t))ct^\mu\big]dt+\Sigma^{1/2}(\bw(t)) d\BB(t),\label{eq:Vrepeat}
	\end{align}
	with initial distribution $\VV(T_k^-)$. It can be verified by \eqref{eq:op1} and Ito calculus that for $t\in (T_k,T_{k+1})$,
	\begin{align}
		\VV(t) = \Phi(t,T_k) \VV(T_k^-) + \int_{T_k}^{t} \Phi(t,s) H(\bw(s)) \sgn(\bw(s)) cs^\mu ds + \int_{T_k}^{t} \Phi(t,s) \Sigma^{1/2}(\bw(s)) d\BB(s) \label{eq:solV}
	\end{align}
	is the solution of \eqref{eq:Vrepeat}. 

	Integration by part, \eqref{eq:op2} and \eqref{eq:op0} yield that
	\begin{align*}
		&\int_{T_k}^{t} \Phi(t,s) H(\bw(s)) \sgn(\bw(s)) cs^\mu ds \\
		&= c t^{\mu}\sgn(\bw(t)) - c T_k^{\mu}\Phi(t,T_k^+)\sgn(\bw(T_k^+)) - \int_{T_k}^t \Phi(t,s) c \mu s^{\mu-1} \sgn(\bw(s)) ds.
	\end{align*}
	If $t>T_K$, last display, induction, \eqref{eq:combine} and \eqref{eq:solV} imply
	\begin{align*}
		&\VV(t)\\
		&= \Phi(t,T_K) \VV(T_{K}^-) + \int_{T_K}^{t} \Phi(t,s) H(\bw(s)) \sgn(\bw(s)) cs^\mu ds + \int_{T_K}^{t} \Phi(t,s) \Sigma^{1/2}(\bw(s)) d\BB(s) \\
			&= c t^{\mu}\sgn(\bw(T_K^+)) + \Phi(t,T_K)\VV(T_{K}^-) - c T_K^{\mu}\Phi(t,T_K) \sgn(\bw(T_K^+)) \\
			&\quad\quad - \int_{T_K}^{t} \Phi(t,s) c \mu s^{\mu-1} \sgn(\bw(s)) ds+ \int_{T_K}^{t} \Phi(t,s) \Sigma^{1/2}(\bw(s)) d\BB(s)\\
			&= c t^{\mu}\sgn(\bw(T_K^+)) + \Phi(t,T_{K-1})\VV(T_{K-1}^-) - cT_K^{\mu} \Phi(t,T_K) \{\sgn(\bw(T_K^+))-\sgn(\bw(T_K^-))\}\\
			&\quad \quad-cT_{K-1}^{\mu} \Phi(t,T_{K-1}) \sgn(\bw(T_{K-1}^+))- \int_{T_{K-1}}^{t} \Phi(t,s) c \mu s^{\mu-1}\sgn(\bw(s)) ds\\
			&\quad \quad + \int_{T_{K-1}}^{t} \Phi(t,s) \Sigma^{1/2}(\bw(s)) d\BB(s)\\
			&\quad \quad\quad \quad\quad \quad\vdots\\
			&= c t^{\mu}\sgn(\bw(T_K^+)) + \Phi(t,T_{K-1})\underbrace{\VV(0)}_{\mbox{\scriptsize $=0$ a.s.}} - \zd(t) + \int_0^t \Phi(t,s) \Sigma^{1/2}(\bw(s)) d\BB(s),
	\end{align*}
	where $\zd(t)=\zd_1(t)+\zd_2(t)$ with
	\begin{align}
		\begin{split}\label{eq:zddecom}
			\zd_1(t)&:= c \mu \int_0^t s^{\mu-1} \Phi(t,s)\sgn(\bw(s)) ds,\\
			\zd_2(t)&:= c \sum_{k=1}^K \Big\{\Phi(t,T_k) \big\{\sgn(\bw(T_k^+))-\sgn(\bw(T_k^-))\big\} T_k^\mu\Big\}.
		\end{split}
	\end{align}
\end{proof}

\subsection{Proof of Theorem \ref{th:dp}}\label{sec:pfdp}

By virtue of Remark \ref{rem:def_dp}, it is enough to show that
	\begin{align}
		w_{\gamma,j}(t) \dap  \sgn\{w_j^{SGD}(t)\} \big\{|w_j^{SGD}(t)| - \lambda_{\gamma,t} \bar s_j\big\}_+
		\end{align}
		where $\lambda_{\gamma,t}=c \sqrt{\gamma}t^\mu$, and $\bar s_j = s_j + o(1)$. This is implied by the following theorem.
\begin{theo}\label{prop:purt}
	 Suppose \ref{as:M}-\ref{as:sign} hold. Assume that $\mu\in(0.5,1)$ and $c>0$ in \eqref{eq:tune}. In addition, if $\sgn\{w_j(t)\}=\sgn\{w_j^{SGD}(t)\}$ for all $j=1,\ldots,d$, then, for a sufficiently large $t>\bar T$, 
	\begin{align}
		w_{\gamma,j}(t) \dap  \sgn\{w_j^{SGD}(t)\} \big\{|w_j^{SGD}(t)| - \sgn\{w_j^{SGD}(t)\}\sqrt{\gamma}\delta_j(t)\big\}_+ \label{eq:sgdst}
		\end{align}
		where $\zd(t)$ has an explicit form in \eqref{eq:zddecom} in the appendix, and satisfies as $t\to\infty$,
		\begin{align} 
			\zd(t) = c t^\mu \Pi_0 \sgn(\bw(t)) + o(t^\mu) + O(t^{\mu-1}), \label{eq:eig_H1}
		\end{align}
		and $\Pi_0$ is the orthogonal projection on the eigenspace corresponding to zero eigenvalues of $\bar H$.
\end{theo}
	
The proof of \eqref{eq:sgdst} will be based on Theorem \ref{th:dd}, and \eqref{eq:eig_H1} relies on the Levinson theorem, which provides asymptotic solution of the ODE in \eqref{eq:odex}. 

\begin{proof}[Theorem \ref{prop:purt}]
	From \eqref{eq:vg}, 
	\begin{align*}
		\bv_\gamma(t) \dap \bg^\dagger(t) + \bw(t) - \sqrt{\gamma}\zd(t)+ \sqrt{\gamma}\bU(t).
	\end{align*}
	where we recall $\bg^\dagger(t):=\sqrt{\gamma} c t^{\mu} \sgn(\bw(T_K^+))$ and $\zd(t)$ in \eqref{eq:zd}.
	\begin{align}
		\bU(t):=\int_{0}^{t} \Phi(t,s)^\top \Sigma^{1/2}(\bw(s))d\BB(s).
	\end{align}
	\begin{itemize}[itemindent=33pt,leftmargin=0pt]
		\labitem{Step 1}{rwdst1} {\bf Show $\sgn\{\bv_\gamma(t)\}=\sgn\{\bg^\dagger(t)\}$ with high probability.} 

	The goal of this step is achieved if we show
	\begin{align}
		|\bg^\dagger(t)| > \sqrt{\gamma} \big|-\zd(t)+\bU(t)\big| \quad\mbox{($|\cdot|$ and $>$ are componentwise)}. \label{eq:pfrwd1}
	\end{align}
	To this end, we will show that
	\begin{align}
		|\bg^\dagger(t)|-\sqrt{\gamma}\big|\zd(t)\big|>\sqrt{\gamma} \big|\bU(t)\big|. \label{eq:pfrwd1_5}
	\end{align}
	Clearly, this implies \eqref{eq:pfrwd1}.

	Recall that $\zd(t)=\zd_1(t)+\zd_2(t)$ where $\zd_j$'s are defined in \eqref{eq:zddecom}. By \ref{as:pms},
	\begin{align}
		|g_j^\dagger(t)|-|\delta_{1,j}(t)| 
		&\geq c_1 t^{\mu},
	\end{align}
	for some $c_1>0$. On the other hand, $\delta_{2,j}(t)$ is defined in \eqref{eq:zddecom}. By \ref{as:pms} and \ref{as:sign},
	$$
	|\delta_{2,j}(t)| < C \bar T^\mu.
	$$
	Hence, $|\delta_{2,j}(t)|=O(1)$. 

From above, we get
\begin{align*}
	|\bg^\dagger(t)|-\sqrt{\gamma}\big|\zd(t)\big|> c_0 t^\mu.
\end{align*}
Using the fact that $\bU(t)$ as a Gaussian process grows like $t^{1/2}$ up to multiplicative logarithmic terms with high probability (Theorem D.4 of \cite{P96}), $\bU(t)$ is dominated by $c_0 t^\mu$ with $\mu>0.5$, the proof of \eqref{eq:pfrwd1_5} is complete.

		\labitem{Step 2}{rwdst2} {\bf Proof of \eqref{eq:sgdst}.}
	
	By the formulation \eqref{eq:grda}, $\bw_\gamma(t) = \Sc_{\bg^\dagger(t)}(\bv_\gamma(t))$. Hence, for $j\leq d$, note that $\bg^\dagger(t)=(g_1^\dagger(t),\ldots,g_d^\dagger(t))$, and the piecewise constant process of SGD (with the same minibatch sequence as gRDA):
	 \begin{align}
	 	w_{\gamma,j}^{SGD}(t) \dap w_j(t) + \sqrt{\gamma} U_j(t),\label{eq:sgdweak}
	 \end{align}
	 which can be obtained under the same assumptions in this Theorem by \cite{BKS93,BMP90}. 

	Hence,
	 \begin{align*}
	 	w_{\gamma,j}(t) &\dap \Sc_{g_j^\dagger(t)}\{v_{\gamma,j}(t)\} \\
		&\dap
	 	\begin{dcases}
	 		w_{\gamma,j}^{SGD}(t)-\sqrt{\gamma}\delta_j(t), &\mbox{ if }\sgn(w_j(t))\big\{w_j(t) -\sqrt{\gamma}\delta_j(t) + \sqrt{\gamma}  U_j(t)\big\}>0,\\
	 		0, &\mbox{otherwise}.
	 	\end{dcases}
	 \end{align*}

	For \eqref{eq:sgdst}, if $\sgn\{w_j(t)\}=\sgn\{w_j^{SGD}(t)\}$, $\sgn(w_j(t))\big\{w_j(t) -\sqrt{\gamma}\delta_j(t) + \sqrt{\gamma}  U_j(t)\big\}>0$ can be rewritten by using \eqref{eq:sgdweak} that
	\begin{align*}
		|w_{\gamma,j}^{SGD}(t)|>\sgn(w_j^{SGD}(t))\sqrt{\gamma}\delta_j(t).
	\end{align*}
	Thus, \eqref{eq:sgdst} follows.
\labitem{Step 3}{pfpurt2} {\bf Proof of \eqref{eq:eig_H1}.} The proof relies on the Levinson theorem \cite{E89} from the theory of asymptotic solution of ordinary differential equations. Note that $\bar H$ as a real symmetric matrix is diagonalizable, i.e. there exists orthonormal matrix $P$ and diagonal matrix $\Lambda$ with non-negative values such that $\bar H = P \Lambda P^{\top}$, where $\Lambda=\diag(\lambda_1,\ldots,\lambda_d)$, and the column vectors of $P$ are eigenvectors $\ub_j$. 

Let $a_t\to 0$ satisfying
\begin{align*}
	\int_t^\infty\|H(\bw(s))- \bar H\| ds = O(a_t).
\end{align*}

The Levinson theorem (Theorem 1.8.1 on page 34 of \cite{E89}), together with the estimation of the remainder term on page 15-16 of \cite{E89}, imply that the principal matrix solution $\Phi(t,s)$ in \eqref{eq:asymode} satisfies
\begin{align}
	\Phi(\tau,s) = P \big(I_d + O(a_\tau)\big) e^{-\Lambda(\tau-s)} P^\top = P_0 P_0^\top + O(e^{-\underline\lambda (\tau-s)}) + O(a_\tau),\label{eq:pms}
\end{align}
where $\underline\lambda$ is the least positive eigenvalue of $\bar H$, the column vectors of $P_0$ are eigenvectors associated with the zero eigenvalue. Clearly, $P_0P_0^\top=\sum_{j:\bar H\ub_j=0} \ub_j\ub_j^\top$.

Recall the time $\{T_k\}_{k=1}^\infty$ defined in Theorem \ref{th:dd}. By the condition of this Proposition, there exists $K\in\N$ such that $\sgn(\bw(t))=\sgn(\bw(T_K))$ for all $t>T_K$. Recall that $\zd(t)=\zd_1(t)+\zd_2(t)$ where $\zd_1(t)$ and $\zd_2(t)$ are defined in \eqref{eq:zddecom}. Then
\begin{align}
	\zd_2(t) &= c P_0P_0^\top \sum_{k=1}^K \big\{\sgn(\bw(T_k^+))-\sgn(\bw(T_k^-))\big\} T_k^\mu + O(e^{-\underline\lambda (t-T_K)}T_K^\mu)+O(a_t T_K^\mu)\notag \\
	 &= -c \mu P_0P_0^\top \Big(\int_0^{T_K} s^{\mu-1} \sgn(\bw(s)) ds - T_K^\mu \sgn(\bw(T_K^+))\Big) + O(e^{-\underline\lambda (t-T_K)}T_K^\mu)+O(a_t T_K^\mu). \label{eq:estd2}
\end{align}

On the other hand, inputing \eqref{eq:pms} into $\zd_1$,
\begin{align}
	\zd_1(t) &= c \mu \int_0^t s^{\mu-1} \Phi(t,s)\sgn(\bw(s)) ds \notag \\
			&= c \mu P_0 P_0^\top \int_0^t s^{\mu-1} \sgn(\bw(s)) ds + I(t) + II(t),\label{eq:estd1}
\end{align}
and note that
\begin{align*}
	I(t) &\lesssim \int_0^t s^{\mu-1} e^{-\underline\lambda (t-s)}\|\sgn(\bw(s))\| ds \leq d^{1/2} \int_0^t s^{\mu-1} e^{-\underline\lambda (t-s)} ds = O(t^{\mu-1}),\\
	II(t) &\lesssim  \int_0^t s^{\mu-1}a_t \sgn(\bw(s)) ds = O(t^\mu a_t),
\end{align*}
where the bound of $I$ is obtained by using similar arguments as the proof of Theorem 4.2 of \cite{CC19} provided that $\mu<1$. The bound for $II(t)$ is elementary.

Note that $t^\mu a_t > T_K^\mu a_t$ by $\mu>0$ and $t>T_K$, and that $e^{-\underline\lambda (t-T_K)}\to 0$ exponentially in $t$ as $T_K$ is fixed. Combining \eqref{eq:estd2} and \eqref{eq:estd1} yield
\begin{align*}
	\zd_1(t) + \zd_2(t) = c t^\mu P_0 P_0^\top \sgn(\bw(t)) + O\big(\max\big\{t^\mu a_t, t^{\mu-1}\big\}\big),
\end{align*}
where $P_0P_0^\top$ is a projection matrix projecting on the subspace spanned by the eigenvectors corresponding to zero eigenvalues. Set $\Pi_0 = P_0P_0^\top$.
\end{itemize}
\end{proof}

\renewcommand{\thesection}{C}
\setcounter{subsection}{0}
\renewcommand{\thesubsection}{C.\arabic{subsection}}

\section{Algorithms for implementation}

\subsection{Basic version with a constant learning rate}\label{app:algorithm}
\begin{algorithm}[H]
\SetKwInOut{Parameter}{Hyperparameters}
\SetKwInOut{Initialization}{Initialization}
 \Parameter{$\gamma$: learning rate}
 \Parameter{$c\in[0,\infty], \mu\in(0.5, 1)$: fixed parameters in $g(n,\gamma)= c \gamma^{1/2}(n\gamma)^\mu$}
 \Initialization{$n\leftarrow 0$: iteration number}
 \Initialization{$\bw_{0}$: initial parameters}
 \Initialization{$G_0\leftarrow \bw_{0}$: accumulator of gradients}
 \While{Testing accuracy not converged}{
  $n\leftarrow n+1$\;
  $G_n\leftarrow G_{n-1} + \gamma \nabla f_j(\bw_{n-1};Z_{n})$\;
  $\bw_{n}\leftarrow \mathrm{sign}(G_n) \max(0, |G_n| - g(n,\gamma))$ \tcp*{entry-wise soft-thresholding} 
 }
 \caption{Generalized Regularized Dual Averaging (gRDA) with $\ell_1$ penalty}\label{alg1}
\end{algorithm}

\subsection{Modified tuning function for constant-and-drop learning rate}\label{sec:cd}
In practice, a ``constant-and-drop'' learning rate schedule is usually adopted. For example, the default learning rate schedule in the PyTorch implementation of ResNet on ImageNet is divided by 10 folds for every 30 epochs.\footnote{\url{https://github.com/pytorch/examples/blob/234bcff4a2d8480f156799e6b9baae06f7ddc96a/imagenet/main.py\#L400}} In this case, we replace Algorithm \ref{alg1} by Algorithm \ref{alg2} below, where we set the solf-thresholding level $\tilde{g}(n)$ that accumulates the increments of $g(n,\gamma)$ at every iteration.


\begin{algorithm}[H]
\SetKwInOut{Parameter}{Hyperparameters}
\SetKwInOut{Initialization}{Initialization}
 \Parameter{$\{\gamma_n\}$: learning rate schedule}
 \Parameter{$c\in[0,\infty], \mu\in(0.5, 1)$: fixed parameters in $g(n,\gamma)= c \gamma^{1/2}(n\gamma)^\mu$}
 \Initialization{$n\leftarrow 0$: iteration number}
 \Initialization{$\bw_{0}$: initial parameters}
 \Initialization{$G_0\leftarrow \bw_{0}$: accumulator of gradients}
 \Initialization{$\tilde{g}(0)\leftarrow 0$: accumulator of thresholds}
 \While{Testing accuracy not converged}{
  $n\leftarrow n+1$\;
  $G_n\leftarrow G_{n-1} + \gamma_n \nabla f_j(\bw_{n-1};Z_{n})$\;
  $\tilde{g}(n)\leftarrow \tilde{g}(n-1) + (g(n,\gamma_n) - g(n-1,\gamma_n))$ \tcp*{threshold increment for $\gamma_n$} 
  $\bw_{n}\leftarrow \mathrm{sign}(G_n) \max(0, |G_n| - \tilde{g}(n))$ \;
 }
 \caption{gRDA with constant-and-drop learning rates}\label{alg2}
\end{algorithm}

\renewcommand{\thesection}{C}
\setcounter{subsection}{0}
\renewcommand{\thesubsection}{C.\arabic{subsection}}
\setcounter{equation}{0}
\renewcommand{\theequation}{C.\arabic{equation}}
\setcounter{theo}{0}
\renewcommand{\thetheo}{C.\arabic{theo}}

\section{Details on numerical analysis}
We did all experiments in this paper using servers with 2 GPUs (Nvidia Tesla P100 or V100, 16GB memory), 2 CPUs (each with 12 cores, Intel Xeon Gold 6126), and 192 GB memory. We use PyTorch \cite{NEURIPS2019_9015} for all experiments.

\subsection{Details for experiments on ImageNet}\label{ap:imagenet_details}
We use the codes from PyTorch official implementation\footnote{\url{https://github.com/pytorch/examples/blob/234bcff4a2d8480f156799e6b9baae06f7ddc96a/imagenet/main.py}} 
of training ResNet-50 on ImageNet. 
The batch size used in all ImageNet experiments is 256 (the default value for training ResNet-50) and the data preprocessing module in the original codes is used as well. We follow the separation of training and validation dataset in the official setting of \texttt{ILSVRC2012} task (1281167 images in training and 50000 images in validation).

Figure \ref{fig:resnet50} presents the training accuracy, testing accuracy as well as sparsity.
Note that the state-of-the-art performance of ResNet50 on ImageNet (top-1 accuracy 77.15\% \cite{HZRS15}) using the SGD with momentum and weight decay is higher than the basic SGD (top-1 accuracy around 68\% as shown in Figure \ref{fig:resnet50}). This is because we fix the learning rate at 0.1, and run SGD without momentum or weight decay. 
Compared with the SGD, gRDA has a lower training accuracy but a slightly higher testing accuracy. When we increase $\mu$, the training accuracy decreases since larger $\mu$ induces 
higher sparsity. However, the testing accuracy for all choices of $\mu$'s in gRDA are similar.  

\begin{figure}[H]
    \begin{minipage}{0.5\textwidth}
    	\includegraphics[width=\textwidth, trim={0 0 0 45}, clip]{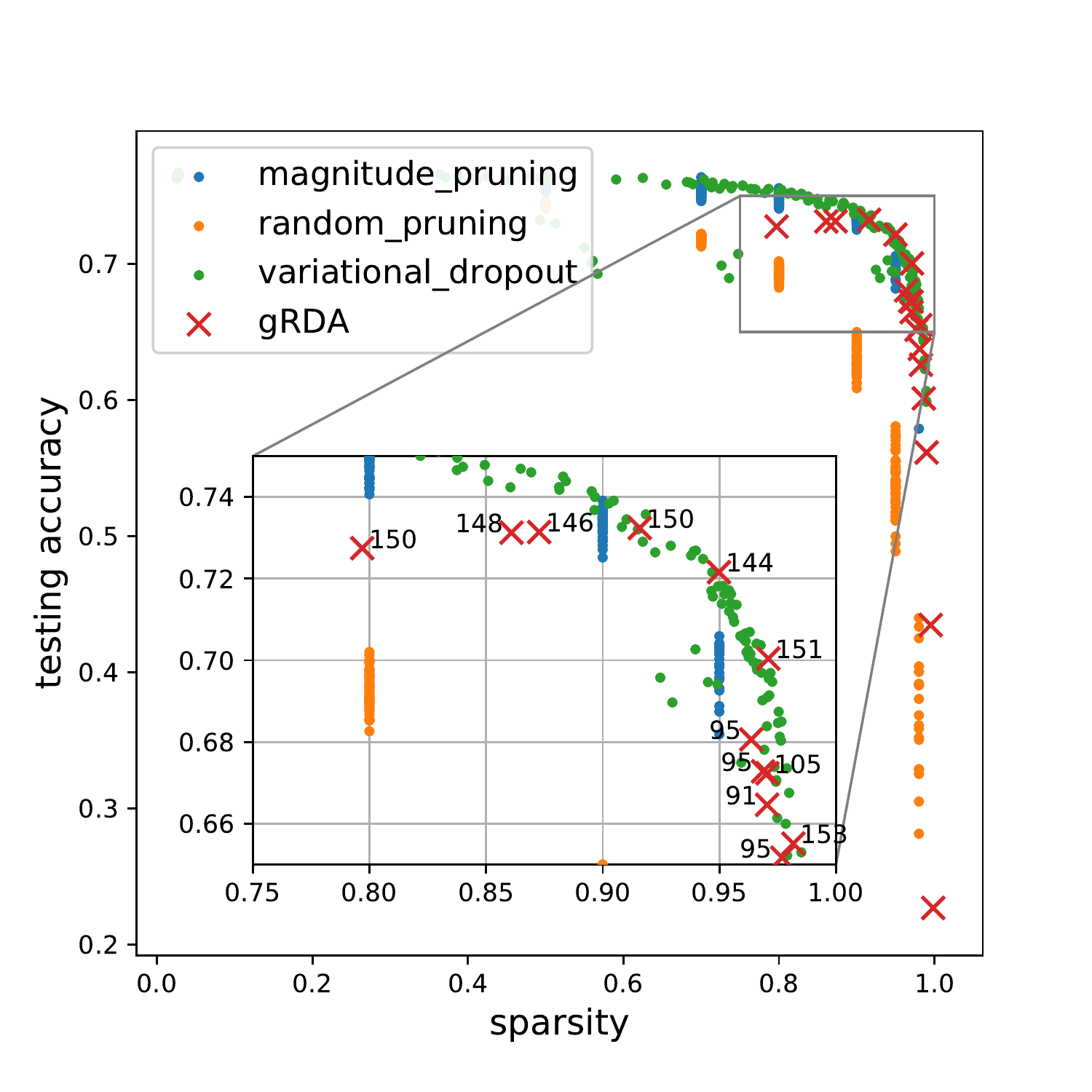}
        \caption{A comparison of gRDA with the magnitude pruning \cite{ZG17} and variational dropout \cite{pmlr-v70-molchanov17a} with ResNet50 on ImageNet. The numbers next to the red crosses are the epochs. }\label{fig:app_compareresnet50}
    \end{minipage}
    \hfill
    \begin{minipage}{0.48\textwidth}
        \footnotesize
        \captionof{table}{The parameters for gRDA in Figure \ref{fig:app_compareresnet50}. }\label{tab:app_compareresnet50}
    	\begin{tabular}{lllll}
    	    \toprule
            $c$     & $\mu$    & Epoch & \makecell{Sparsity \\ (\%)} & \makecell{Test Acc. \\ (\%)} \\\midrule
            \multicolumn{5}{c}{PyTorch Official Learning Rate}\vspace{0.05in}      \\
            0.005 & 0.85  & 85  & 99.84 & 22.69  \\
            0.005 & 0.8   & 90  & 99.51 & 43.47 \\
            0.005 & 0.65  & 91  & 97.05 & 66.46 \\
            0.005 & 0.75  & 92  & 98.99 & 56.15 \\
            0.005 & 0.7   & 94  & 98.26 & 62.60 \\
            0.004 & 0.7   & 95  & 97.69 & 65.17 \\
            0.003 & 0.7   & 95  & 96.87 & 67.28  \\
            0.004 & 0.65  & 95  & 96.36 & 68.06 \\
            0.003 & 0.75  & 103 & 98.10 & 63.76  \\
            0.002 & 0.75  & 105 & 97.06 & 67.23 \\
            0.004 & 0.75  & 121 & 98.62 & 60.12 \\\midrule
            \multicolumn{5}{c}{Only Drop at Epoch 140}  \vspace{0.05in}    \\
            0.005 & 0.6   & 144 & 94.98 & 72.16 \\
            0.005 & 0.51  & 146 & 87.28 & 73.14  \\
            0.005 & 0.501 & 148 & 86.09 & 73.13 \\
            0.005 & 0.55  & 150 & 91.60 & 73.24  \\
            0.01  & 0.4   & 150 & 79.69 & 72.75 \\
            0.005 & 0.65  & 151 & 97.10 & 70.04 \\
            0.005 & 0.7   & 153 & 98.17 & 65.51 \\\bottomrule
        \end{tabular}
    \end{minipage}
\end{figure}

The left panel of Figure \ref{fig:compareresnet50} is reproduced from the bottom panel of Figure 3 in \cite{GEH19}, and we add the results of gRDA which are marked by the red crosses. The gRDA is performed using a ``constant-and-drop'' learning rate schedule. Concretely, $\gamma = 0.1$ for epoch 1 to 140, and $\gamma = 0.01$ for epoch after 140. Figure \ref{fig:app_compareresnet50} provides additional results of the gRDA using the learning rate schedule given in the PyTorch official implementation:\footnote{\url{https://github.com/pytorch/examples/blob/234bcff4a2d8480f156799e6b9baae06f7ddc96a/imagenet/main.py\#L400}} 
\begin{itemize}
	\item $\gamma = 0.1$ for epoch 1 to 30
	\item $\gamma = 0.01$ for epoch 31 to 60
	\item $\gamma = 0.001$ for epoch 61 to 90, and $\gamma = 0.0001$ for epoch after 90
\end{itemize}
We found that the gRDA relatively underperforms with this learning rate schedule. This schedule for the ImageNet is usually applied jointly with the SGD with Polyak's momentum. As we find that SGD without momentum only yields a test accuracy of 68.76\% for ImageNet-ResNet50 under this learning rate schedule, we suspect that the absence of momentum in \eqref{eq:grda} could be a reason for the underperformance.

The right panel of Figure \ref{fig:compareresnet50} shows the layerwise sparsity using several different pruning methods. The results of AutoML for Model Compression are from stage4 in Figure 3 of \cite{He18}. And the results of Variational Dropout are from \cite{GEH19}\footnote{\url{https://github.com/google-research/google-research/tree/master/state_of_sparsity}} and we choose the one with 90\% sparsity. The results of Erd\H{o}s-R\'enyi-Kernel are from Figure 12 (90 \% Sparse ERK, i.e. the subfigure on right) in \cite{evci19}.



\subsection{Settings of training models on CIFAR-10 and CIFAR-100} \label{append:cifar}
The two datasets CIFAR-10 and CIFAR-100 are described in \cite{krizhevsky2009learning}. Particularly, we follow the separation of training and validation dataset in the official setting (50000 images in training and 10000 images in validation for both CIFAR-10 and CIFAR-100). For our experiments on CIFAR-10 and CIFAR-100, we mostly follow the codes of \cite{Garipov18}.\footnote{\url{https://github.com/timgaripov/dnn-mode-connectivity}}
The batch size used in all experiments is 128 and the data preprocessing module in the original codes is used as well. 
We follow the example in \cite{Garipov18} and set \texttt{--use\_test}. For optimizers, we use \texttt{SGD(momentum=0,weight\_decay=0)} and \texttt{gRDA($c$,$\mu$)} as defaults. For the two architectures we used, VGG16, as in its vanilla version, does not have batch normalization, while WRN28x10 has batch normalization. 

For both SGD and gRDA, the base learning rate $\gamma$ and epochs are the same as summarized in Table \ref{tab:cifar10}. We follow the learning rate schedule adopted by \cite{Garipov18}: 
\begin{itemize}
    \item For the first 50\% of epochs, we use the base learning rate, i.e. $\gamma_i = \gamma, \text{ if } \frac{i}{n} \in [0,0.5)$;
    \item For 50\% to 90\% of epochs, the learning rate decreases linearly from the base learning rate to 1\% of the base learning rate, i.e. $\gamma_i = (1.0 - (\frac{i}{n} - 0.5) \frac{0.99}{0.4})\gamma, \text{ if } \frac{i}{n} \in [0.5,0.9)$;
    \item For the last 10\% of epochs, we keep using the 1\% of the base learning rate as learning rate, i.e. $\gamma_i = 0.01\gamma, \text{ if } \frac{i}{n} \in [0.9,1]$.
\end{itemize}

\begin{table}[!ht]
    \centering
    \caption{Details for training models on CIFAR-10 and CIFAR-100. The minibatch size is 128. Parameters not included in this table are selected as the default values in the code of \cite{Garipov18}.}
    \label{tab:cifar10}
    \begin{tabular}{ccccccc}
\toprule Data & Model & \makecell{Base\\Learning\\ Rate} & Epochs & Results & \makecell{Used in \\Section \ref{append:conn} \\ (connectivity)} & \makecell{Used in \\Section \ref{app:proj} \\ (projection)} \\\midrule
CIFAR-10 & VGG16  & 0.1 & 600 & \makecell{Figure \ref{fig:cifar10vgg} \\ Table \ref{tab:cifar10vgg}} & Yes & Yes \\\midrule
CIFAR-10 & WRN28x10  & 0.1 & 200 & \makecell{Figure \ref{fig:cifar10wideres} \\ Table \ref{tab:cifar10wideres}} & No & Yes \\\midrule
CIFAR-100 & WRN28x10 & 0.1 & 200 & \makecell{Figure \ref{fig:cifar100wideres} \\ Table \ref{tab:cifar100wideres}} & Yes & No \\\bottomrule
    \end{tabular}
\end{table}

We train our models with ten different seeds using both SGD and gRDA, and show the training accuracy/loss, testing accuracy/loss, and sparsity along the training process in Figure \ref{fig:cifar10vgg}, \ref{fig:cifar10wideres}, and \ref{fig:cifar100wideres} (as in Figure \ref{fig:resnet50}). Table \ref{tab:cifar10vgg}, \ref{tab:cifar10wideres} and \ref{tab:cifar100wideres} provide specific numbers for selected epochs.

For Figure \ref{fig:sgdsign}, we show the result of the first seed under the two settings: VGG16-CIFAR-10 (gRDA with $c=0.0005,\mu=0.51$) and WRN28x10-CIFAR-100 (gRDA with $c=0.001,\mu=0.501$). We also select other seeds among the ten seeds, and the curve nearly overlaps with each other. Therefore we only show the result of the first seed.


\subsection{Details for Section \ref{sec:conn}} \label{append:conn}

For the analysis of the connectivity between two neural networks, we follow \cite{Garipov18} to train a quadratic B\'ezier curve interpolating two fixed endpoints $\widehat \bw_1$ and $\widehat \bw_2$, which are parameters trained by the SGD and the gRDA, respectively. 
$\widehat{w}_1$ and $\widehat{w}_2$ are trained with 600 epochs for VGG16, and 200 epochs for WRN28x10. Instead of training the entire curve, we follow \cite{Garipov18} and train random points sampled from the curve between the two endpoints, i.e., we sample $t\sim \text{Uniform}(0,1)$ and generate a model with weights being $\theta_{\bw}(t)=\widehat \bw_1(1-t)^2+\widehat \bw_2t^2+2t(1-t)\bw$ with a trainable vector $\bw$ (initialized at $(\widehat \bw_1 + \widehat \bw_2)/2$), and train $\bw$ with the loss $\ell(\theta_{\bw}(t))$ at a fixed $t$ using the SGD to get $\widehat\bw_3$. 

We use the program in \cite{Garipov18} to produce Figure \ref{fig:conn}, and the settings are summarized in Table \ref{tab:curve}. Parameters that are more technical are set by the default values in the GitHub repository of \cite{Garipov18}. 
The top panels of Figure \ref{fig:conn} illustrate the training loss contour on the hyperplane determined by the ($\widehat{\bw}_1, \widehat{\bw}_2, \widehat\bw_3$). 
The bottom panels are obtained through the models on the curve, i.e. the model $\theta_{\widehat\bw_3}(t)$ for $t\in[0,1]$. 
More results are showing in Figure \ref{app:conn}.

	

\begin{table}[!ht]
    \centering
    \caption{Details for training quadratic B\'ezier curve on models with CIFAR-10 and CIFAR-100. Here, we use the SGD with momentum in the CIFAR-10 task because the SGD without momentum results in NaN during training. Parameters not included in this table are selected as the default values in the code of \cite{Garipov18}.}
    \label{tab:curve}
    \begin{tabular}{cccccc}
\toprule Data & Model & Learning Rate & Epochs & Momentum & Weight Decay\\\midrule
CIFAR-10 & VGG16 & 0.1 & 300 & 0.9 & 0  \\\midrule
CIFAR-100 & WRN28x10 & 0.1 & 200 & 0 & 0 \\\bottomrule
    \end{tabular}
\end{table}


\subsection{Details for Section \ref{sec:proj}} \label{app:proj}

We use the code from \cite{hessian-eigenthings}\footnote{\url{https://github.com/noahgolmant/pytorch-hessian-eigenthings}} 
to calculate the eigenvalues/eigenvectors of the Hessian of a deep neural network using training data. We set \texttt{mode="lanczos"} to use the Lanczos algorithm. It uses the \texttt{scipy.sparse.linalg.eigsh hook} to the ARPACK Lanczos algorithm to find the top $k$ eigenvalues/eigenvectors using batches of data. We set \texttt{full\_dataset=True} to use all data to calculate the eigenvalues. 

Our goal is to find the top $10$ positive eigenvalues and their associated eigenvectors. We use the default argument \texttt{which="LM"} in the Lanczos algorithm, which returns the top $k$ (assigned by the argument \texttt{num\_eigenthings=k}) eigenvalues with the largest magnitude which may contain negative ones. In our experiment, $k=30$ is large enough to contain the top $10$ positive eigenvalues. Although the Lanczos algorithm supports method \texttt{"LA"} to directly return top $k$ positive eigenvalues, from our experience, the results are always significantly less than the top $10$ positive eigenvalues chosen by the above procedure. 
We also replace the default \texttt{max\_steps=20} to $1000$ since in few cases the algorithm does not converge in $20$ steps. We use the default tolerance  \texttt{tol=1e-6}.

The DNNs used here are the same with those used in Section \ref{sec:conn} with the same initializations and the same minibatches.


\begin{figure}[!h]
		\centering
    \hspace{0.2cm}
	\includegraphics[width=\textwidth, trim={50 50 50 50}, clip]{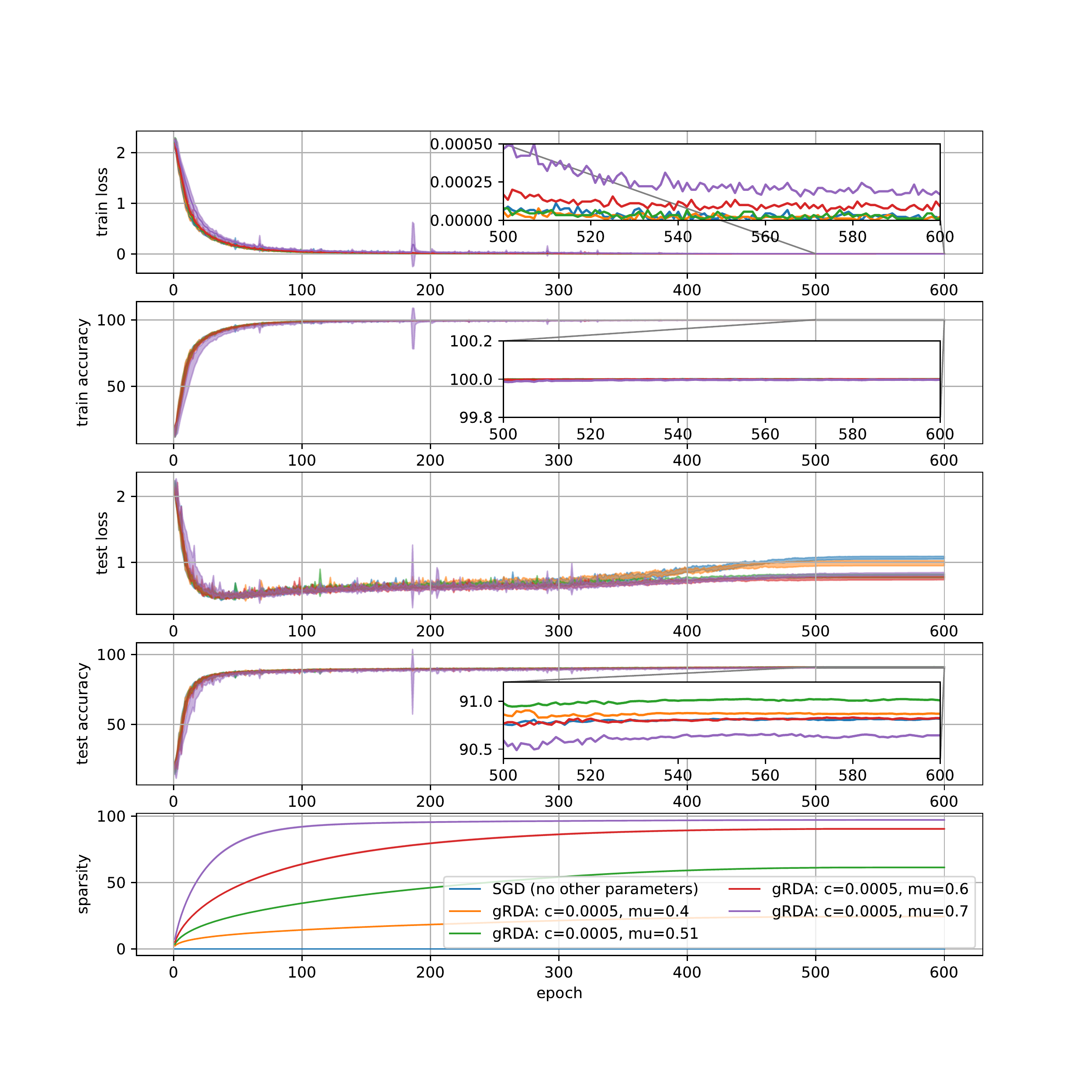} 
 \vspace{-0.4cm}
 \caption{Learning trajectories of \eqref{eq:sgd} and \eqref{eq:grda} for VGG16 on CIFAR-10. See Section \ref{append:cifar} for the selection of hyperparameters about training.}\label{fig:cifar10vgg}
\vspace{-0.3cm}
\end{figure}

\begin{figure}[!h]
		\centering
\hspace{0.2cm}
	\includegraphics[width=\textwidth, trim={50 50 50 50}, clip]{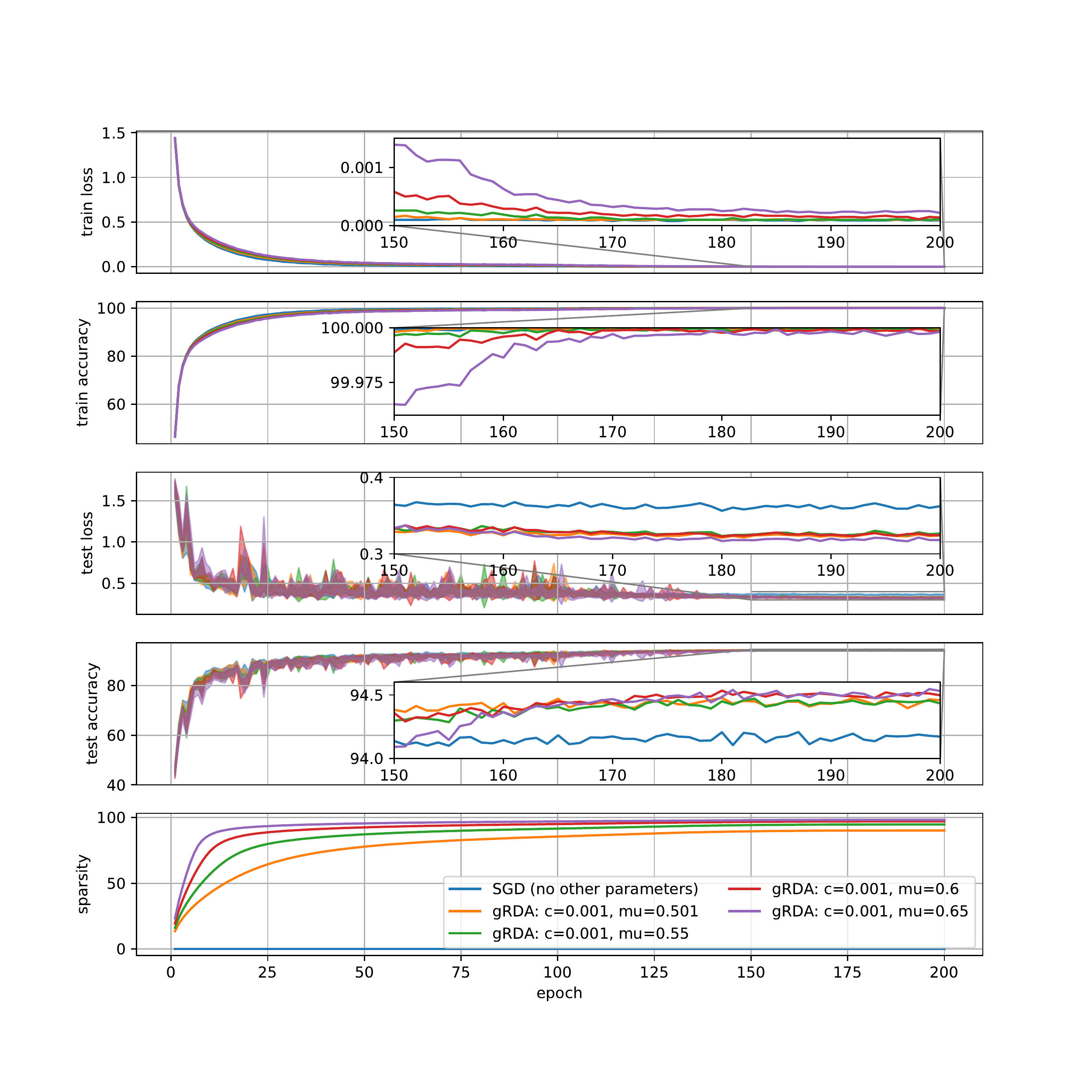}
 \vspace{-0.4cm}
 \caption{Learning trajectories of \eqref{eq:sgd} and \eqref{eq:grda} for WRN28x10 on CIFAR-10. See Section \ref{append:cifar} for the selection of hyperparameters about training.}\label{fig:cifar10wideres}
\vspace{-0.3cm}
\end{figure}

\begin{figure}[!h]
		\centering
\hspace{0.2cm}
	\includegraphics[width=\textwidth, trim={50 50 50 50}, clip]{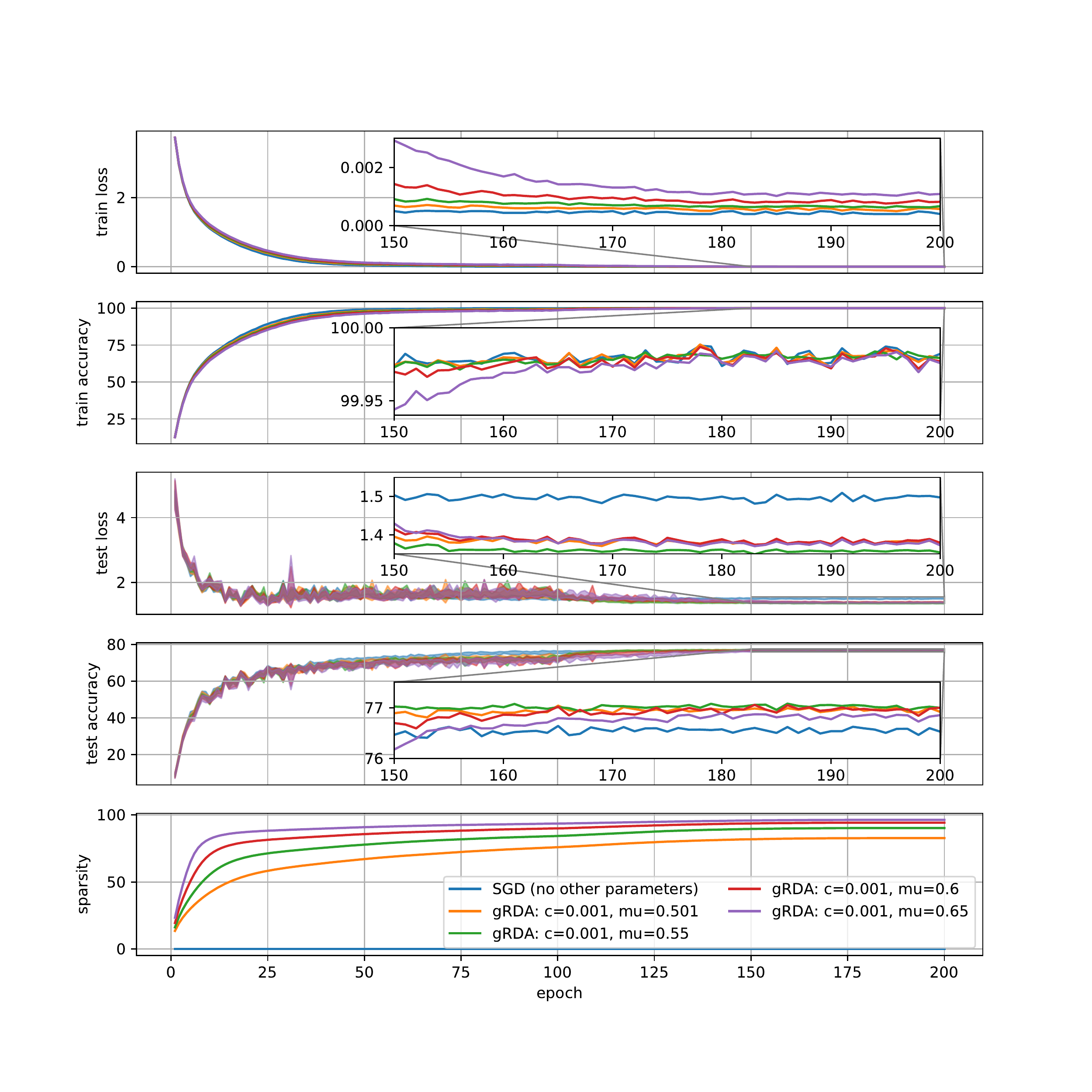}
 \vspace{-0.4cm}
 \caption{Learning trajectories of \eqref{eq:sgd} and \eqref{eq:grda} for WRN28x10 on CIFAR-100. See Section \ref{append:cifar} for the selection of hyperparameters about training.}\label{fig:cifar100wideres}
\vspace{-0.3cm}
\end{figure}

\begin{table}[!t]
\centering
\scriptsize
\caption{Details of the learning trajectories in Figure \ref{fig:cifar10vgg} at some selected epoch, which compare \eqref{eq:sgd} and \eqref{eq:grda} for VGG16 on CIFAR-10.  The means and the standard deviations (in the parenthesis) are taken on 10 independent trials initialized with independent random initializers.} \label{tab:cifar10vgg}
\begin{tabular}{lllllll}
\toprule Epoch      & 1               & 50              & 100             & 200             & 300             & 600             \\\midrule 
\multicolumn{7}{c}{Training Loss}                  \vspace{0.05in} \\
SGD        & 2.1797(0.0859)  & 0.1453(0.0091)  & 0.0404(0.0081)  & 0.0129(0.0017)  & 0.007(0.0019)   & 0.0(0.0)        \\
gRDA(0.4)  & 2.191(0.0897)   & 0.1537(0.0097)  & 0.0362(0.003)   & 0.0124(0.0047)  & 0.005(0.0014)   & 0.0(0.0)        \\
gRDA(0.51) & 2.2044(0.0939)  & 0.149(0.0185)   & 0.0351(0.0036)  & 0.011(0.0031)   & 0.0057(0.0018)  & 0.0(0.0)        \\
gRDA(0.6)  & 2.1735(0.0817)  & 0.1557(0.012)   & 0.0392(0.0054)  & 0.0169(0.0051)  & 0.0096(0.0023)  & 0.0001(0.0001)  \\
gRDA(0.7)  & 2.2394(0.0473)  & 0.2262(0.0386)  & 0.0644(0.0072)  & 0.0304(0.0094)  & 0.0223(0.0039)  & 0.0002(0.0001)  \\\midrule 
\multicolumn{7}{c}{Training Accuracy (\%)}                   \vspace{0.05in} \\
SGD        & 16.6538(3.6064) & 95.1673(0.2846) & 98.666(0.2707)  & 99.5902(0.0466) & 99.7911(0.0554) & 99.9993(0.0013) \\
gRDA(0.4)  & 15.8911(3.8974) & 94.8553(0.3168) & 98.7909(0.1281) & 99.6102(0.1431) & 99.8367(0.0429) & 99.9993(0.0009) \\
gRDA(0.51) & 15.7698(3.9934) & 95.0562(0.6049) & 98.8393(0.1085) & 99.6491(0.1029) & 99.8244(0.0484) & 99.9996(0.0008) \\
gRDA(0.6)  & 16.8571(3.6324) & 94.7998(0.415)  & 98.7262(0.1691) & 99.4624(0.1419) & 99.6913(0.0658) & 99.9991(0.001)  \\
gRDA(0.7)  & 14.7902(3.0521) & 92.6507(1.1549) & 97.8649(0.2249) & 99.0433(0.2869) & 99.3093(0.1221) & 99.9973(0.0019) \\\midrule 
\multicolumn{7}{c}{Testing Loss}                   \vspace{0.05in} \\
SGD        & 2.1212(0.0866)  & 0.4768(0.017)   & 0.5667(0.0142)  & 0.6561(0.0338)  & 0.6843(0.0392)  & 1.0713(0.0207)  \\
gRDA(0.4)  & 2.1185(0.0575)  & 0.4966(0.0234)  & 0.588(0.0314)   & 0.6729(0.0463)  & 0.7159(0.0267)  & 0.9868(0.0372)  \\
gRDA(0.51) & 2.1459(0.0956)  & 0.5002(0.0367)  & 0.5748(0.0217)  & 0.6485(0.0294)  & 0.6745(0.0288)  & 0.793(0.0268)   \\
gRDA(0.6)  & 2.0925(0.0481)  & 0.4856(0.0284)  & 0.5709(0.0209)  & 0.6005(0.0364)  & 0.6388(0.043)   & 0.7618(0.0262)  \\
gRDA(0.7)  & 2.1854(0.0459)  & 0.5055(0.0493)  & 0.567(0.0233)   & 0.6104(0.0221)  & 0.6412(0.0305)  & 0.8248(0.0215)  \\\midrule 
\multicolumn{7}{c}{Testing Accuracy (\%)}                   \vspace{0.05in} \\
SGD        & 19.1611(3.6631) & 87.3144(0.3584) & 88.73(0.2176)   & 89.4422(0.2907) & 89.8178(0.1843) & 90.8178(0.1779) \\
gRDA(0.4)  & 19.01(2.6052)   & 86.9711(0.4733) & 88.6333(0.3905) & 89.5444(0.2929) & 90.0522(0.2752) & 90.87(0.2082)   \\
gRDA(0.51) & 18.4556(4.2166) & 87.0656(0.789)  & 88.8511(0.3015) & 89.54(0.3037)   & 90.0478(0.2855) & 91.01(0.1464)   \\
gRDA(0.6)  & 20.1911(2.4566) & 87.1733(0.6538) & 88.6722(0.4554) & 89.3278(0.4227) & 89.6856(0.4379) & 90.8244(0.2414) \\
gRDA(0.7)  & 17.3589(2.4806) & 85.69(0.9893)   & 87.7489(0.4771) & 89.1044(0.3582) & 89.2267(0.4733) & 90.6433(0.224)  \\\midrule 
\multicolumn{7}{c}{Sparsity}                   \vspace{0.05in} \\
SGD        & 0.0(0.0)        & 0.0(0.0)        & 0.0(0.0)        & 0.0(0.0)        & 0.0(0.0)        & 0.0(0.0)        \\
gRDA(0.4)  & 2.4922(0.0042)  & 11.2733(0.0245) & 14.4122(0.0364) & 18.5011(0.0519) & 21.4367(0.0646) & 24.4433(0.0766) \\
gRDA(0.51) & 3.73(0.0)       & 25.4678(0.0464) & 34.5056(0.0677) & 46.2378(0.0961) & 54.2222(0.1095) & 61.4322(0.1143) \\
gRDA(0.6)  & 5.19(0.0)       & 47.5333(0.0994) & 63.9333(0.169)  & 79.6367(0.2028) & 86.3644(0.1928) & 90.49(0.1593)   \\
gRDA(0.7)  & 7.48(0.0)       & 80.3589(0.225)  & 92.0922(0.2057) & 95.5589(0.1481) & 96.39(0.109)    & 97.2167(0.0767) \\\bottomrule 
\end{tabular}

\end{table}

\begin{table}[]
\centering
\scriptsize
\caption{Details of the learning trajectories in Figure \ref{fig:cifar10wideres} at some selected epoch, which compare \eqref{eq:sgd} and \eqref{eq:grda} for WRN28x10 on CIFAR-10. The means and the standard deviations (in the parenthesis) are taken on 10 independent trials initialized with independent random initializers.} \label{tab:cifar10wideres}
\begin{tabular}{lllllll}
\toprule Epoch       & 1               & 25              & 50              & 75              & 100             & 200             \\\midrule
\multicolumn{7}{c}{Training Loss}                  \vspace{0.05in} \\
SGD         & 0.9002(0.0056)  & 0.0778(0.002)   & 0.0177(0.0023)  & 0.0078(0.0016)  & 0.0047(0.001)   & 0.0001(0.0)     \\
gRDA(0.501) & 0.9087(0.005)   & 0.0982(0.0014)  & 0.0298(0.0017)  & 0.0161(0.003)   & 0.0124(0.003)   & 0.0001(0.0)     \\
gRDA(0.55)  & 0.9112(0.0034)  & 0.1101(0.0018)  & 0.0359(0.0026)  & 0.0202(0.0026)  & 0.0143(0.0027)  & 0.0001(0.0)     \\
gRDA(0.6)   & 0.9155(0.0027)  & 0.1194(0.0015)  & 0.0378(0.0022)  & 0.0233(0.0017)  & 0.0178(0.0052)  & 0.0001(0.0)     \\
gRDA(0.65)  & 0.9162(0.0053)  & 0.1251(0.002)   & 0.0442(0.0024)  & 0.0275(0.0026)  & 0.0209(0.0027)  & 0.0002(0.0)     \\\midrule
  \multicolumn{7}{c}{Training Accuracy (\%)}            \vspace{0.05in} \\
SGD         & 67.9678(0.2445) & 97.2832(0.0693) & 99.398(0.096)   & 99.7476(0.0537) & 99.8456(0.0387) & 99.9996(0.0008) \\
gRDA(0.501) & 67.7258(0.2829) & 96.5754(0.0475) & 99.0054(0.0772) & 99.4766(0.1011) & 99.587(0.1076)  & 99.9998(0.0006) \\
gRDA(0.55)  & 67.545(0.1816)  & 96.1636(0.0878) & 98.7718(0.1102) & 99.3382(0.101)  & 99.5198(0.0993) & 99.9986(0.0018) \\
gRDA(0.6)   & 67.4044(0.1635) & 95.8476(0.0772) & 98.702(0.0927)  & 99.2146(0.0636) & 99.3944(0.1851) & 99.9984(0.0012) \\
gRDA(0.65)  & 67.42(0.3047)   & 95.6366(0.0744) & 98.4866(0.0999) & 99.0694(0.0985) & 99.299(0.1049)  & 99.9982(0.0017) \\\midrule
\multicolumn{7}{c}{Testing Loss}                         \vspace{0.05in} \\
SGD         & 1.2606(0.1645)  & 0.4845(0.0837)  & 0.3849(0.0278)  & 0.4238(0.0942)  & 0.4151(0.0275)  & 0.3624(0.0074)  \\
gRDA(0.501) & 1.2986(0.2294)  & 0.4099(0.0584)  & 0.4008(0.0631)  & 0.4547(0.0998)  & 0.3918(0.0521)  & 0.3241(0.0099)  \\
gRDA(0.55)  & 1.2489(0.1411)  & 0.5151(0.193)   & 0.3854(0.048)   & 0.4738(0.1341)  & 0.3899(0.0455)  & 0.3266(0.0069)  \\
gRDA(0.6)   & 1.345(0.1042)   & 0.4063(0.0352)  & 0.3941(0.0632)  & 0.4135(0.095)   & 0.4414(0.1226)  & 0.3245(0.0075)  \\
gRDA(0.65)  & 1.3211(0.216)   & 0.4487(0.0556)  & 0.346(0.0343)   & 0.3805(0.0521)  & 0.3873(0.0603)  & 0.3182(0.0048)  \\\midrule
  \multicolumn{7}{c}{Testing Accuracy (\%)}            \vspace{0.05in} \\
SGD         & 61.478(2.674)   & 88.06(1.4363)   & 91.978(0.3655)  & 92.363(1.0626)  & 92.968(0.3802)  & 94.173(0.1127)  \\
gRDA(0.501) & 61.021(4.0857)  & 88.858(1.0066)  & 91.161(0.9286)  & 91.109(1.4645)  & 92.517(0.7402)  & 94.459(0.12)    \\
gRDA(0.55)  & 61.565(2.8849)  & 86.546(2.9551)  & 91.219(0.7465)  & 90.717(1.6355)  & 92.4(0.5953)    & 94.433(0.0805)  \\
gRDA(0.6)   & 59.559(1.8942)  & 88.481(0.9584)  & 90.922(1.1019)  & 91.486(1.2848)  & 91.554(1.6144)  & 94.497(0.166)   \\
gRDA(0.65)  & 59.934(3.7627)  & 87.733(0.7991)  & 91.637(0.5235)  & 91.683(0.9814)  & 91.996(0.8251)  & 94.531(0.1119)  \\\midrule
  \multicolumn{7}{c}{Sparsity (\%)}       \vspace{0.05in} \\
SGD         & 0.0(0.0)        & 0.0(0.0)        & 0.0(0.0)        & 0.0(0.0)        & 0.0(0.0)        & 0.0(0.0)        \\
gRDA(0.501) & 19.1382(0.003)  & 64.4465(0.0217) & 77.8759(0.0513) & 82.8773(0.0502) & 85.5712(0.0425) & 90.1746(0.0301) \\
gRDA(0.55)  & 23.6897(0.0022) & 79.9356(0.0209) & 87.2516(0.0189) & 90.0057(0.0223) & 91.5633(0.0357) & 94.7444(0.0367) \\
gRDA(0.6)   & 29.4495(0.0013) & 88.8105(0.0315) & 92.5813(0.019)  & 94.155(0.0219)  & 95.0531(0.0178) & 97.0747(0.0179) \\
gRDA(0.65)  & 36.6118(0.0024) & 93.5285(0.0143) & 95.6032(0.0127) & 96.5137(0.0111) & 97.0338(0.0138) & 98.2877(0.0125) \\\bottomrule
\end{tabular}
\end{table}

\begin{table}[]
\centering
\scriptsize
\caption{Details of the learning trajectories in Figure \ref{fig:cifar100wideres} at some selected epoch, which compare \eqref{eq:sgd} and \eqref{eq:grda} for WRN28x10 on CIFAR-100. The means and the standard deviations (in the parenthesis) are taken on 10 independent trials initialized with independent random initializers.} \label{tab:cifar100wideres}
\begin{tabular}{lllllll}
\toprule Epoch         & 1               & 25              & 50              & 75              & 100             & 200             \\\midrule
\multicolumn{7}{c}{Training Loss}                  \vspace{0.05in} \\
SGD            & 2.9477(0.0037)    & 0.3347(0.0038)  & 0.0348(0.0044)  & 0.0076(0.0023)  & 0.0023(0.0028)  & 0.0004(0.0)     \\
gRDA(0.501)    & 2.9632(0.0052)    & 0.3921(0.0031)  & 0.0652(0.0047)  & 0.033(0.0073)   & 0.0233(0.0063)  & 0.0006(0.0)     \\
gRDA(0.55)     & 2.9752(0.0052)    & 0.4167(0.0025)  & 0.0758(0.005)   & 0.0427(0.0055)  & 0.0317(0.0077)  & 0.0007(0.0)     \\
gRDA(0.6)      & 2.9798(0.0038)    & 0.4407(0.0042)  & 0.0966(0.0041)  & 0.0542(0.0039)  & 0.0394(0.0075)  & 0.0008(0.0)     \\
gRDA(0.65)     & 2.9885(0.0062)    & 0.4633(0.0035)  & 0.1188(0.0046)  & 0.0692(0.0053)  & 0.0495(0.0038)  & 0.0011(0.0)     \\\midrule
  \multicolumn{7}{c}{Training Accuracy (\%)}            \vspace{0.05in} \\
SGD           & 25.7816(0.1029) & 89.2406(0.1343) & 99.071(0.1465)  & 99.825(0.0637)  & 99.9418(0.0723) & 99.9822(0.0038) \\
gRDA(0.501)   & 25.4118(0.0721) & 87.4652(0.141)  & 98.069(0.1749)  & 99.0446(0.256)  & 99.327(0.2105)  & 99.9792(0.0022) \\
gRDA(0.55)    & 25.188(0.1628)  & 86.6598(0.0732) & 97.7116(0.1959) & 98.7498(0.2099) & 99.0674(0.2492) & 99.9794(0.0035) \\
gRDA(0.6)     & 25.092(0.0661)  & 86.0122(0.1322) & 97.0026(0.1456) & 98.3896(0.1561) & 98.8524(0.253)  & 99.977(0.0029)  \\
gRDA(0.65)    & 24.9426(0.128)  & 85.3238(0.1064) & 96.2586(0.1602) & 97.9104(0.1846) & 98.533(0.1325)  & 99.9758(0.0014) \\\midrule
\multicolumn{7}{c}{Testing Loss}                       \vspace{0.05in} \\
SGD            & 3.6524(0.1035)    & 1.423(0.1287)   & 1.5834(0.0411)  & 1.5164(0.0427)  & 1.5109(0.0621)  & 1.4977(0.0164)  \\
gRDA(0.501)    & 3.6615(0.1535)    & 1.3842(0.0697)  & 1.7227(0.1908)  & 1.646(0.2176)   & 1.6685(0.1295)  & 1.3766(0.016)   \\
gRDA(0.55)     & 3.8418(0.1471)    & 1.3317(0.0558)  & 1.7881(0.1633)  & 1.6067(0.1051)  & 1.6622(0.1452)  & 1.3537(0.0204)  \\
gRDA(0.6)      & 3.8368(0.1623)    & 1.332(0.0551)   & 1.6861(0.1355)  & 1.6166(0.0618)  & 1.6807(0.1061)  & 1.3789(0.0181)  \\
gRDA(0.65)     & 3.7701(0.1759)    & 1.3689(0.1135)  & 1.6668(0.1081)  & 1.5665(0.1276)  & 1.6499(0.1364)  & 1.3722(0.0156)  \\\midrule
  \multicolumn{7}{c}{Testing Accuracy (\%)}           \vspace{0.05in} \\
SGD           & 19.408(1.1196)  & 66.605(1.8987)  & 72.018(0.6428)  & 74.775(0.4029)  & 75.807(0.638)   & 76.529(0.169)   \\
gRDA(0.501)   & 18.798(1.2235)  & 66.621(1.0782)  & 69.549(1.7699)  & 71.9(1.7964)    & 72.47(1.2624)   & 76.916(0.1894)  \\
gRDA(0.55)    & 17.644(0.9865)  & 67.039(0.8587)  & 68.437(1.6978)  & 72.095(1.3079)  & 71.792(1.8297)  & 76.996(0.1713)  \\
gRDA(0.6)     & 17.761(0.9144)  & 67.107(1.0409)  & 69.13(1.6105)   & 71.334(0.6645)  & 71.328(0.9034)  & 76.999(0.3635)  \\
gRDA(0.65)    & 17.653(1.0898)  & 66.13(2.0093)   & 68.698(1.1661)  & 71.416(1.5185)  & 71.3(1.5667)    & 76.853(0.1996)  \\\midrule
  \multicolumn{7}{c}{Sparsity (\%)}       \vspace{0.05in} \\
SGD           & 0.0(0.0)        & 0.0(0.0)        & 0.0(0.0)        & 0.0(0.0)        & 0.0(0.0)        & 0.0(0.0)        \\
gRDA(0.501)   & 19.1087(0.0031) & 58.3232(0.0284) & 67.0737(0.0783) & 72.4028(0.0876) & 75.8975(0.1454) & 82.6918(0.0986) \\
gRDA(0.55)    & 23.6534(0.0026) & 71.3479(0.0412) & 77.8894(0.0343) & 81.8828(0.0499) & 84.2722(0.0676) & 90.2043(0.0508) \\
gRDA(0.6)     & 29.4058(0.0024) & 81.4373(0.0293) & 85.6999(0.0211) & 88.3771(0.0221) & 89.9111(0.0318) & 94.2045(0.0649) \\
gRDA(0.65)    & 36.5536(0.002)  & 88.195(0.0166)  & 90.845(0.0191)  & 92.5057(0.018)  & 93.4805(0.0184) & 96.3042(0.0219)\\\bottomrule
\end{tabular}
\end{table}

\subsection{Wall time and GPU memory consumption}\label{ap:resource_comparison}

In this section, we compare the wall time and the memory consumption between the gRDA and the SGD. All results in this section are done using the same server containing two Nvidia Tesla V100 (16GB memory) GPUs. We use two cards in training ResNet50-ImageNet, one card in training VGG16-CIFAR10 and WRN28x10-CIFAR100. Experiments are done serially. The training details are the same as described in Section \ref{ap:imagenet_details} and \ref{append:cifar}. For the hyperparameters of the gRDA, we take $c=0.001, \mu=0.6$ for ResNet50-ImageNet, $c=0.001, \mu=0.4$ for VGG16-CIFAR10, and $c=0.001, \mu=0.501$ for WRN28x10-CIFAR100. The choice of $c$ and $\mu$ in gRDA does not affect the time usage and memory footprint.

For the wall time, in the case of ResNet50-ImageNet, we record the calculation time for the first 200 iterations per 10 iteration. We omit the first iteration since it is much larger than the others due to model initiation on GPU. We calculate the average and the standard deviation using the remaining sample of size 19. In the cases of VGG16-CIFAR10 and WRN28x10-CIFAR-100, we record the calculation time for the first 20 epochs (390 iterations per epoch) and omit the very first epoch. We calculate the mean and the standard deviation of the 19 trials.

For the memory consumption, we focus on the peak GPU memory usage, i.e. the maximum memory usage during training, since it determines whether the task is trainable on the given platform. In the case of ResNet50-ImageNet, we record the usage of the first 200 iterations among 5 different training tries. We show the memory usage for two GPU cards separately because the design of PyTorch leads to a higher memory usage in the card0. In the cases of VGG16-CIFAR10 and WRN28x10-CIFAR-100, we record the peak GPU memory usage throughout the first 20 epochs. We calculate the mean and the standard deviation of the 5 tries.

From Table \ref{tab:resource_consumption}, the gRDA generally requires a higher wall time than the SGD, because gRDA requires an additional step for the soft thresholding. For the memory consumption, one can observe that the difference between the gRDA and the SGD depends on the tasks and architectures, although it is generally small. In particular for the case of ResNet50-ImageNet, the difference in means of the SGD and the gRDA is not significant since it is less than their respective standard deviations. In fact, we find that the GPU memory consumption is unstable in these 5 tries, and sometimes the gRDA uses slightly less GPU memory than the SGD. The reason of the difference could be due to the underlying design of PyTorch, which may be interesting for future research.  

\begin{table}[]
\centering
\footnotesize
\caption{Comparison of SGD and gRDA on time and GPU memory consumption. The values in the upper penal of the table are the average time consumption of 19 records excluding the initial iterations. The values in the lower penal are the average peak GPU memory consumption of 5 different tries. The numbers in parenthesis are the standard deviation.}
\label{tab:resource_consumption}
\begin{tabular}{llll}
\toprule 
  \multicolumn{4}{c}{Time per iteration (s)}      \\\midrule
             & ResNet50-ImageNet & VGG16-CIFAR10 & WRN28x10-CIFAR100 \\\midrule
SGD          & 0.3964 (0.0183)            & 0.0214 (0.0002)         & 0.2271 (0.0008)            \\
gRDA         & 0.4582 (0.0166)              & 0.0303 (0.0004)        & 0.2510 (0.0011)            \\\toprule
  \multicolumn{4}{c}{GPU Memory (MiB)}       \\\midrule
             & ResNet50-ImageNet (card0,1)   & VGG16-CIFAR10 (card0) & WRN28x10-CIFAR100 (card0) \\\midrule
SGD          & 14221 (376), 14106 (380)      & 1756 (48.6)           & 10301 (0)                     \\
gRDA         & 14159 (167), 13947 (208)      & 1809 (10.2)           & 10589 (0)                \\\bottomrule 
\end{tabular}
\end{table}

\begin{figure}[!h]
	\centering
	\begin{subfigure}[t]{0.48\textwidth}
		\centering
		\includegraphics[width=0.9\textwidth]{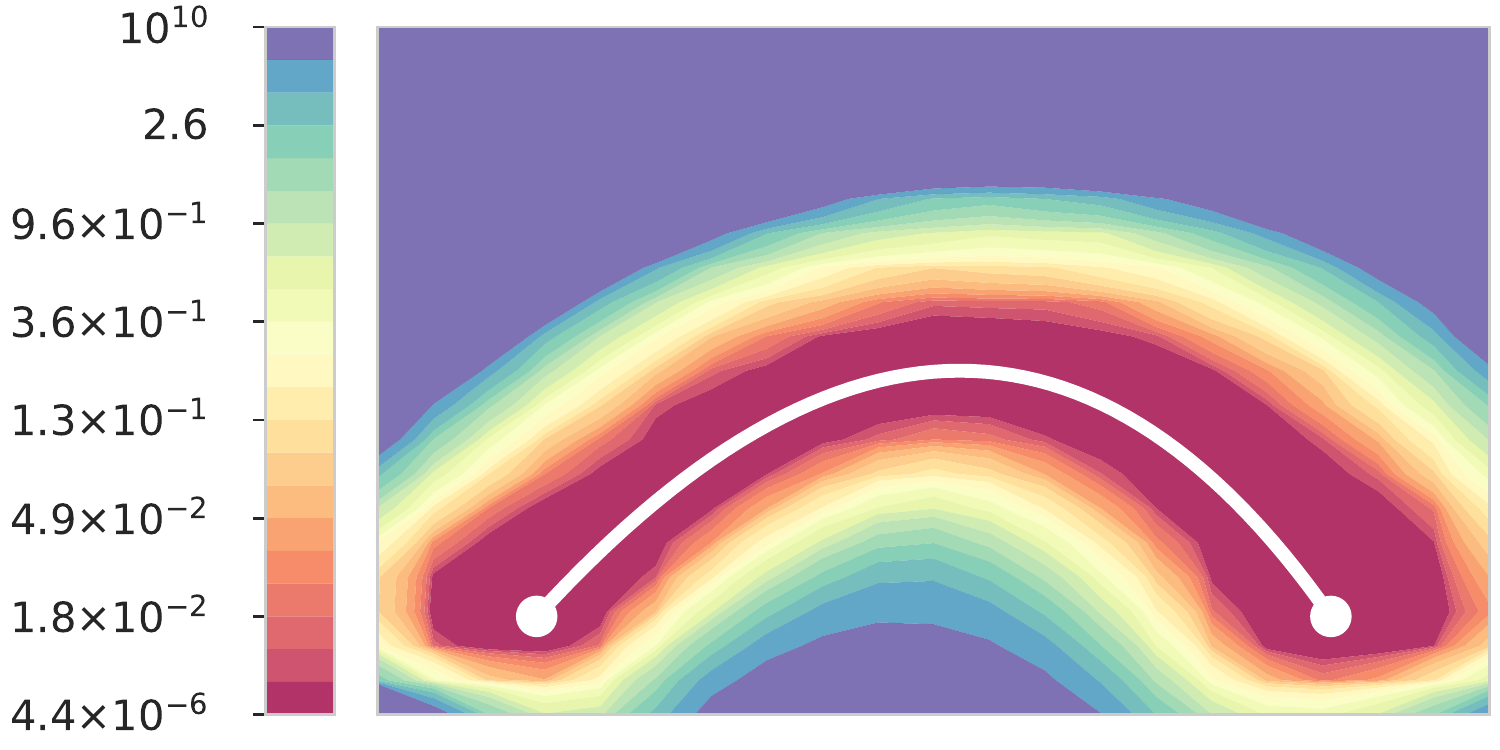}
		\caption{VGG16/CIFAR-10/Train loss/$\mu=0.4$}\label{fig:7a}		
	\end{subfigure}
	\quad
	\begin{subfigure}[t]{0.48\textwidth}
		\centering
		\includegraphics[width=0.9\textwidth]{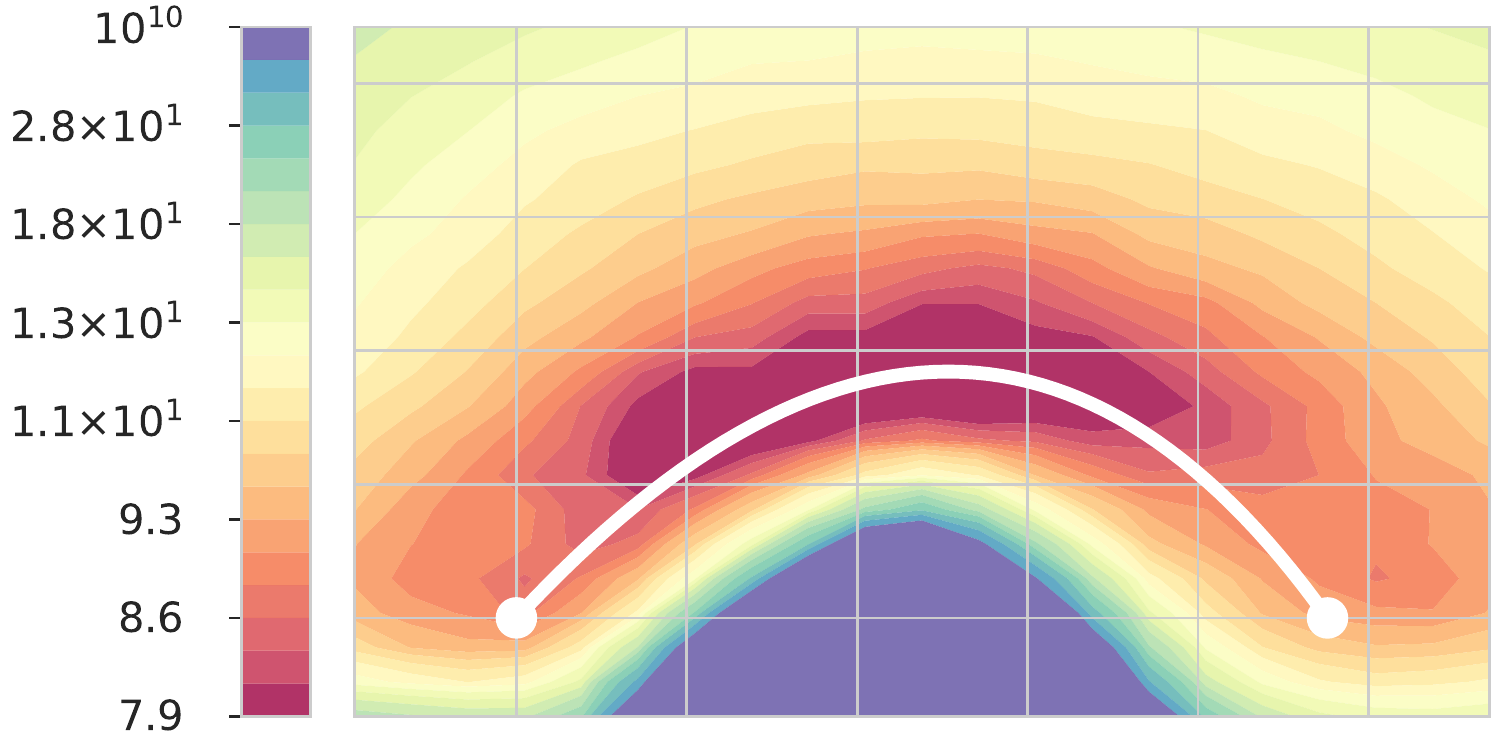}
		\caption{VGG16/CIFAR-10/Test error/$\mu=0.4$}\label{fig:7b}
	\end{subfigure}\\
	\begin{subfigure}[t]{0.48\textwidth}
		\centering
		\includegraphics[width=0.9\textwidth]{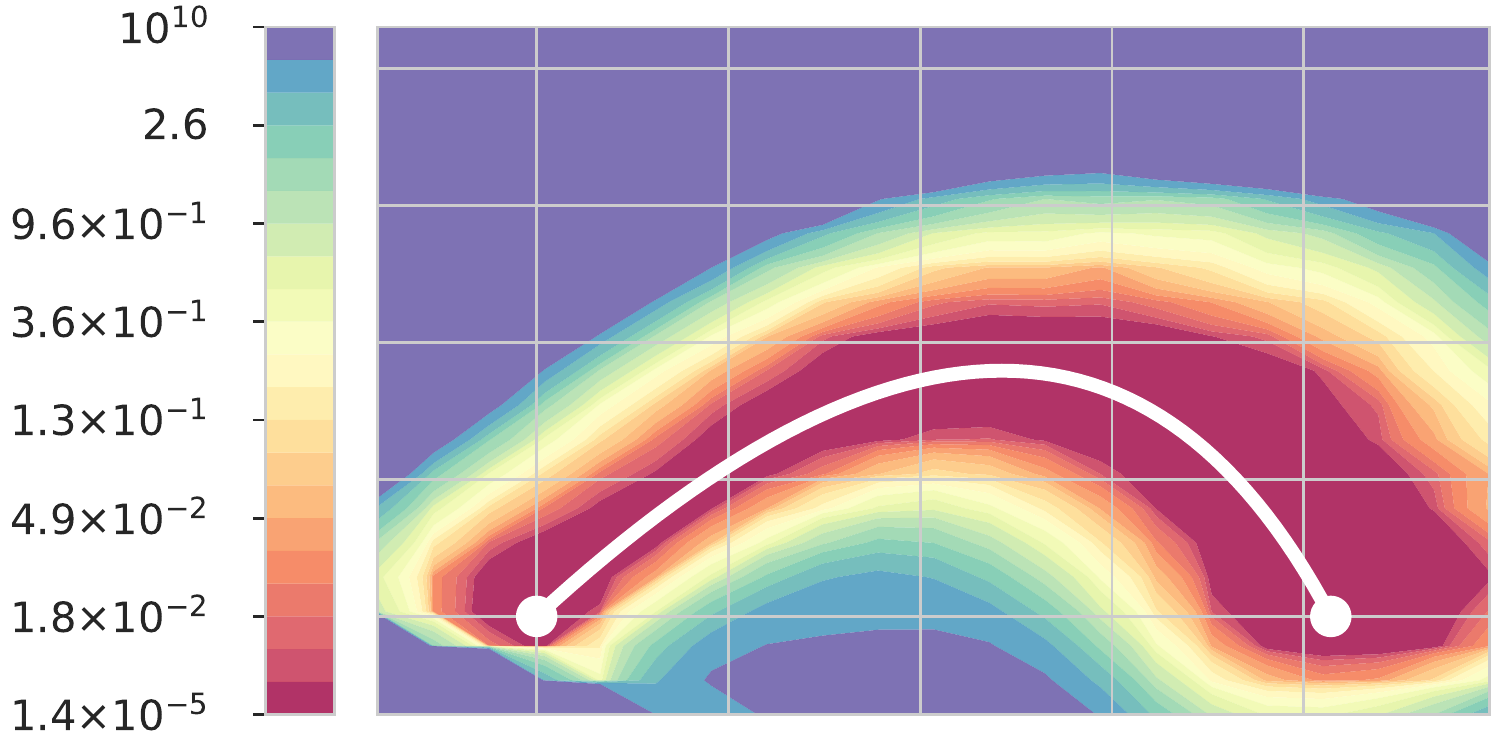}
		\caption{VGG16/CIFAR-10/Train loss/$\mu=0.51$}\label{fig:7c}		
	\end{subfigure}
	\quad
	\begin{subfigure}[t]{0.48\textwidth}
		\centering
		\includegraphics[width=0.9\textwidth]{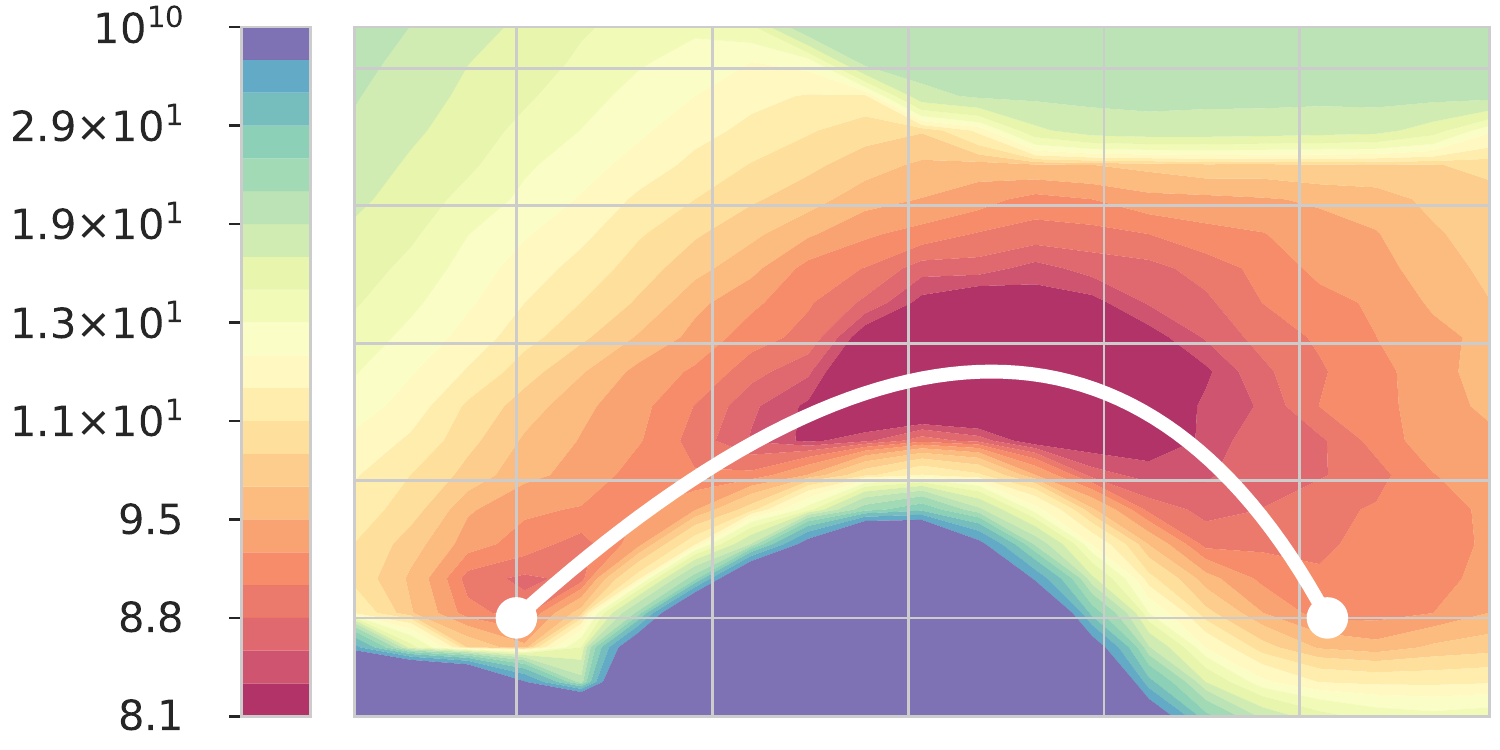}
		\caption{VGG16/CIFAR-10/Test error/$\mu=0.51$}\label{fig:7d}
	\end{subfigure}\\
	\begin{subfigure}[t]{0.48\textwidth}
		\centering
		\includegraphics[width=0.9\textwidth]{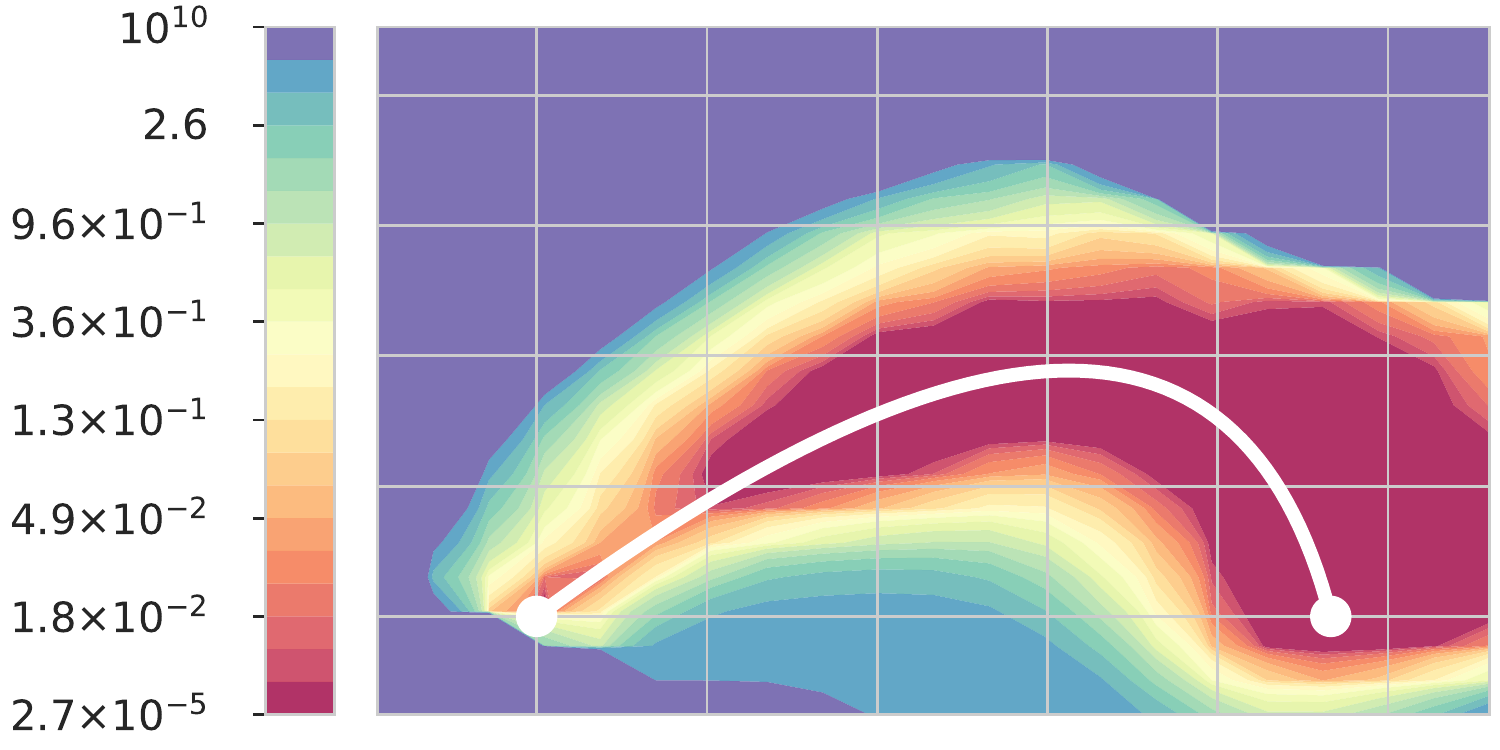}
		\caption{VGG16/CIFAR-10/Train loss/$\mu=0.7$}\label{fig:7e}		
	\end{subfigure}
	\quad
	\begin{subfigure}[t]{0.48\textwidth}
		\centering
		\includegraphics[width=0.9\textwidth]{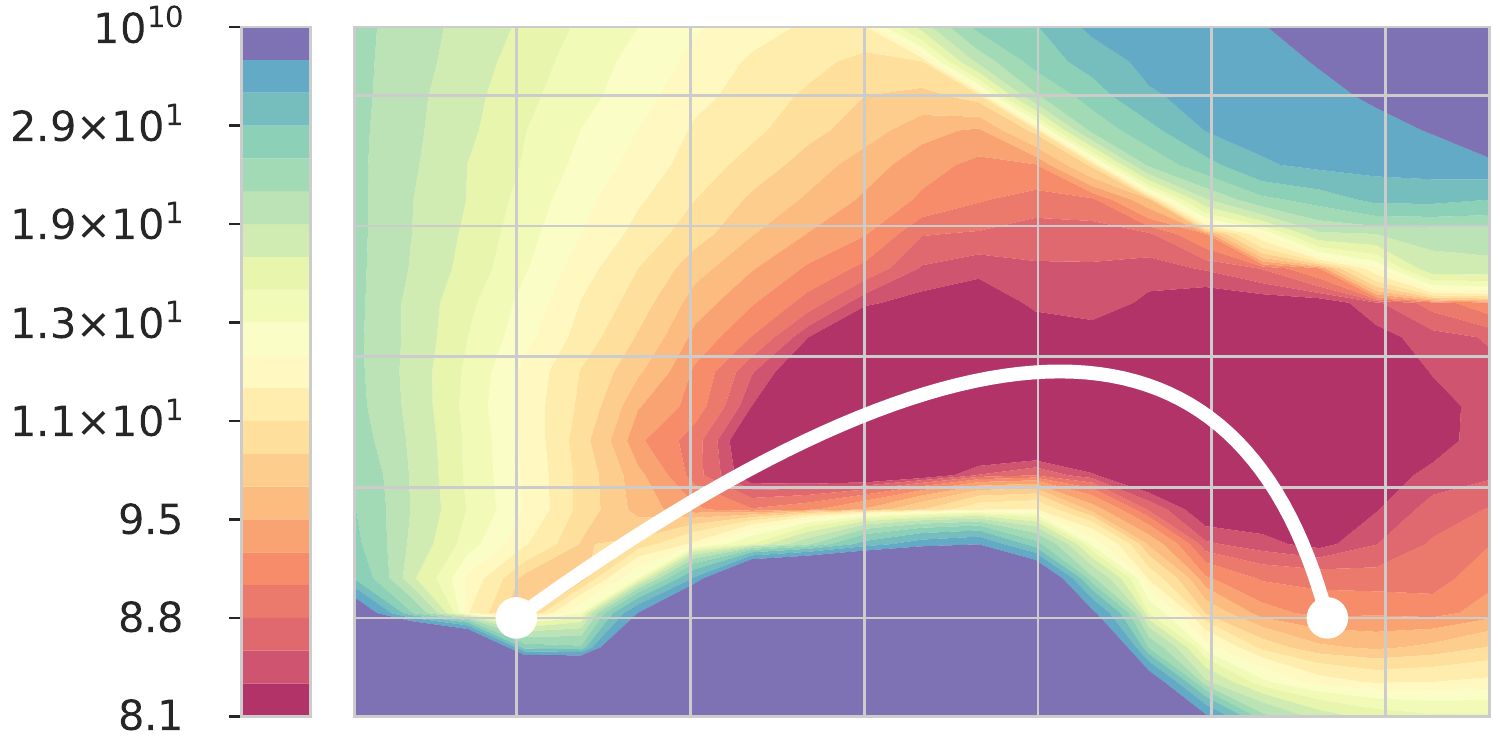}
		\caption{VGG16/CIFAR-10/Test error/$\mu=0.7$}\label{fig:7f}
	\end{subfigure}\\
	\begin{subfigure}[t]{0.48\textwidth}
		\centering
		\includegraphics[width=0.9\textwidth]{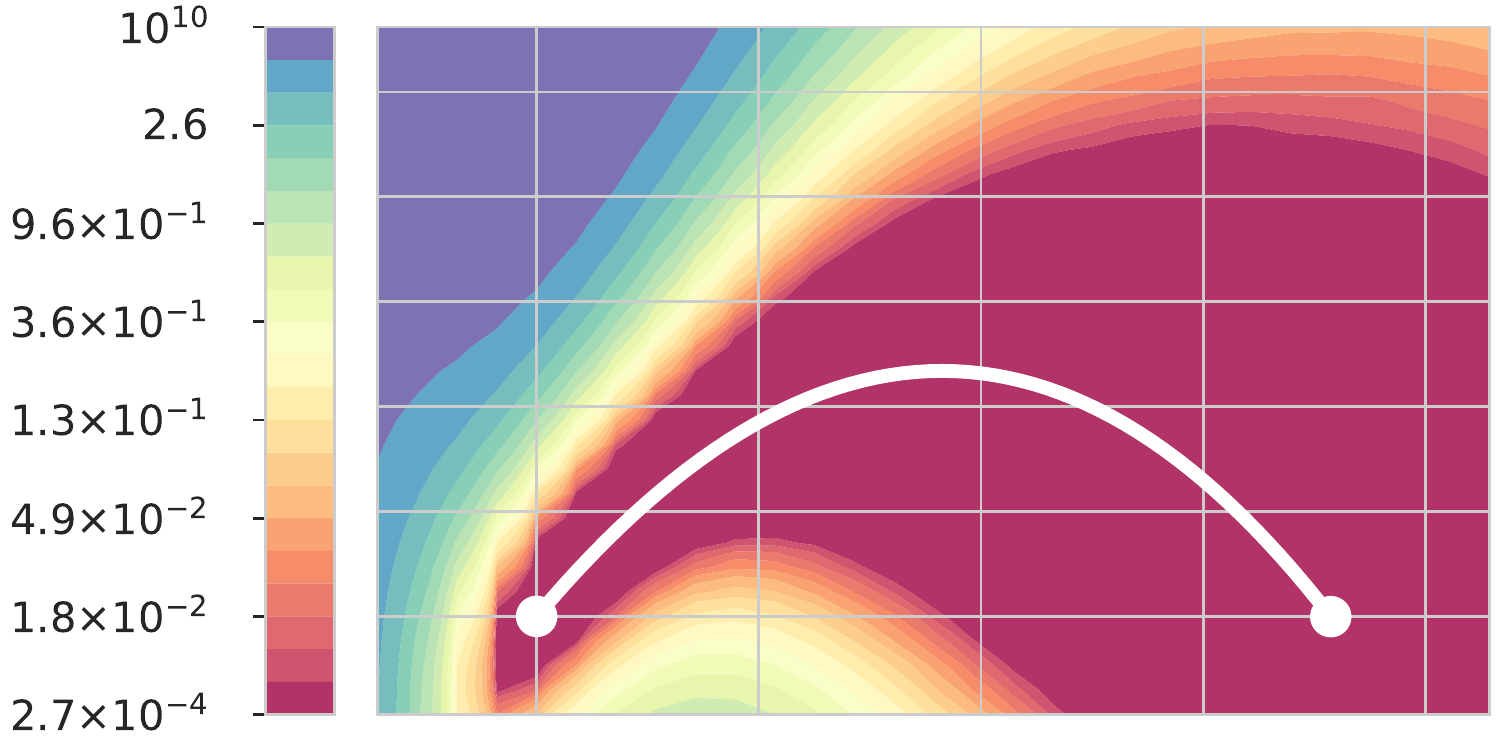}
		\caption{WRN28x10/CIFAR-100/Train loss/$\mu=0.55$}\label{fig:7a1}		
	\end{subfigure}
	\quad
	\begin{subfigure}[t]{0.48\textwidth}
		\centering
		\includegraphics[width=0.9\textwidth]{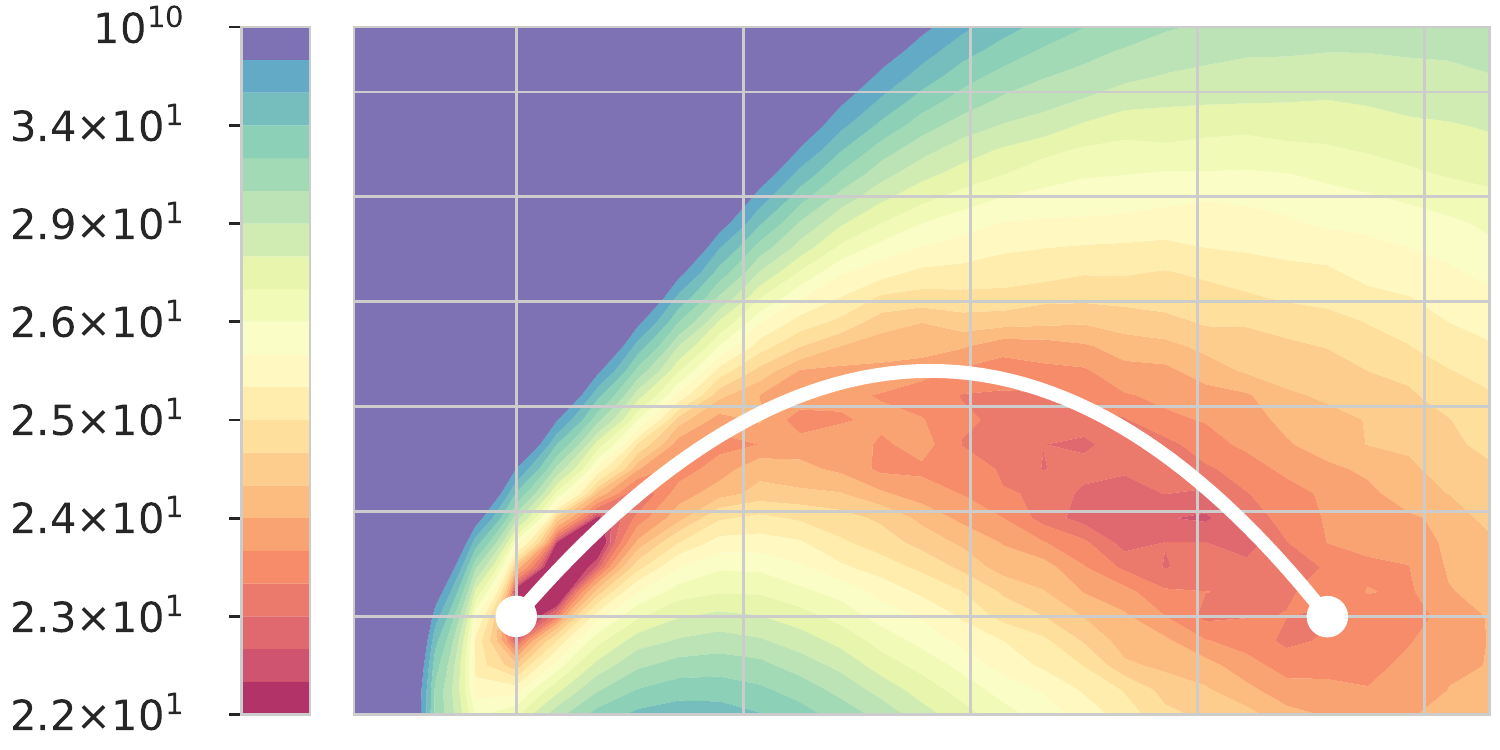}
		\caption{WRN28x10/CIFAR-100/Test error/$\mu=0.55$}\label{fig:7b1}
	\end{subfigure}\\
	\begin{subfigure}[t]{0.48\textwidth}
		\centering
		\includegraphics[width=0.9\textwidth]{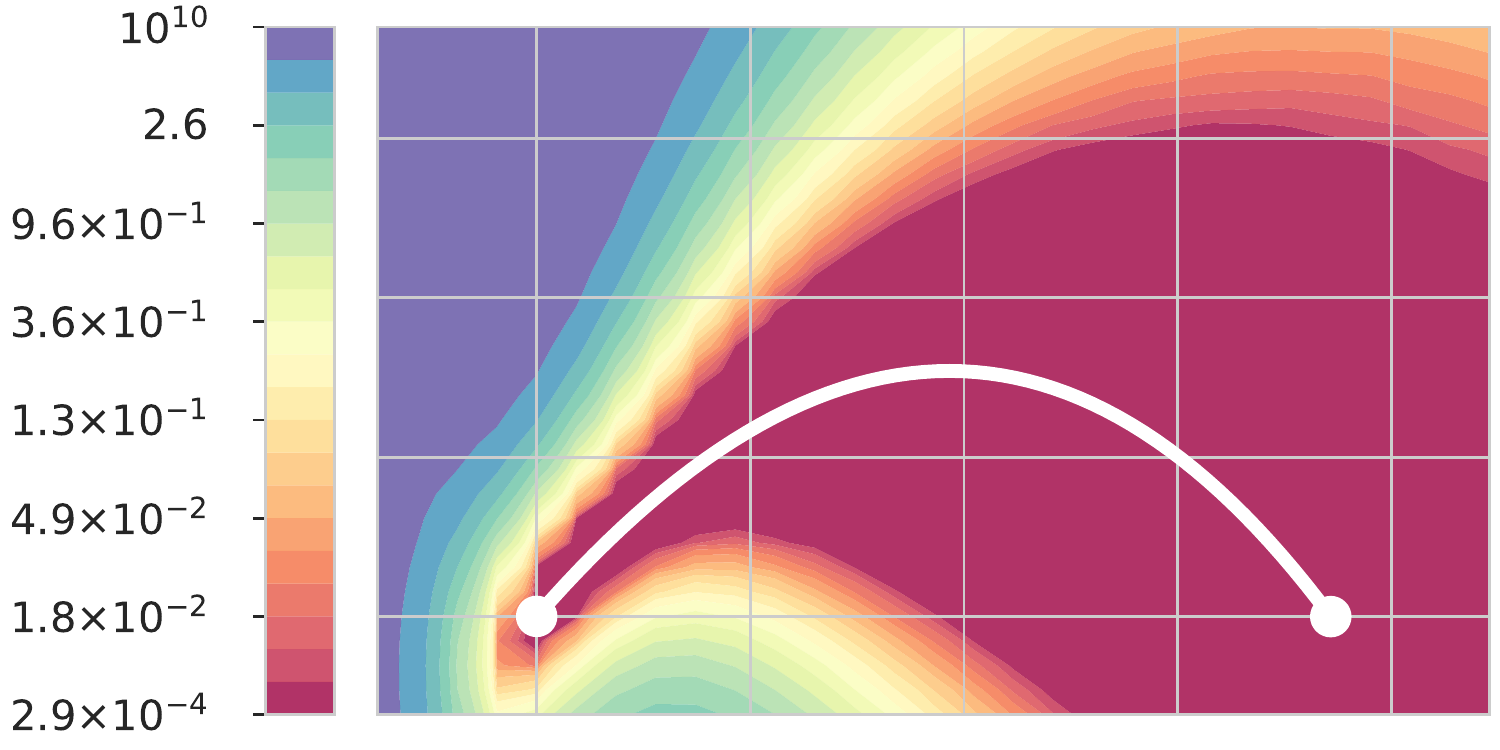}
		\caption{WRN28x10/CIFAR-100/Train loss/$\mu=0.6$}\label{fig:7c1}		
	\end{subfigure}
	\quad
	\begin{subfigure}[t]{0.48\textwidth}
		\centering
		\includegraphics[width=0.9\textwidth]{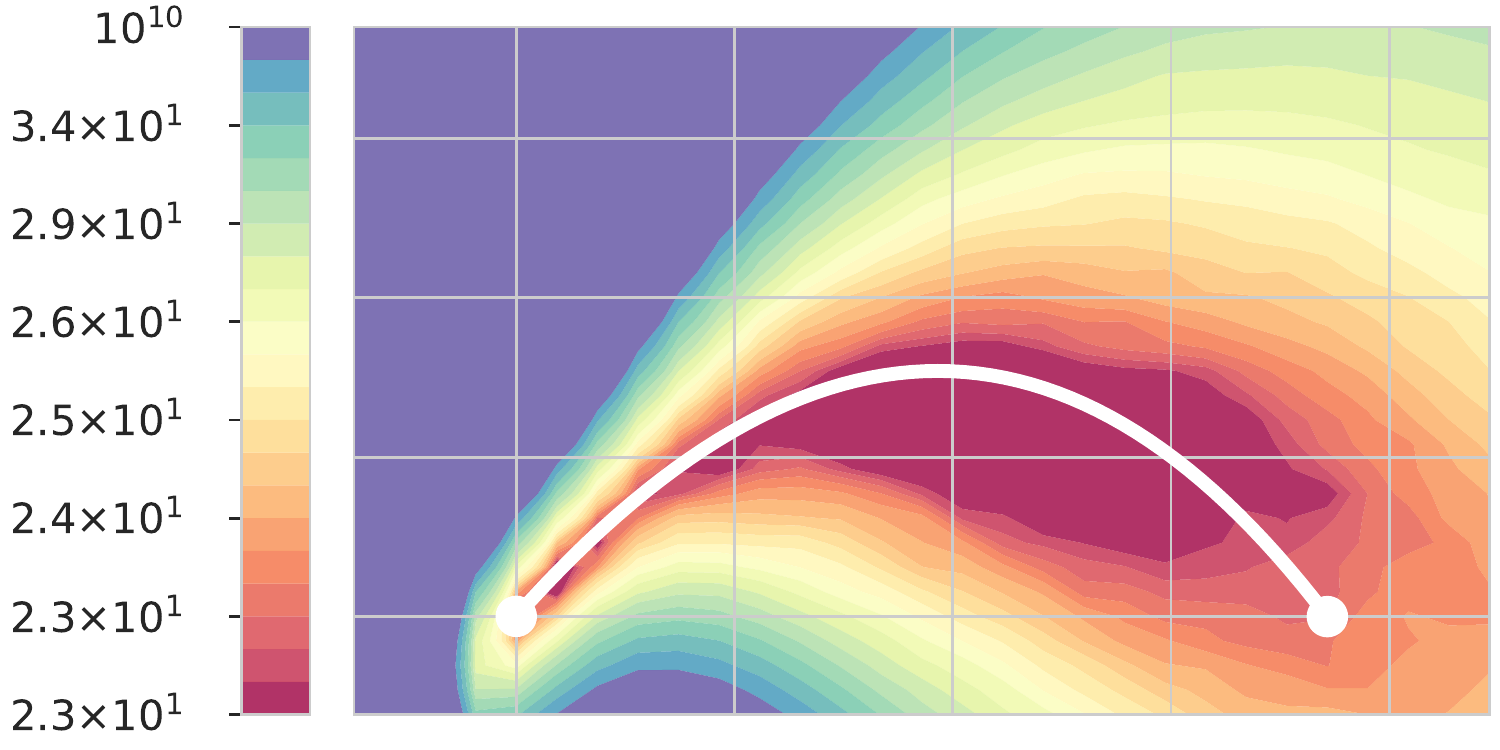}
		\caption{WRN28x10/CIFAR-100/Test error/$\mu=0.6$}\label{fig:7d1}
	\end{subfigure}\\
	\begin{subfigure}[t]{0.48\textwidth}
		\centering
		\includegraphics[width=0.9\textwidth]{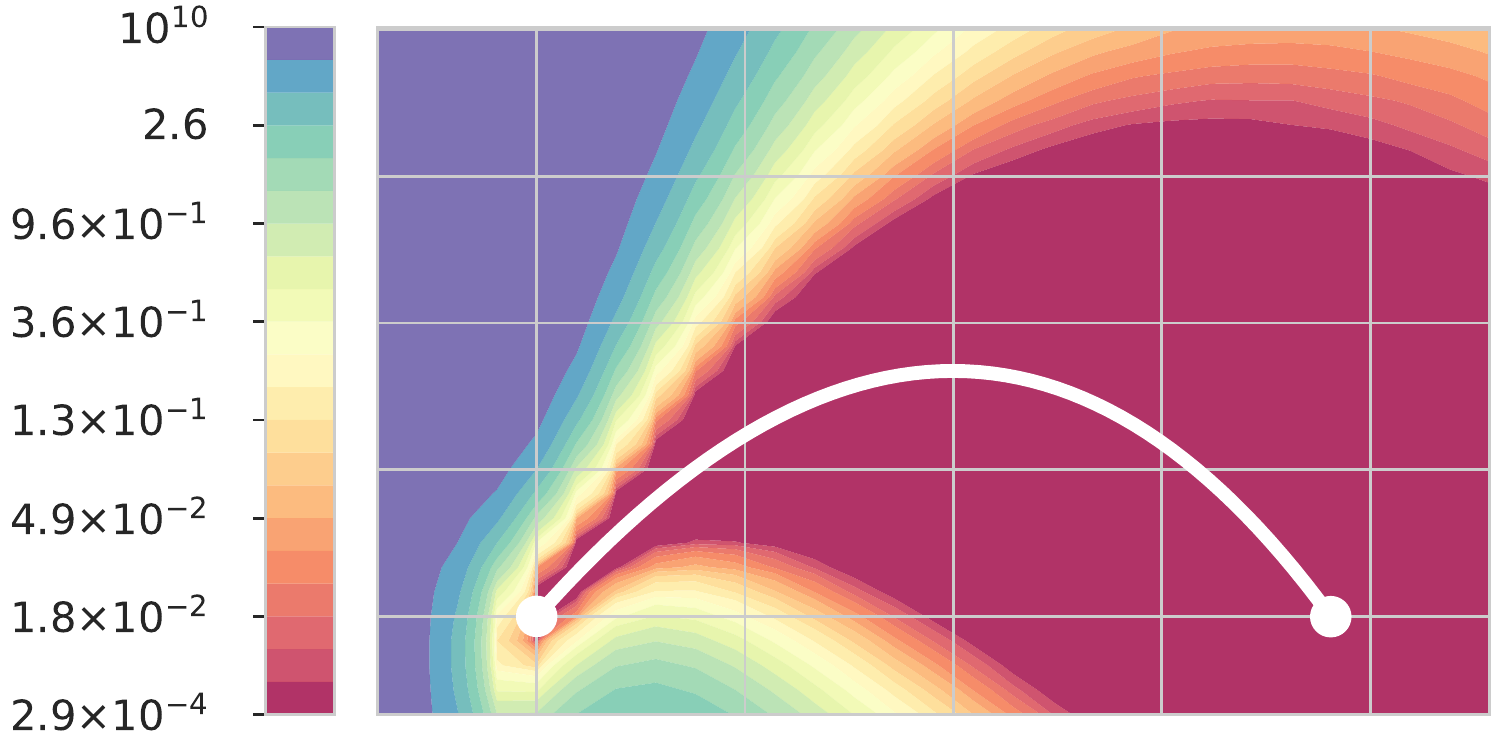}
		\caption{WRN28x10/CIFAR-100/Train loss/$\mu=0.65$}\label{fig:7e1}		
	\end{subfigure}
	\quad
	\begin{subfigure}[t]{0.48\textwidth}
		\centering
		\includegraphics[width=0.9\textwidth]{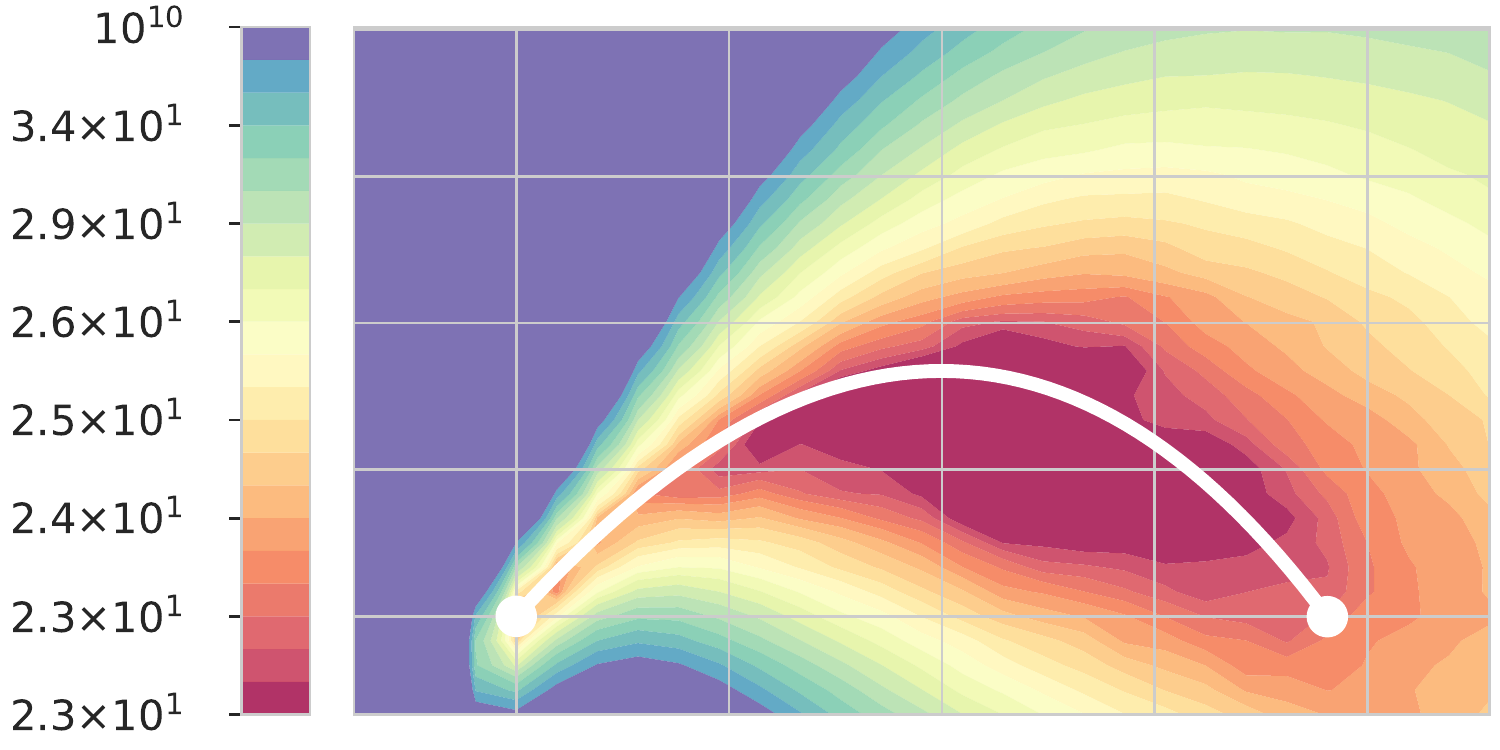}
		\caption{WRN28x10/CIFAR-100/Test error/$\mu=0.65$}\label{fig:7f1}
	\end{subfigure}
 \caption{The contour of training loss and testing error around the minimal loss/error curves. (Figure \ref{fig:conn}) The right end point is the SGD, and the left end point is the gRDA. 
 }\label{app:conn}
\end{figure}

\end{document}